\documentclass[11pt]{article}
\usepackage{enumerate}
\usepackage{bm,bbm}
\usepackage[OT1]{fontenc}
\usepackage{amsmath,amssymb}
\usepackage{natbib}
\usepackage[usenames]{color}

\usepackage{smile}
\usepackage{natbib}
\usepackage{multirow}
\usepackage{bbding}
\usepackage{colortbl}

\usepackage{stackengine}

\usepackage[breaklinks,colorlinks,
            linkcolor=red,
            anchorcolor=blue,
            citecolor=blue
            ]{hyperref}

\def\tr{\mathop{\text{tr}}\kern.2ex}

\long\def\comment#1{}

\def\tr{\mathop{\text{Tr}}}

\def\cS{{\mathcal{S}}}

\def\cO{\mathcal{O}}
\def\cP{{\mathcal{P}}}

\def\tr{{\text{Tr}}}

\def\dr{\displaystyle \rm}

\newcommand{\bel}{\begin{eqnarray}\label}
\newcommand{\eel}{\end{eqnarray}}
\newcommand{\bes}{\begin{eqnarray*}}
\newcommand{\ees}{\end{eqnarray*}}

\newcommand{\la}{\langle}
\newcommand{\ra}{\rangle}

\usepackage{mathrsfs}

\usepackage{fullpage}

\def\sep{\;\big|\;}

\usepackage{hyperref}
\usepackage{breakcites}
\def\##1\#{\begin{align}#1\end{align}}
\def\$#1\${\begin{align*}#1\end{align*}}
\usepackage{enumitem}

\usepackage{tikz}
\usetikzlibrary{automata, positioning}
\usepackage{verbatim}

\def \algo {\textrm{Model-Based OPPO }}
\def \algoplus {\textrm{PROPO }}
\def \algovi {\textrm{SW-LSVI-UCB }}

\begin{document}

\title{\LARGE Optimistic Policy Optimization is Provably Efficient in Non-stationary MDPs}
\author
{\normalsize
 Han Zhong\thanks{Peking University. Email: \texttt{hanzhong@stu.pku.edu.cn}} \qquad Zhongren Chen\thanks{Yale University. Email: \texttt{zhongren.chen@yale.edu}} \qquad  Zhuoran Yang\thanks{Yale University. Email: \texttt{zhuoran.yang@yale.edu}} \qquad Zhaoran Wang\thanks{Northwestern University. Email: \texttt{zhaoranwang@gmail.com}}\qquad Csaba Szepesv\'ari\thanks{DeepMind and
 University of Alberta. Email: \texttt{csaba.szepesvari@ulberta.ca}}
}
\date{}
\maketitle

\setlength{\abovedisplayskip}{6pt}
\setlength{\belowdisplayskip}{6pt}

\begin{abstract}
    We study episodic reinforcement learning (RL) in non-stationary linear kernel  Markov decision processes (MDPs). In this setting, both the reward function and the transition kernel are linear with respect to the given feature maps and are allowed to vary over time, as long as their respective parameter variations do not exceed certain variation budgets. We propose the \underline{p}eriodically \underline{r}estarted \underline{o}ptimistic \underline{p}olicy \underline{o}ptimization algorithm (PROPO), which is an optimistic policy optimization algorithm with linear function approximation. PROPO features two mechanisms: sliding-window-based policy evaluation and periodic-restart-based policy improvement, 
    which are tailored for policy optimization in a non-stationary environment. In addition, only utilizing the technique of sliding window, we propose a value-iteration algorithm. We establish dynamic upper bounds for the proposed methods and a minimax lower bound which shows the (near-) optimality of the proposed methods. To our best knowledge, PROPO is the first provably efficient policy optimization algorithm that handles  non-stationarity. 
\end{abstract}

\section{Introduction}
Reinforcement Learning (RL)  \citep{sutton2018reinforcement},  coupled with powerful function approximators such as deep neural networks, has demonstrated great  potential in solving complicated sequential decision-making tasks such as games \citep{silver2016mastering, silver2017mastering, vinyals2019grandmaster} and robotic control  \citep{kober2013reinforcement,gu2017deep, akkaya2019solving, andrychowicz2020learning}. 
Most of these empirical successes are driven by deep policy optimization methods such as trust region policy optimization (TRPO) \citep{schulman2015trust}  and proximal policy optimization (PPO) \citep{schulman2017proximal}, whose performance has been extensively studied recently  \citep{agarwal2019theory, liu2019neural, shani2020adaptive,mei2020global,  cen2020fast}. 

While classical RL assumes that an  agent interacts with a time-invariant (stationary)  environment, 
 when deploying RL to real-world applications, 
both the reward function and Markov transition kernel can be time-varying. 
For example, in autonomous driving \citep{sallab2017deep}, the vehicle needs to handle varying  conditions of weather and traffic. 
When the environment changes with time, the agent must quickly adapt its policy to maximize the expected 
total rewards  in the new environment. 
Meanwhile, another example of such a non-stationary scenario is when the environment is subject to adversarial manipulations, 
which is the case  of  adversarial attacks \citep{pinto2017robust, huang2017adversarial, pattanaik2017robust}. 
 In this situation, it is desired that the RL agent is robust against the malicious adversary.

Although there is a huge body of literature on developing provably efficient RL methods, most of the existing works focus on the classical stationary setting, with a few exceptions including \citet{JMLR:v11:jaksch10a,gajane2018sliding,cheung2019hedging,cheung2019non,cheung2020reinforcement,mao2020nearoptimal,ortner2020variational, domingues2020kernel, zhou2020nonstationary,touati2020efficient,wei2021non}. 
However, these works focus on value-based methods which only output deterministic/greedy policies, and mostly focus on the tabular case where the state space is finite. {It is well-known that {deterministic algorithms cannot tackle the adversarial rewards even in the bandit scenario \citep{lattimore2020bandit}. In contrast, the policy optimization (EXP3-type) algorithms are ready to tackle the adversarial rewards in both bandits \citep{lattimore2020bandit} and RL \citep{cai2019provably,efroni2020optimistic}. Thus, motivated by the empirical successes and theoretical advantages of policy optimization algorithms, it is desired to answer the following problem:}} 
\vspace{-4pt} 
\begin{center}
 How can we design a provably efficient policy optimization algorithm for non-stationary environment in the context of function approximation? 
\end{center}  

There are five intertwined challenges associated with this problem: (i) non-stationary rewards with bandit feedback or even the adversarial rewards with full-information feedback. See Remark \ref{remark:lower:bound} for the reason why  we only consider the full-information feedback adversarial rewards setting. (ii) non-stationary transition kernel, (iii) exploration-exploitation tradeoff that is inherent to online RL, (iv) incorporating  function approximation in the algorithm, and (v) characterizing the convergence and optimality of  policy optimization. 
Existing works merely address a subset of these five challenges and it remains open how to tackle all of them simultaneously. See Table \ref{tab:addlabel} for detailed comparisons.

 {\begin{table}[t]
  \centering
    \begin{tabular}{|c|c|c|c|c|c|}
    \hline 
    & \makecell{Non-stationary \\ Rewards} & \makecell{Adversarial \\ Rewards} & \makecell{Non-stationary \\ Transitions} & \makecell{Function \\ Approximation} & \makecell{PO}  \\
    \hline 
    (I)   &   \XSolidBrush    &   \XSolidBrush    &   \XSolidBrush    & \Checkmark  & \XSolidBrush  \\
    \hline 
    (II)   &   \XSolidBrush    &   \Checkmark    &   \XSolidBrush    & \Checkmark & \Checkmark \\
    \hline 
    (III)   & \Checkmark    &   \XSolidBrush    &    \Checkmark   & \XSolidBrush  &  \XSolidBrush \\
    \hline 
    (IV)   &  \Checkmark     &  \XSolidBrush    &   \Checkmark    & \Checkmark  & \XSolidBrush \\
    \hline 
    (V)   &  \XSolidBrush    &   \Checkmark   &    \XSolidBrush   & \XSolidBrush &  \Checkmark \\
    \hline 
    \rowcolor{blue!40} SW-LSVI-UCB     & \Checkmark     &  \XSolidBrush   &   \Checkmark    & \Checkmark & \XSolidBrush \\
    \hline 
    \rowcolor{blue!40} PROPO    &  \Checkmark     &   \Checkmark    &  \Checkmark     & \Checkmark & \Checkmark \\
    \hline 
    \end{tabular}%
    \caption[]{
    (I) represents previous optimistic-based value iteration algorithms \citep{jiang2017contextual, jin2019provably, wang2019optimism, zanette2020learning, wang2020provably, ayoub2020model,zhou2020provably} that successfully handles the exploration-exploitation tradeoff that is inherent to online RL and incorporate function approximation in the algorithm; (II) denotes the policy optimization algorithms \citep{cai2019provably,efroni2020optimistic} that can additionally tackle the adversarial rewards with full information feedback; (III) are the value-based algorithms \citep{cheung2019hedging,cheung2019non,mao2020nearoptimal} that can handles the non-stationary environments, and (IV) are the value-based algorithms \citep{zhou2020nonstationary,touati2020efficient,wei2021non} incorporating function approximation. (V) is the policy optimization algorithm in \citet{fei2020dynamic} that can tackle the adversarial rewards. However, they assumes the transition kernels are static, which circumvents the fundamental difficulties in non-stationarity RL. Furthermore, they only consider the tabular case.}
  \label{tab:addlabel}%
\end{table}

More importantly, these five challenges are coupled together, which requires  sophisticated algorithm design. In particular, due to challenges (i), (ii) and (iv), we need to track the non-stationary or even adversarial reward function  and transition kernel by function estimation  based on the feedbacks. The  estimated model is also time-varying and thus the corresponding  policy optimization problem (challenge (v)) has a non-stationary objective function. Moreover, to obtain sample efficiency, we need to strike a balance between exploration and exploitation in the policy update steps (challenge (iii)). 

In this work, we propose a \underline{p}eriodically \underline{r}estarted \underline{o}ptimistic \underline{p}olicy \underline{o}ptimization algorithm (PROPO) which successfully tackle the five challenges above.
Specifically, we focus on the model of episodic linear kernel MDP \citep{ayoub2020model,zhou2020provably} where both the reward and transition functions are parameterized by linear functions. 
Besides, we focus on the non-stationary  setting  and adopt the dynamic regret as the performance metric. 
Moreover, PROPO performs a policy evaluation step and a policy improvement step in each iteration. To handle challenges (i)--(iii),   we propose a novel optimistic policy evaluation method that incorporates the technique of sliding window to handle non-stationarity. 
Specifically, based on the non-stationary bandit feedbacks, we propose to estimate the time-varying model via a sliding-window-based least-squares regression problem, where we only keep a subset of  recent samples in regression. Based on the model estimator, we construct an optimistic  value function by implementing model-based policy evaluation and adding an  exploration bonus. 
Then, using such an optimistic value function as the update direction, in the policy improvement step, we propose to obtain a new policy by solving a Kullback-Leibler (KL) divergence regularized problem, which can be viewed as a mirror descent step. Moreover, as the underlying optimal policy is time-varying (challenge (v)), we additionally restart the policy periodically by setting it to uniform policy every $\tau$ episodes. 
The two  novel mechanisms, sliding window and periodic restart, respectively enable us to track the non-stationary MDP based on bandit feedbacks and handle the time-varying policy optimization problem. 

Finally, to further exhibit effect of these two mechanisms, we propose an optimism-based value iteration algorithm, dubbed as SW-LSVI-UCB, which only utilize the sliding window and does not restart the policy as challenge (v) disappears.

\subsection{Our Contributions}

Our contribution is four-fold. 
First, we propose PROPO,  a policy optimization algorithm designed for non-stationary linear kernel MDPs. This algorithm features two novel mechanisms, namely     sliding window  and periodic restart, and also incorporates linear function approximation and a bonus function to incentivize exploration. 
Second, we prove that PROPO achieves a dynamic regret sublinear in $d, \Delta, H, T$, and $P_T$. Here $d$ is  the feature dimension, $\Delta$ is  the total variation budget, $H$ is  the episode horizon, $T$ is  the total number of steps, and $P_T$ is the variation budget of adjacent optimal policies.
Third, to separately demonstrate the effect of sliding window, we propose a value-iteration algorithm, SW-LSVI-UCB,  which  adopts sliding-window-based regression to handle non-stationarity. Such an algorithm is shown to achieve a  similar sublinear 
dynamic regret. 
Finally, we establish a 
$\Omega(d^{5/6} \Delta^{1/3}H^{2/3}T^{2/3})$ lower bound on the dynamic regret,  which shows the (near-)optimality of the proposed algorithms.
To the best of our knowledge, PROPO is the first provably efficient policy optimization algorithm under the non-stationary environment. 

\begin{itemize}
  \item First, we propose PROPO,  a policy optimization algorithm designed for non-stationary linear kernel MDPs. This algorithm features two novel mechanisms, namely sliding window  and periodic restart, and also incorporates linear function approximation and a bonus function to incentivize exploration. To the best of our knowledge, PROPO is the first provably efficient policy optimization algorithm under the non-stationary transitions, for both non-stationary rewards and adversarial rewards with full-information feedback.
  \item We prove that PROPO achieves a dynamic regret sublinear in $d, \Delta, H, T$, and $P_T$. Here $d$ is  the feature dimension, $\Delta$ is  the total variation budget, $H$ is  the episode horizon, $T$ is  the total number of steps, and $P_T$ is the variation budget of adjacent optimal policies.
  \item To separately demonstrate the effect of sliding window, we propose a value-iteration algorithm, SW-LSVI-UCB,  which  adopts sliding-window-based regression to handle non-stationarity. Such an algorithm is shown to achieve a slightly better sublinear dynamic regret for non-stationary rewards and transitions. However, this approach is not suitable for handling adversarial rewards, unlike PROPO.
  \item We establish a 
  $\Omega(d^{5/6} \Delta^{1/3}H^{2/3}T^{2/3})$ lower bound on the dynamic regret,  which shows the (near-)optimality of the proposed algorithms.
\end{itemize}
We provide simulated experiments to verify our theory. Based on our theory and experimental results, our work conveys two key messages: (i) policy optimization is provably efficient for non-stationary environments; (ii) We should opt for value-based algorithms when dealing with environments subject to minor reward fluctuations. Conversely, policy optimization approaches are more suitable for scenarios characterized by significant reward shifts.

\subsection{Related Work}

Our work adds to the vast body of existing literature on decision-making in non-stationary environments.
  As a special case of MDP problems with unit horizon, bandit problems have been the subject of intense recent interest. See \citet{besbes2014stochastic,besbes2019optimal,russac2019weighted,cheung2019hedging,chen2019new} and the references therein for details. For the more challenging non-stationary MDP problems,  \citet{JMLR:v11:jaksch10a,gajane2018sliding,cheung2019hedging,cheung2019non,cheung2020reinforcement,fei2020dynamic,mao2020nearoptimal,ortner2020variational} provide results for the basic tabular case, where the state and action spaces are finite and small. Recently, \citet{domingues2020kernel} consider the non-stationary RL in continuous environments and proposes a kernel-based algorithm. More related works are \citet{zhou2020nonstationary,touati2020efficient}, which study non-stationary linear MDPs with the periodic restart technique, but their setting is not directly comparable with ours since linear MDPs cannot imply linear kernel MDPs. Moreover, \citet{zhou2020nonstationary,touati2020efficient} do not incorporate policy optimization methods, which are more difficult because we need to handle the variation of the optimal policies of adjacent episodes and value-based methods only need to handle the non-stationarity drift of reward functions and transition kernels. \citet{fei2020dynamic} also makes an attempt to investigate policy optimization algorithm for non-stationary environments. However, this work requires full-information feedback and only focuses on the tabular MDPs with time-varying reward functions and time-invariant transition kernels. Finally, the work of \citet{cheung2019hedging} applies the sliding-window technique to non-stationary bandits, while \citet{cheung2020reinforcement} extends this approach to non-stationary infinite-horizon average-reward MDPs. We emphasize that infinite-horizon average-reward RL and finite-horizon RL are not directly comparable to our work. Due to the different settings we study, \citet{cheung2020reinforcement} strategically divides the learning process into several episodes and defines each episode as a sliding window, whereas our work considers several adjacent episodes as a sliding window. Furthermore, we highlight that our work can handle linear kernel models, while \citet{cheung2020reinforcement} is restricted to the tabular case. Moreover, in addition to the model-based algorithm similar to \citet{cheung2020reinforcement}, our work introduces a policy optimization algorithm capable of handling adversarial rewards with full-information feedback. Finally, we provide a comprehensive comparison with \citet{wei2021non}, who developed a unified framework for non-stationary problems. Though \citet{wei2021non} does not list linear kernel MDPs as an examples, their framework could potentially accommodate linear kernel MDPs with stochastic rewards and achieve a tighter regret bound.  Compared with this work, our work offers several distinct advantages: (i) our approach extends to adversarial reward settings, showcasing the inherent flexibility of policy optimization algorithms; (ii) we provide rigorous time complexity guarantees and achieve better space complexity, as our method eliminates the need to maintain multiple simultaneous instances; and (iii) our algorithms are more straightforward to implement in practice.

Broadly speaking, our work is also related to the works on adversarial MDPs and policy optimization. In particular, \citep{even2009online,neu2010online,neu2012adversarial,zimin2013online,rosenberg2019online,jin2019learning} studies the static regret under the adversarial case, while we focus on the more challenging dynamic regret under the non-stationary dynamics and adversarial rewards. It is well-known that value-based algorithms \citep[more generally, deterministic algorithms;][]{bradtke1996linear,jiang2017contextual, jin2019provably, wang2019optimism, zanette2020learning, wang2020provably, ayoub2020model,zhou2020provably} cannot solve adversarial MDPs, which, together with the amazing performance of policy optimization algorithms in practice, motivate us to study policy optimization methods. As proved in  \citet{yang2019global,agarwal2019theory,liu2019neural,wang2019neural}, policy optimization enjoys computational efficiency. Recently \citet{cai2019provably,efroni2020optimistic,  agarwal2020pc} proposed optimistic policy optimization methods which simultaneously attain computational efficiency and sample efficiency. 

\subsection{Notation}
We denote by $\| \cdot \|_2$ the $\ell_2$-norm of a vector or the spectral norm of a matrix. Furthermore, for a positive definite matrix $A$, we denote by $\|x\|_A$ the matrix norm $\sqrt{x^\top Ax}$ of a vector $x$. For any number $a$, we denote by $\lceil a \rceil$ the smallest integer that is no smaller than $a$, and $\lfloor a \rfloor$ the largest integer that is no larger than $a$. Also, for any two numbers $a$ and $b$, let $a \vee b = \max\{a, b\}$ and $a \wedge b = \min\{a, b\}$. For some positive integer $K$, $[K]$ denotes the index set $ \{1, 2, \cdots, K\}$. When logarithmic factors are omitted, we use $\tilde{O}$ to denote function growth.

\section{Preliminaries}

\subsection{Non-stationary MDPs} \label{sec:pre:nonstationary}
An episodic non-stationary MDP is defined by a tuple $ (\cS, \cA, H, P, r)$, where $\cS$ is a state space, $\cA$ is an action space, $H$ is the length of each episode, $P = \{P_h^k\}_{(k, h) \in [K] \times [H]}$, $r =  \{r_h^k\}_{(k, h) \in [K] \times [H]}$, where $P_h^k :  \cS \times \cA \times \cS \rightarrow [0,1]$ is the probability transition kernel at the $h$-th step of the $k$-th episode, and $r_h^k : \cS \times \cA \rightarrow [0,1]$ is the reward function at the $h$-th step of the $k$-th episode.
We consider an agent which iteratively interacts with a non-stationary MDP in a sequence of $K$ episodes. At the beginning of the $k$-th episode, the initial state $s_1^k$ is a fixed state $s_1$\footnote{Our subsequent analysis can be generalized to the setting where the initial state is sampled from a fixed distribution.}, and the agent determines a policy $\pi^k = \{\pi_h^k\}_{h=1}^H$. Then, at each step $h \in [H]$, the agent observes the state $s_h^k$, takes an action following the policy $a_h^k \sim \pi_h^k(\cdot \,|\, s_h^k)$ and receives a bandit feedback $r_h^k(s_h^k,a_h^k)$ or the full-information feedback $r_h^k(\cdot, \cdot)$\footnote{Bandit feedback means that the agent only observes the values of reward function $r_h^k(s_h^k,a_h^k)$ when visited state-action pair $(s_h^k,a_h^k)$. Full-information feedback represents that the entire reward functions $r_h^k$ is revealed after the $h$-th step of $k$-th episode.}. Meanwhile,  the MDP evolves into the next state $s_{h+1}^k \sim P_h^k(\cdot \,|\, x_h^k,a_h^k)$. The $k$-th episode ends at state $s_{H+1}^k$, when this happens, no control action is taken and reward is equal to zero.  
We define the state and state-action value functions of policy $\pi = \{\pi_h\}_{h=1}^H$ recursively via the following Bellman equation:
\# \label{eq:bellman}
Q_h^{\pi,k}(s,a) = r_h^k(s,a) + (\PP_h^k V_{h+1}^{\pi,k})(s, a), \quad V_h^{\pi,k}(s) = \la Q_h^{\pi,k}(s,\cdot ), \pi_h(\cdot \, | \, s) \ra_{\cA}, \quad V^{\pi, k}_{H + 1} = 0,
\#
where $\PP_h^k$ is the operator form of the transition kernel $P_h^k(\cdot\,|\,\cdot,\cdot)$, which is defined as
\# \label{eq:def:p}
(\PP_h^k f)(s, a) = \EE[f(s')\,|\,s' \sim P_h^k(s'\,|\,s,a)]
\#
for any function $f : \cS \rightarrow \RR$. Here $\la \cdot , \cdot \ra_{\cA}$ denotes the inner product over $\cA$.


In the literature of optimization and reinforcement learning, the performance of the agent is measured by its dynamic regret, which  measures the difference between the agent's policy and the benchmark policy $\pi^{*} = \{\pi^{*, k}\}_{k = 1}^K$. Specifically, the dynamic regret is defined as
\# \label{eq:def:regret:dyn}
\text{D-Regret}(T, \pi^*) = \sum_{k=1}^K \bigl( V^{\pi^{*, k},k}_1(s^k_1) - V^{\pi^k,k}_1(s^k_1)  \bigr) ,
\#
where $T=HK$ is the number of steps taken by the agent and $\pi^{*, k}$ is the benchmark policy for episode $k$. 
It is worth mentioning that when the benchmark policy is the optimal policy of each individual episode, that is, $\pi^{*, k} = \argmax_{\pi}V_1^{\pi, k}(s_1^k)$, the dynamic regret reaches the maximum, and this special case is widely considered in previous works \citep{cheung2020reinforcement, mao2020nearoptimal, domingues2020kernel}. Throughout the paper, when $\pi^*$ is clear from the context, we may omit $\pi^*$ from $\text{D-Regret}(T, \pi^*)$.





\subsection{Model Assumptions}
We focus on the linear setting of the Markov decision process, where the reward functions and transition kernels are assumed to be linear. We formally make the following assumption. 



\begin{assumption}[Non-stationary Linear Kernel MDP] \label{assumption:linear:mdp}
	MDP $(\cS, \cA, H, P, r)$ is a linear kernel MDP with known feature maps $\phi : \cS \times \cA \rightarrow \RR^d$ and $\psi : \cS \times \cA \times \cS \rightarrow \RR^d$, if for any $(k, h) \in [K] \times [H]$, there exist unknown vectors $\theta_h^k \in \RR^d$ and $\xi_h \in \RR^d$, such that 
	\$
	r_h^k(s,a) = \phi(s,a)^\top \theta_h^k, \qquad P_h^k(s' \,|\, s, a) = \psi(s, a, s')^\top \xi_h^k 
	\$
	for any $(s, a, s') \in \cS \times \cA \times \cS$. Without loss of generality, we assume that
	\$
	\| \phi(s, a) \|_2 \le 1, \qquad \|\theta_h^k\|_2 \le \sqrt{d}, \qquad  \| \xi_h^k \|_2 \le \sqrt{d}
	\$
	for any $(k, h) \in [K] \times [H]$. Moreover, we assume that
	\$
	\int_{\cS} \| \psi(s, a, s')\|_2\, \mathrm{d}s'  \le \sqrt{d}
	\$
	for any $(s, a) \in \cS \times \cA$. 
\end{assumption}

Our assumption consists of two parts. One is about reward functions, which follows the setting of linear bandits \citep{abbasi2011improved, agrawal2013thompson,besbes2014stochastic, besbes2015non, cheung2019hedging,cheung2019learning}. The other part is about transition kernels. As shown in \cite{cai2019provably, ayoub2020model, zhou2020provably}, linear kernel MDPs as defined above cover several other MDPs studied in previous works, as special cases. For example, tabular MDPs with canonical basis \citep{cai2019provably, ayoub2020model, zhou2020provably}, feature embedding of transition models \citep{yang2019reinforcement} and linear combination of base models \citep{modi2020sample} are special cases. However, it is worth mentioning that \cite{jin2019provably, yang2019sample} studied another ``linear MDPs'',  which assumes the transition kernels can be represented as $\PP_h(s' \,|\,s, a) = \psi'(s, a)^\top \mu_h(s')$ for any $h \in [H]$ and $(s, a, s') \in \cS \times \cA \times \cS$.  Here $\psi'(\cdot, \cdot)$ is a known feature map and $\mu(\cdot)$ is an
unknown measure. It is worth noting that linear MDPs studied in our paper and linear MDPs \citep{jin2019provably, yang2019sample} are two different classes of MDPs since their feature maps $\psi(\cdot, \cdot, \cdot)$ and $\psi'(\cdot, \cdot)$ are different and neither class of MDPs includes the other.


To facilitate the following analysis, we denote by $P_h^{k, \pi}$ the Markov kernel of policy at the $h$-step of the $k$-th episode, that is, for $s \in \cS$, $P_h^{k, \pi}(\cdot \,|\, s) = \sum_{a \in \cA}P_h^k(\cdot \,|\, s, a) \cdot \pi_h(a \,|\, s)$. Also, we define 
\$
\| \pi_h - \pi'_h \|_{\infty, 1} &= \max_{s \in \cS}\| \pi_h(\cdot \,|\, s) - \pi'_h(\cdot \,|\, s) \|_1,   \\
\| P_h^{k, \pi} - P_h^{k, \pi'} \|_{\infty, 1} &= \max_{s \in \cS}\| P_h^{k, \pi}(\cdot \,|\, s) - P_h^{k, \pi'}(\cdot \,|\, s) \|_1 .
\$

Next, we introduce several measures of change in MDPs. First, we denote by $P_T$ the total variation in the benchmark policies of adjacent episodes:
	\# \label{eq:def:pt}
	P_T = \sum_{k=1}^K\sum_{h=1}^H \| \pi^{*, k}_h - \pi_h^{*, k -1} \|_{\infty, 1} ,
	\#
	where we choose $\pi_h^{*, 0} = \pi_h^{*, 1}$ for any $h \in [H]$.


Next, we assume the drifting environment \citep{besbes2014stochastic, besbes2015non, cheung2019hedging, russac2019weighted}, that is, $\theta_h^k$ and $\xi_h^k$ can change over different indexes $(k, h)$, with the constraint that the sum of the Euclidean distances between consecutive $\theta_h^k$ and $\xi_h^k$ is bounded by variation budgets $B_T$ and $B_P$. 
Formally, we impose the following assumption. 
\begin{assumption} \label{assumption:tv}
	There exists some variation budget $\Delta_{T} = \cO(T^{\nu_1})$ and $\Delta_{P} = \cO(T^{\nu_2})$  for some $\nu_1 \in (0, 1]$ and $\nu_2 \in (0, 1)$, such that 
	\# \label{eq:def:tv}
	\sum_{h=1}^H\sum_{k=1}^K \| \theta_h^{k-1} - \theta_h^k \|_2 \le \Delta_{T}, \quad  \sum_{h=1}^H\sum_{k=1}^K \| \xi_h^{k-1} - \xi_h^k \|_2 \le \Delta_{P}, \quad \Delta = \Delta_{T} + \Delta_{P},
	\#
	where $H$ is the length of each episode, $K$ is the total number of episodes, and $T = HK$ is the total number of steps taken by the agent. Here $\Delta$ is the total variation budget, which quantifies the non-stationarity of a linear kernel MDP. We remark that when the reward functions are adversarially chosen, $\nu_1$ can be 1. 
\end{assumption}

\section{Minimax Lower Bound}
In this section, we provide the information-theoretical lower bound result. The following theorem shows a minimax lower bound of dynamic regret for any algorithm to learn non-stationary linear kernel MDPs.

\begin{theorem}[Minimax lower bound]\label{thm:lower:bound}
	Fix $\Delta > 0$, $H > 0$, $d \geq 2$, and $T = \Omega(d^{5/2}\Delta H^{1/2})$. 
	Then, there exists a non-stationary linear kernel MDP with a $d$-dimensional feature map and maximum total variation budget $\Delta$, such that,  
	\$
	\min_{\mathbb{A}}\max_{\pi^*}\text{D-Regret}(T, \pi^*) \geq \Omega(d^{5/6} \Delta^{1/3}H^{2/3}T^{2/3}),
	\$
	where $\mathbb{A}$ denotes the learning algorithm which has only access to the bandit feedback.
\end{theorem} 

\begin{proof}[Proof sketch]
	As mentioned above, we only need to establish the lower bound of the dynamic regret when the benchmark policy is the optimal policy of each individual episode. The proof of the lower bound relies on the construction of a hard-to-learn non-stationary linear kernel MDP instance. To handle the non-stationarity, we need to divide the total $T$ steps into $L$ segments, where each segment contains $T_0 = \lfloor \frac{T}{L} \rfloor$ steps and has $K_0 = \lfloor \frac{K}{L} \rfloor$ episodes.   Within each segment, the construction of the MDP is similar to the hard-to-learn instance constructed in stationary RL problems \citep{JMLR:v11:jaksch10a, lattimore2012pac,osband2016lower}. Then, we can derive a lower bound of $\Omega(dH\sqrt{T_0})$ for the stationary RL problem. Meanwhile, the transition kernel of this hard-to-learn MDP changes abruptly between two consecutive segments, which forces the agent to learn a new stationary MDP in each segment. Finally, by optimizing $L$ subject to the total budget constraint, we obtain the lower bound of $\Omega(d^{5/6} \Delta^{1/3}H^{2/3}T^{2/3})$. See Appendix \ref{appendix:proof:lower:bound} for details.
\end{proof}

\begin{remark} \label{remark:lower:bound}
	This theorem states that the non-stationary MDP with $\Delta = \Omega(T)$ is unlearnable. Thus, when reward functions are adversarially chosen, there is no algorithm that can efficiently solve the underlying MDPs with only the bandit feedback. However, if we can obtain the full-information feedback, solving non-stationary MDPs with adversarial rewards is possible. In Section \ref{sec:alg}, we will show that, under the non-stationary environments, our policy optimization algorithm can tackle the adversarial rewards with full-information feedbacks.
\end{remark}

\section{Algorithm and Theory} \label{sec:alg}

\subsection{\algoplus}
Now we present Periodically Restarted Optimistic Policy Optimization (PROPO) in Algorithm \ref{alg:2}, which includes a policy improvement step and a policy evaluation step.

\vspace{4pt}
\noindent
{\bf Policy Improvement Step.}
At $k$-th episode, \algo  updates $\pi^k = \{\pi_h^k\}_{h=1}^H$ according to $\pi^{k-1} = \{\pi_h^{k-1}\}_{h=1}^H$. Motivated by the policy improvement step in NPG \citep{kakade2002natural}, TRPO \citep{schulman2015trust}, and PPO \citep{schulman2017proximal}, we consider the following policy improvement step.
\# \label{eq:update:pi}
\pi^k = \argmax_{\pi}L_{k-1}(\pi),
\#  
where $L_{k-1}(\pi)$ is defined as
\# \label{eq:def:l}
 L_{k-1}(\pi) &=   \EE_{\pi^{k-1}} \biggl[\sum_{h=1}^H  \la Q^{k-1}_h(s_h, \cdot), \pi_h(\cdot\,|\, s_h) -  \pi^{k-1}_h(\cdot\,|\, s_h)   \ra  \biggr]  \\
  & \qquad \quad - \alpha^{-1} \cdot    \EE_{\pi^{k-1}} \biggl[\sum_{h=1}^H { \rm {KL}} \bigl( \pi_h(\cdot\,|\, s_h) \,\big\|\, \pi^{k-1}_h(\cdot\,|\, s_h) \bigr)   \biggr] , \notag 
\#
where $\alpha > 0$ is a stepsize and $Q_h^{k-1}$ which is obtained in Line \ref{line:Q} of Algorithm \ref{alg:3} is the estimator of $Q_h^{\pi^{k-1}, k-1}$. Here the expectation $\EE_{\pi^{k - 1}}$ is taken over the random state-action pairs $\{(s_h,a_h)\}_{h=1}^H$, where the initial state $s_1 = s_1^k$, the distribution of action $a_h$ follows $\pi(\cdot \,|\, s_h)$, and the distribution of the next state $s_{h+1}$ follows the transition dynamics $P_h^k(\cdot \,|\, s_h,a_h)$. 
Such a policy improvement step can also be regarded as one iteration of  infinite-dimensional mirror descent \citep{nemirovsky1983problem, liu2019neural, wang2019neural}. 

By the optimality condition, policy update in \eqref{eq:update:pi} admits a closed-form solution
\# \label{eq:closed:form}
\pi^{k}_h(\cdot\,|\,s) \propto \pi^{k-1}_h(\cdot\,|\,s) \cdot \exp\{\alpha\cdot Q^{k-1}_h(s, \cdot)\}
\#
for any $s \in \cS$ and $(k, h) \in [K] \times [H]$. 

\vspace{4pt}
\noindent
{\bf Policy Evaluation Step.}
At the end of the $k$-th episode, \algo  evaluates the policy $\pi^k$ based on the $(k - 1)$ historical trajectories. Then, we show the details of estimating the reward functions and transition kernels, respectively.

\vspace{4pt}
\noindent
{\bf (i) Estimating Reward.}
To estimate the reward functions, we use the sliding window regularized least squares estimator (SW-RLSE) \citep{garivier2011upper,cheung2019hedging, cheung2019learning}, which is a key tool in estimating the unknown parameters online. At $h$-th step of $k$-th episode, we aim to estimate the unknown parameter $\theta_h^k$ based on the historical observation $\{(s_h^\tau, a_h^\tau), r_h^\tau(s_h^\tau, a_h^\tau) \}_{\tau = 1}^{k-1}$.  The design of SW-RLSE is based on the ``forgetting principle" \citep{garivier2011upper}, that is, under non-stationarity, the historical observations far in the
past are obsolete, and they do not contain relevant information for the evaluation of the current policy. Therefore, we could estimate $\theta_h^k$ using only $\{(s_h^\tau, a_h^\tau), r_h^\tau(s_h^\tau, a_h^\tau) \}_{\tau = 1 \vee (k-w)}^{k-1}$, the observations during the sliding window $1\vee(k-w)$ to $k - 1$, 
\# \label{eq:least-square}
\hat{\theta}_h^k = \argmin_{\theta} \biggl( \sum_{\tau = 1\vee (k-w)}^{k-1}  \bigl(r_h^\tau(s_h^\tau, a_h^\tau) - \phi(s_h^k, a_h^k)^\top\theta \bigr)^2 + \lambda \cdot  \|\theta\|_2^2 \biggr) ,
\#
where $\lambda$ is the regularization parameter and $w$ is the length of a sliding window. By solving~\eqref{eq:least-square}, we obtain the estimator of $\theta_h^k$:
\# \label{eq:estimate:reward}
&\hat{\theta}_h^k = (\Lambda_h^{k})^{-1}\sum_{\tau = 1\vee (k-w)}^{k-1}  \phi(s_h^\tau,a_h^\tau) r_h^\tau(s_h^\tau,a_h^\tau),  \\ 
&\text{where }  \Lambda_h^{k} = \sum_{\tau = 1\vee (k-w)}^{k-1}  \phi(s_h^\tau,a_h^\tau) \phi(s_h^\tau,a_h^\tau)^\top + \lambda I_d . \notag
\# 

\vspace{4pt}
\noindent
{\bf (ii) Estimating Transition.}
Similar to the estimation of reward functions, for any $(k, h) \in [K] \times [H]$, we define the sliding window empirical mean-squared Bellman error (SW-MSBE)  as
\$
M_h^k(\xi) = \sum_{\tau = 1\vee (k-w)}^{k-1}\bigl(V_{h+1}^\tau(s_{h+1}^{\tau}) - {\eta}_h^\tau(s_h^\tau, a_h^\tau)^\top \xi \bigr)^2 ,
\$
where we denote $\eta_h^\tau(\cdot,\cdot)$ as
\# \label{eq:def:eta}
\eta_h^\tau(\cdot,\cdot) = \int_\cS \psi(\cdot, \cdot, s') \cdot V_{h+1}^{\tau}(s') \mathrm{d} s' .
\#
By Assumption \ref{assumption:linear:mdp}, we have
\$
\| \eta_h^k (\cdot, \cdot) \|_2 \le H\sqrt{d}
\$
 for any $(k, h) \in [K] \times [H]$. Then we estimate $\xi_h^k$ by solving the following problem:
 \# \label{eq:estimate:xi}
 \hat{\xi}_h^k = \argmin_{w \in \RR^d}(M_h^k(w) + \lambda' \cdot \| w \|_2^2) ,
 \#
 where $\lambda'$ is the regularization parameter. 
 By solving \eqref{eq:estimate:xi}, we obtain 
 \# \label{eq:def:w}
 &\hat{\xi}_h^k = (A_h^k)^{-1}\biggl(\sum_{\tau=1 \vee (k - w)}^{k-1} \eta_h^{\tau}(s_h^\tau, a_h^\tau) \cdot V_{h+1}^\tau(s_{h+1}^\tau) \biggr)   \\
 &\text{where } A_h^k =  \sum_{\tau = 1 \vee (k - w)}^{k-1} \eta_h^{\tau}(s_h^\tau, a_h^\tau)\eta_h^{\tau}(s_h^\tau, a_h^\tau)^\top + \lambda' I_d. \notag
 \#
 The policy evaluation step is iteratively updating the estimated Q-function $Q^k = \{Q_h^k\}_{h=1}^H$ by
\# \label{eq:evaluation}
&Q_h^k (\cdot , \cdot) = \min\{\phi(\cdot, \cdot)^\top\hat{\theta}_h^k + \eta_h^k(\cdot, \cdot)^\top \hat{\xi}_h^k + B_h^k(\cdot, \cdot) + \Gamma_h^k(\cdot, \cdot), H - h +1\}^+ , \\
& V_h^k(s) = \la Q_h^k(s,\cdot), \pi_h^k(\cdot | s) \ra _\cA  \notag
\#
in the order of  $h = H , H-1, \cdots, 1$. Here bonus functions $B_h^k(\cdot, \cdot) : \cS \times \cA \rightarrow \RR^+$ and $\Gamma_h^k(\cdot, \cdot) : \cS \times \cA \rightarrow \RR^+$ are used to quantify the uncertainty in estimating reward $r_h^k$ and quantity $\PP_h^{k}V_{h+1}^k$ respectively, defined as
\# \label{eq:bonus}
B_h^k(\cdot, \cdot) = \beta \bigl( \phi(\cdot, \cdot)^\top(\Lambda_h^k)^{-1}\phi(\cdot, \cdot) \bigr) ^{1/2} , \quad \Gamma_h^k(\cdot, \cdot) = \beta' \bigl( \eta_h^k(\cdot, \cdot)^\top(A_h^k)^{-1}\eta_h^k(\cdot, \cdot) \bigr) ^{1/2} ,
\#
where $\beta > 0$ and $\beta' > 0$ are parameters depending on $d$, $H$ and $K$, which are specified in Theorem~\ref{thm:regret:dyn}.

To handle the non-stationary drift incurred by the different optimal policies in different episodes, Algorithm \ref{alg:2} also includes a periodic restart mechanism, which resets the policy estimates every $\tau$ episodes. We call the $\tau$ episodes between every two resets a segment. In each segment, each episode is approximately the same as the first episode, which means that we can regard it as a stationary MDP. Then we can use the method of solving the stationary MDP to analyze each segment with a small error, and finally combine each segment and choose the value of $\tau$ to get the desired result. Such a restart mechanism is widely used in RL \citep{auer2009near, ortner2020variational}, bandits \citep{besbes2014stochastic, zhao2020simple}, and non-stationary optimization \citep{besbes2015non, jadbabaie2015online}.

The pseudocode of the \algoplus algorithm is given in Algorithm \ref{alg:2}.

\noindent{\bf Learning Adversarial Rewards.}
We remark that PROPO is an EXP3-type algorithm, and thus can naturally tackle the adversarial rewards with full information. Specifically, we only need to replace the previous policy evaluation step in \eqref{eq:evaluation} by
\# 
&Q_h^k (\cdot , \cdot) = \min\{r_h^k(\cdot, \cdot) + \eta_h^k(\cdot, \cdot)^\top \hat{\xi}_h^k + \Gamma_h^k(\cdot, \cdot), H - h + 1\}^+ , 
\#
where $\eta_h^k$, $\hat{\xi}_h^k$, and $\Gamma_h^k$ are defined in \eqref{eq:def:eta}, \eqref{eq:def:w}, and \eqref{eq:bonus}, respectively. For completeness, we provide the pseudocode in Appendix \ref{appendix:alg:adv}. We remark that the value-based algorithm cannot tackle the adversarial rewards with full-information feedback with such a simple modification.

\begin{algorithm}[H] 
	\caption{Periodically Restarted Optimistic Policy Optimization (PROPO)}
	\begin{algorithmic}[1]
		\REQUIRE Reset cycle length $\tau$, sliding window length $w$, stepsize $\alpha$, regularization factors $\lambda$ and $\lambda'$, and bonus multipliers $\beta$ and $\beta'$.
		\STATE Initialize $\{\pi_h^0(\cdot \,|\, \cdot)\}_{h=1}^H$ as uniform distribution policies, $\{Q_h^0(\cdot, \cdot)\}_{h=1}^H$ as zero functions.
		
		\FOR{$k = 1, 2, \dots, K$}
		\STATE Receive the initial state $s_1^k$.
		\IF{$k \mod \tau = 1$}
		\STATE Set $\{Q_h^{k-1}\}_{h \in [H]}$ as zero functions and $\{\pi_h^{k-1}\}_{h \in [H]}$ as uniform distribution on $\cA$.
		\ENDIF
		\FOR{$h = 1, 2,  \dots ,H$}      
		\STATE $\pi^k_h(\cdot\,|\,\cdot) \propto \pi^{k-1}_h(\cdot\,|\,\cdot) \cdot \exp\{\alpha\cdot Q^{k-1}_h(\cdot,\cdot)\}$.
		\STATE Take action $a_h^k \sim \pi_h^k(\cdot \,|\, s_h^k)$.
		\STATE Observe the reward $r_h^k(s_h^k, a_h^k)$ and receive the next state $s_{h+1}^k$.
		\ENDFOR
	    \STATE Compute $Q_h^k$ by SWOPE$(k, \{\pi_h^k\}, \lambda, \lambda', \beta, \beta')$ (Algorithm \ref{alg:3}).
		\ENDFOR
	\end{algorithmic} \label{alg:2}
\end{algorithm}

\begin{algorithm}[H] 
	\caption{Sliding Window Optimistic Policy Evaluation (SWOPE)}
	\begin{algorithmic}[1]
		\REQUIRE Episode index $k$, policies $\{\pi_h\}$, regularization factors $\lambda$ and $\lambda'$, and bonus multipliers $\beta$ and $\beta'$.
		\STATE Initialize $V^k_{H+1}$ as a zero function.
		\FOR{$h = H, H-1, \dots ,0$}  
		\STATE $\eta_h^k(\cdot,\cdot) = \int_\cS \psi(\cdot, \cdot, s') \cdot V_{h+1}^{k}(s') \mathrm{d} s' $.
		\STATE $\Lambda_h^{k} = \sum_{\tau = 1\vee(k - w)}^{k-1}  \phi(s_h^\tau,a_h^\tau) \phi(s_h^\tau,a_h^\tau)^\top + \lambda I_d$ .
		\STATE $\hat{\theta}_h^k = (\Lambda_h^{k})^{-1}\sum_{\tau = 1\vee(k-w)}^{k-1}  \phi(s_h^\tau,a_h^\tau) r_h^\tau(s_h^\tau,a_h^\tau)$.
		\STATE $A_h^k =  \sum_{\tau = 1 \vee (k - w)}^{k-1} \eta_h^{\tau}(s_h^\tau, a_h^\tau)\eta_h^{\tau}(s_h^\tau, a_h^\tau)^\top + \lambda' I_d.$.
		\STATE $\hat{\xi}_h^k = (A_h^k)^{-1}\bigl(\sum_{\tau=1 \vee (k - w)}^{k-1} \eta_h^{\tau}(s_h^\tau, a_h^\tau) \cdot V_{h+1}^\tau(s_{h+1}^\tau) \bigr)$.
		\STATE $B_h^k(\cdot, \cdot) = \beta \bigl( \phi(\cdot, \cdot)^\top(\Lambda_h^k)^{-1}\phi(\cdot, \cdot) \bigr) ^{1/2}$.
		\STATE $\Gamma_h^k(\cdot, \cdot) = \beta' \bigl( \eta_h^k(\cdot, \cdot)^\top(A_h^k)^{-1}\eta_h^k(\cdot, \cdot) \bigr) ^{1/2}$.
		\STATE $Q_h^k (\cdot , \cdot) = \min\{\phi(\cdot, \cdot)^\top\hat{\theta}_h^k + \eta_h^k(\cdot, \cdot)^\top \hat{\xi}_h^k + B_h^k(\cdot, \cdot) + \Gamma_h^k(\cdot, \cdot), H - h +1\}^+$.  \label{line:Q} 
		\STATE$V_h^k(s) = \la Q_h^k(s,\cdot), \pi_h^k(\cdot | s) \ra _\cA$.
		\ENDFOR
	\end{algorithmic} \label{alg:3}
\end{algorithm}

\subsection{\algovi}
In this subsection, we present the details of Sliding Window Least-Square Value Iteration with UCB (SW-LSVI-UCB) in Algorithm \ref{alg:4}.

Similar to Least-Square Value Iteration with UCB (LSVI-UCB) in \cite{jin2019provably}, SW-LSVI-UCB is also an optimistic modification of Least-Square Value Iteration (LSVI) \citep{bradtke1996linear}, where the optimism is realized by Upper-Confidence Bounds (UCB). Specifically, the optimism is achieved due to the bonus functions $B_h^k$ and $\Gamma_h^k$, which quantify the uncertainty of reward functions and transition kernels, respectively. It is worth noting that in order to handle the non-stationarity, SW-LSVI-UCB also uses the sliding window method \citep{garivier2011upper,cheung2019hedging, cheung2019learning}.

 In detail, at $k$-th episode, SW-LSVI-UCB consists of two steps. In the first step, by solving the sliding window least-square problems \eqref{eq:least-square} and \eqref{eq:estimate:xi}, SW-LSVI-UCB updates the parameters $\Lambda_h^k$ in \eqref{eq:estimate:reward}, $\hat{\theta}_h^k$ in \eqref{eq:estimate:reward}, $A_h^k$ in \eqref{eq:def:w}, and $\hat{\xi}_h^k$ in \eqref{eq:def:w}, which are used to form the Q-function $Q_h^k$ . In the second step, SW-LSVI-UVB obtains the greedy policy with respect to the Q-function $Q_h^k$ gained in the first step. 
 See Algorithm \ref{alg:4} for more details.
 
The sliding-window technique we introduce is novel, as it marks the first use of this method in reinforcement learning for non-stationary linear mixture MDPs, with proven efficiency. The importance of this technique stems from its strong practical relevance. In real-world RL applications, algorithms often rely on only the most recent data stored in batches, reflecting the dynamic nature of data collection. As new information arrives in a streaming manner, recent experiences are integrated into decision-making, while older, potentially irrelevant data is discarded. This mirrors the approach used in deep reinforcement learning, where the 'experience replay buffer' technique—highlighted by works such as \cite{mnih2015human}, \cite{schaul2015prioritized}, and \cite{wang2016dueling}—updates the buffer in a sliding-window fashion, retaining a fixed number of recent experiences to ensure the agent learns from the most relevant data. Our work provides a rigorous theoretical foundation for using sliding windows in the challenging policy optimization problems. By aligning theory with the empirical success of these methods, we not only validate their practical significance but also pave the way for future theoretical research in offline deep RL.
\begin{algorithm}[t] 
	\caption{Sliding Window Least-Square Value Iteration with UCB (SW-LSVI-UCB)}
	\begin{algorithmic}[1]
		\REQUIRE Sliding window length $w$, regularization factors $\lambda$ and $\lambda'$, and bonus multipliers $\beta$ and $\beta'$.
		\STATE Initialize $\{\pi_h^0(\cdot \,|\, \cdot)\}_{h=1}^H$ as uniform distribution policies, $\{Q_h^0(\cdot, \cdot)\}_{h=1}^H$ as zero functions.
	   	\FOR{$k = 1, 2, \dots, K$}
	    \STATE Receive the initial state $s_1^k$.
		\STATE Initialize $V_{H+1}^k$ as a zero function.
		\FOR{$h = H, H-1, \dots ,0$}  
		\STATE $\eta_h^k(\cdot,\cdot) = \int_\cS \psi(\cdot, \cdot, s') \cdot V_{h+1}^{k}(s') \mathrm{d} s' $.
		\STATE $\Lambda_h^{k} = \sum_{\tau = 1\vee(k - w)}^{k-1}  \phi(s_h^\tau,a_h^\tau) \phi(s_h^\tau,a_h^\tau)^\top + \lambda I_d$ .
		\STATE $\hat{\theta}_h^k = (\Lambda_h^{k})^{-1}\sum_{\tau = 1\vee(k-w)}^{k-1}  \phi(s_h^\tau,a_h^\tau) r_h^\tau(s_h^\tau,a_h^\tau)$.
		\STATE $A_h^k =  \sum_{\tau = 1 \vee (k - w)}^{k-1} \eta_h^{\tau}(s_h^\tau, a_h^\tau)\eta_h^{\tau}(s_h^\tau, a_h^\tau)^\top + \lambda' I_d.$.
		\STATE $\hat{\xi}_h^k = (A_h^k)^{-1}\bigl(\sum_{\tau=1 \vee (k - w)}^{k-1} \eta_h^{\tau}(s_h^\tau, a_h^\tau) \cdot V_{h+1}^\tau(s_{h+1}^\tau) \bigr)$.
		\STATE $B_h^k(\cdot, \cdot) = \beta \bigl( \phi(\cdot, \cdot)^\top(\Lambda_h^k)^{-1}\phi(\cdot, \cdot) \bigr) ^{1/2}$.
		\STATE $\Gamma_h^k(\cdot, \cdot) = \beta' \bigl( \eta_h^k(\cdot, \cdot)^\top(A_h^k)^{-1}\eta_h^k(\cdot, \cdot) \bigr) ^{1/2}$.
		\STATE $Q_h^k (\cdot , \cdot) = \min\{\phi(\cdot, \cdot)^\top\hat{\theta}_h^k + \eta_h^k(\cdot, \cdot)^\top \hat{\xi}_h^k + B_h^k(\cdot, \cdot) + \Gamma_h^k(\cdot, \cdot), H - h +1\}^+$.  \label{line:Q2}
		\STATE$V_h^k(s) =  \max_{a}Q_h^k(s, a)$.
		\STATE $\pi_h^{k}(s) = \argmax_{a} Q_h^k (s, a)$.
		\ENDFOR
		\ENDFOR
	\end{algorithmic} \label{alg:4}
\end{algorithm}

\subsection{Regret Analysis}\label{sec:regret analysis}

In this subsection, we analyze the dynamic regret incurred by Algorithms \ref{alg:2} and \ref{alg:4} and compare the theoretical regret upper bounds derived for these two algorithms. To derive sharper dynamic regret bounds, we impose the following technical assumption. 
\begin{assumption} \label{assumption:orthonormal}
		There exists an orthonormal basis $\Psi = (\varPsi_1, \cdots, \varPsi_d)$ such that for any $(s, a) \in \cS \times \cA$, there exists a number $z \in \RR$ and an $i \in [d]$ satisfying $\phi(s, a) = z \cdot \varPsi_{i}$. We also assume the existence of another orthonormal basis $\Psi' = (\varPsi'_1, \cdots, \varPsi'_d)$ such that for any $(s, a, k, h) \in \cS \times \cA \times [K] \times [H],$ there exists a number $z' \in \RR$ and an $i' \in [d]$ satisfying $\eta_h^k(s, a) = z' \cdot \varPsi'_{i'} .$
\end{assumption}
	It is not difficult to show that this assumption holds in the tabular setting, where the state space and action space are finite. Similar assumption is also adopted by previous work in non-stationary optimization \citep{cheung2019hedging}. 
First, we establish an upper bound on the dynamic regret of PROPO. Recall that the dynamic regret is defined in \eqref{eq:def:regret:dyn} and $d$ is the dimension of the feature maps $\phi$ and $\psi$. Also, $|\cA|$ is the cardinality of $\cA$. We also define $\rho = \lceil K /\tau \rceil$ to be the number of restarts that take place in Algorithm \ref{alg:2}. 

\begin{theorem}[Upper bound for Algorithm \ref{alg:2}] \label{thm:regret:dyn}
	Suppose Assumptions \ref{assumption:linear:mdp}, \ref{assumption:tv}, and \ref{assumption:orthonormal} hold. Let $\tau = \Pi_{[1,K]}(\lfloor (\frac{T\sqrt{\log|\cA|}}{H(P_{T} + \sqrt{d}\Delta)})^{2/3} \rfloor )$, $\alpha = \sqrt{\rho \log | \cA | /(H^2K)}$ in \eqref{eq:def:l}, $w = \Theta(d^{1/3}\Delta^{-2/3}T^{2/3})$ in \eqref{eq:least-square}, $\lambda = \lambda' = 1$ in \eqref{eq:least-square} and \eqref{eq:evaluation}, $\beta = \sqrt{d}$ in \eqref{eq:bonus}, and $\beta' = C'\sqrt{dH^2\cdot \log(dT/\zeta)}$ in \eqref{eq:bonus}, where $C' > 1$ is an absolute constant and $\zeta \in (0, 1]$. We have 
	\$
	\text{D-Regret}(T) &\lesssim  d^{5/6}\Delta^{1/3}HT^{2/3} \cdot \log(dT/\zeta) \\
	& \,+ 
	\left\{
	\begin{array}{rcl}
		&\sqrt{H^3T\log|\cA|},            & {\text{if } 0 \le P_{T} + \sqrt{d}\Delta \le \sqrt{\frac{\log|\cA|}{K}} },\\
		&(H^2T\sqrt{\log|\cA|})^{2/3}(P_{T} + \sqrt{d}\Delta)^{1/3},       & {\text{if } \sqrt{\frac{\log|\cA|}{K}} \le P_{T} + \sqrt{d}\Delta \lesssim K\sqrt{\log|\cA|} },\\
		&H^2(P_{T} + \sqrt{d}\Delta),         & {\text{if } P_{T} + \sqrt{d}\Delta \gtrsim K\sqrt{\log|\cA|} },
	\end{array} \right.
	\$ 
	with probability at least $1-\zeta$.
\end{theorem}

\begin{proof}
	See Section \ref{sec:sketch} for a proof sketch and Appendix \ref{appendix:proof:regret:dyn} for a detailed proof. 
\end{proof}

{\color{red} }

Then we discuss the regret bound throughout three regimes of $ P_{T} + \sqrt{d}\Delta$:
\begin{itemize}
	\item Small $ P_{T} + \sqrt{d}\Delta$: when $0 \le  P_{T} + \sqrt{d}\Delta \le \sqrt{\frac{\log|\cA|}{K}}$, the restart period $\tau = K$, which means that we do not need to periodically restart in this case. Assuming that $\log | \cA | = \cO(d^{5/3}\Delta^{2/3}H^{-1} T^{1/3} )$,  Algorithm \ref{alg:2} attains a $\tilde{\cO}(d^{5/6}\Delta^{1/3}HT^{2/3})$ dynamic regret. 
	Combined with the lower bound established in Theorem \ref{thm:lower:bound}, our result matches the lower bound in $d$, $\Delta$ and $T$ up to logarithmic factors. Hence, we can conclude that Algorithm~\ref{alg:2} is a near-optimal algorithm;
	\item Moderate $ P_{T} + \sqrt{d}\Delta$: when $\sqrt{\frac{\log|\cA|}{K}} \le  P_{T} + \sqrt{d}\Delta \lesssim K\sqrt{\log|\cA|}$, the restart period $\tau = (\frac{T\sqrt{\log|\cA|}}{H(P_{T} + \sqrt{d}\Delta)})^{2/3} \in [2, K]$. Algorithm \ref{alg:3} incurs a $\tilde{\cO}( T^{2/3} )$ dynamic regret if $\Delta = \cO(1)$ and $P_T = \cO(1)$;
	\item Large $ P_{T} + \sqrt{d}\Delta$: when $ P_{T} + \sqrt{d}\Delta \gtrsim K\sqrt{\log|\cA|}$, restart period $\tau = K$. Since the model is highly non-stationary, we only obtain a linear regret in $T$.
\end{itemize}

Notably, when the reward functions' total variation budget $\Delta_T$ defined in Assumption \ref{assumption:tv} is linear in $T$, we know $\Delta \ge \Delta_T = \Omega(T)$, then the regret bound in Theorem \ref{thm:regret:dyn} becomes vacuous. Under this scenario, if we have access to the full-information feedback of rewards, we can obtain the following theorem.

\begin{theorem}[Informal] \label{thm:adv:informal}
	For the MDPs with adversarial rewards and non-stationary transitions, if the full-information feedback of rewards is available, then the dynamic regret of PROPO is bounded by 
	\$
	\text{D-Regret}(T) &\lesssim  d^{5/6}\Delta_{p}^{1/3}HT^{2/3} \cdot \log(dT/\zeta) \\
	& \,+ 
	\left\{
	\begin{array}{rcl}
		&\sqrt{H^3T\log|\cA|},            & {\text{if } 0 \le P_{T} + \sqrt{d}\Delta_P \le \sqrt{\frac{\log|\cA|}{K}} },\\
		&(H^2T\sqrt{\log|\cA|})^{2/3}(P_{T} + \sqrt{d}\Delta_P)^{1/3},       & {\text{if } \sqrt{\frac{\log|\cA|}{K}} \le P_{T} + \sqrt{d}\Delta_P \lesssim K\sqrt{\log|\cA|} },\\
		&H^2(P_{T} + \sqrt{d}\Delta_P),         & {\text{if } P_{T} + \sqrt{d}\Delta_P \gtrsim K\sqrt{\log|\cA|} },
	\end{array} \right.
	\$ 
	with probability at least $1-\zeta$.
\end{theorem}

\begin{proof}
	See Appendix \ref{appendix:alg:adv} for a detailed description and proof.
\end{proof}

The bound in Theorem \ref{thm:adv:informal} is similar to the bound in Theorem \ref{thm:regret:dyn}, with one key distinction: the total variation budget $\Delta$ is replaced by $\Delta_P,$ the variation budget of the kernel function. This highlights a crucial insight: even with adversarial rewards, where $\Delta \ge \Delta_T = \Omega(T),$ as long as the kernel's variation is sublinear in $T$, PROPO achieves near-optimal regret. We remark that this is the first regret result for adversarial rewards in the non-stationary setting. By contrast, \cite{zhou2020nonstationary,touati2020efficient} are concurrent works that study nonstationary MDPs in a different framework, achieving a regret bound of $\tilde{O}(d^{5/4}\Delta^{1/4}H^{5/4}T^{3/4}),$ which does not address adversarial rewards. \cite{efroni2020optimistic,cai2019provably} focus on adversarial rewards under stationary transitions and tabular settings, with regret bounds of $\tilde{O}(|\cS||\cA|^{1/2} H^{4/3} T^{2/3})$ and $\tilde{O}(dH^{3/2}T^{1/2}),$ respectively.

In the following theorem, we establish the upper bound of dynamic regret incurred by SW-LSVI-UCB (Algorithm \ref{alg:4}).

\begin{theorem}[Upper bound for Algorithm \ref{alg:4}] \label{thm:regret:dyn2}
	Suppose Assumption \ref{assumption:linear:mdp}, \ref{assumption:tv}, and \ref{assumption:orthonormal} hold. Let $w = \Theta(d^{1/3}\Delta^{-2/3}T^{2/3})$ in \eqref{eq:least-square}, $\lambda = \lambda' = 1$ in \eqref{eq:least-square} and \eqref{eq:evaluation}, $\beta = \sqrt{d}$ in \eqref{eq:bonus}, and $\beta' = C'\sqrt{dH^2\cdot \log(dT/\zeta)}$ in \eqref{eq:bonus}, where $C' > 1$ is an absolute constant and $\zeta \in (0, 1]$. We have 
	\$
	\text{D-Regret}(T) \lesssim  d^{5/6}\Delta^{1/3}HT^{2/3} \cdot \log(dT/\zeta)
	\$ 
	with probability at least $1-\zeta$.
\end{theorem}

\begin{proof}
See Appendix \ref{appendix:proof:regret:dyn2} for a detailed proof.
\end{proof}


\vspace{4pt}
\noindent{\bf Comparison.}
Compared with PROPO, SW-LSVI-UCB achieves a slightly better regret without the help of the periodic restart mechanism. Especially in the highly non-stationary case, that is $P_{T} + \sqrt{d}\Delta \gtrsim K\sqrt{\log|\cA|}$, SW-LSVI-UCB achieves a $\tilde{\cO}(T^{2/3})$ regret, where \algoplus only attains a linear regret in $T$. However, \algoplus achieves the same $\tilde{\cO}(T^{2/3})$ regret as SW-LSVI-UCB when $P_{T} + \sqrt{d}\Delta \lesssim K\sqrt{\log|\cA|}$, which suggests that \algoplus is provably efficient for solving slightly or even moderately non-stationary MDPs. Meanwhile, PROPO are ready to tackle the adversarial rewards with full-information feedback. 

\vspace{4pt}
\noindent
{\bf Optimality of the Bounds.} 
With Assumption \ref{assumption:orthonormal}, there is only a gap of $H^{1/3}$ between the sharper upper bound and the lower bound in Theorem~\ref{thm:lower:bound}. We conjecture that this gap can be bridged by using the ``Bernstein'' type bonus functions \citet{azar2017minimax,jin2018q,zhou2020nearly}. Since our focus is on designing a provably efficient policy optimization algorithm for non-stationary linear kernel MDPs, we do not use this technique for the clarity of our analysis. In Section \ref{sec:without assump}, we show that while the regret is slightly worse, our algorithms are also efficient even without Assumption \ref{assumption:orthonormal}.

\vspace{4pt}
\noindent
{\bf Best of Both Worlds.} 
The versions of PROPO for non-stationary rewards with bandit feedback and that for adversarial rewards with full-information feedback are different. Therefore, PROPO does not have a best-of-both-worlds guarantee. Research has shown that best-of-both-worlds guarantees can be achieved in the case of multi-armed bandits \citet{bubeck2012best,seldin2014one,wei2018more}. \cite{jin2020simultaneously} introduced a best-of-both-worlds algorithm for episodic MDPs with known transition kernels, while \cite{jin2021best} extended this result to unknown, stationary transition kernels. Given our focus on adversarial rewards, we defer the development of a best-of-both-worlds algorithm for non-stationary transitions to future work.

\subsection{Regret Analysis without Assumption \ref{assumption:orthonormal}}\label{sec:without assump}
As stated before, Assumption \ref{assumption:orthonormal} is helpful to derive sharp regret bounds. 
Next, we show that our algorithms are also efficient even without Assumption \ref{assumption:orthonormal}. Detailed proofs are deferred to Appendix \ref{appendix:without:assump}.



	\begin{theorem}[Upper bound for Algorithm \ref{alg:2}]  \label{thm:dyn3}
		Suppose Assumptions \ref{assumption:linear:mdp} and \ref{assumption:tv} hold. Let $\tau = \Pi_{[1,K]}(\lfloor (\frac{T\sqrt{\log|\cA|}}{H(P_{T} + \sqrt{d}\Delta)})^{2/3} \rfloor )$, $\alpha = \sqrt{\rho \log | \cA | /(H^2K)}$ in \eqref{eq:def:l}, $w = \Theta(\Delta^{-1/4}T^{1/4})$ in \eqref{eq:least-square}, $\lambda = \lambda' = 1$ in \eqref{eq:least-square} and \eqref{eq:evaluation}, $\beta = \sqrt{d}$ in \eqref{eq:bonus}, and $\beta' = C'\sqrt{dH^2\cdot \log(dT/\zeta)}$ in \eqref{eq:bonus}, where $C' > 1$ is an absolute constant and $\zeta \in (0, 1]$. We have 
		\$
		\text{D-Regret}(T) &\lesssim  d\Delta^{1/4}HT^{3/4} \cdot \log(dT/\zeta) \\
		& \,+ 
		\left\{
		\begin{array}{rcl}
			&\sqrt{H^3T\log|\cA|},            & {\text{if } 0 \le P_{T} + \sqrt{d}\Delta \le \sqrt{\frac{\log|\cA|}{K}} },\\
			&(H^2T\sqrt{\log|\cA|})^{2/3}(P_{T} + \sqrt{d}\Delta)^{1/3},       & {\text{if } \sqrt{\frac{\log|\cA|}{K}} \le P_{T} + \sqrt{d}\Delta \lesssim K\sqrt{\log|\cA|} },\\
			&H^2(P_{T} + \sqrt{d}\Delta),         & {\text{if } P_{T} + \sqrt{d}\Delta \gtrsim K\sqrt{\log|\cA|} },
		\end{array} \right.
		\$ 
		with probability at least $1-\zeta$.
	\end{theorem}

	\begin{theorem}[Informal] \label{thm:adv:informal:worse}
		For the MDPs with adversarial rewards and non-stationary transitions, if the full-information feedback is available, then the dynamic regret of PROPO is bounded by 
		\$
		\text{D-Regret}(T) &\lesssim  d\Delta_p^{1/4}HT^{3/4} \cdot \log(dT/\zeta) \\
		& \,+ 
		\left\{
		\begin{array}{rcl}
			&\sqrt{H^3T\log|\cA|},            & {\text{if } 0 \le P_{T} + \sqrt{d}\Delta_P \le \sqrt{\frac{\log|\cA|}{K}} },\\
			&(H^2T\sqrt{\log|\cA|})^{2/3}(P_{T} + \sqrt{d}\Delta_P)^{1/3},       & {\text{if } \sqrt{\frac{\log|\cA|}{K}} \le P_{T} + \sqrt{d}\Delta_P \lesssim K\sqrt{\log|\cA|} },\\
			&H^2(P_{T} + \sqrt{d}\Delta_P),         & {\text{if } P_{T} + \sqrt{d}\Delta_P \gtrsim K\sqrt{\log|\cA|} },
		\end{array} \right.
		\$ 
		with probability at least $1-\zeta$.
	\end{theorem}
 
	\begin{theorem}[Upper bound for Algorithm \ref{alg:4}] \label{thm:vi:worse}
		Suppose Assumptions \ref{assumption:linear:mdp} and \ref{assumption:tv} hold. Let $w = \Theta(\Delta^{-1/4}T^{1/4})$ in \eqref{eq:least-square}, $\lambda = \lambda' = 1$ in \eqref{eq:least-square} and \eqref{eq:evaluation}, $\beta = \sqrt{d}$ in \eqref{eq:bonus}, and $\beta' = C'\sqrt{dH^2\cdot \log(dT/\zeta)}$ in \eqref{eq:bonus}, where $C' > 1$ is an absolute constant and $\zeta \in (0, 1]$. We have 
		\$
		\text{D-Regret}(T) \lesssim  d\Delta^{1/4}HT^{3/4} \cdot \log(dT/\zeta)
		\$ 
		with probability at least $1-\zeta$.
	\end{theorem}

	Notably the term $\tilde{\cO}(d\Delta^{1/4}HT^{3/4})$ appears in both the results in Theorems \ref{thm:dyn3}, \ref{thm:adv:informal:worse}, and~\ref{thm:vi:worse}. Ignoring logarithmic factors, there is a gap $\tilde{\cO}(d^{1/6}\Delta^{-1/12}H^{1/3}T^{1/12})$ between our upper bounds and the lower bound $\Omega(d^{5/6}\Delta^{1/3}H^{2/3}T^{2/3})$ established in Theorem~\ref{thm:lower:bound}. We also remark that most existing work \citep{cheung2020reinforcement,zhao2021non,zhao2020simple,russac2019weighted,zhou2020nonstationary,touati2020efficient} can only obtain the similar bound $\tilde{\cO}(T^{3/4})$ without additional assumption. The only exception is \citet{wei2021non}, which establishes the $\tilde{\cO}(T^{2/3})$ regret upper bound. However, their algorithms cannot tackle the adversarial rewards. 
\subsection{Computational Complexity}
We present the computational complexity of PROPO. We first assume that the state space is finite ($|\cS| < \infty.)$ Recall that PROPO involves a policy improvement step and a policy evaluation step. The policy evaluation step estimates the Q-function $Q_h^k(s_h^k, a)$ for all $a \in \cA$ at each iteration of $(k, h),$ as shown in Algorithm \ref{alg:3}. For a fixed pair of $(s_h^k, a),$ evaluating the integral $\hat{\eta}_h^k(s_h^k, a) = \int_\cS \psi(s_h^k, a, s') \cdot V_{h+1}^{k}(s') \mathrm{d} s' $ requires $\cO(|\cS|).$ The evaluations of $\Lambda_h^{k}$ and $A_h^{k}$ require $\cO(wd^2).$ Calculating $\hat{\theta}_h^{k}$ and $\hat{\xi}_h^{k}$ takes $\cO(wd^2 + d^3).$ Given $\Lambda_h^{k}$ and $A_h^{k},$ $B_h^{k}(s_h^k, a)$ and $\Gamma_h^k(s_h^k, a)$ can be computed in $\cO(d^2).$ Calculating $V_h^k(s_h^k)$ takes $\cO(|\cA|).$

For the policy improvement step, at each iteration of $(k, h),$ we need to evaluate $\pi_h^k(a \,|\, s_h^k)$ for all $a \in \cA,$ introducing another factor of $|\cA|$ for computing $\hat{\eta}_h^k.$ Combining these terms and multiplying by the number of iterations $KH,$ the total computational complexity of PROPO is $\cO(KH[|\cS||\cA| + wd^2 + d^3]).$ Setting $w = d^{1/3}\Delta^{-2/3}T^{2/3}$ as suggested by Theorem \ref{thm:regret:dyn}, the computational complexity of PROPO is $\cO(KH|\cS||\cA| +K^{5/3}H^{5/3}d^{7/3}\Delta^{-2/3} +KHd^3).$

For the case where the size of the state space $|S|$ is infinite, we approximate $\int_\cS \psi(s_h^k, a, s') \cdot V_{h+1}^{k}(s') \mathrm{d} s'$ using the Monte Carlo method, which can be computed in $\cO(KH).$ Therefore, we replace $|S|$ with $KH$ in the previous result and obtain $\cO(K^2H^2|\cA| +K^{5/3}H^{5/3}d^{7/3}\Delta^{-2/3} +KHd^3).$

The computational complexity of SW-LSVI-UCB includes an additional term that solves the optimization problem for each state $s:$ $\pi_h^{k}(s) = \argmax_{a} Q_h^k (s, a).$ The running time of this term depends on the convergence rate of the optimization method used for estimating the Q-functions, which varies on a case-by-case basis. When the state space and the action space are finite, we can simply obtain $\argmax$ over the action space. So the computational complexity of SW-LSVI-UCB is the same as PROPO in this case: $\cO(KH|\cS||\cA| +K^{5/3}H^{5/3}d^{7/3}\Delta^{-2/3} +KHd^3).$

\section{Parameter-Free Algorithm}\label{sec:para_free_algo}

The requirement of prior knowledge about the variation budget $\Delta_P$, $\Delta_T$, and $P_{T}$ is often unrealistic in practice. In this section, we propose two parameter-free algorithms, B-PROPO and B-SW-LSVI-UCB, to address this issue. We also present the dynamic regret bound of B-SW-LSVI-UCB.

\subsection{Algorithm Description}
Building on the bandit-over-bandit mechanism \citep{cheung2019learning,zhou2020nonstationary}, we design the parameter-free version of PROPO and SW-LSVI-UCB, B-PROPO and B-SW-LSVI-UCB. To illustrate, we first take B-SW-LSVI-UCB as an example and derive its dynamic regret. The algorithm for B-PROPO is similar and its dynamic regret can be obtained in the same way.

The idea is that we divide the $K$ episodes into $M$ blocks. For the first block, we select the sliding window length from a uniform distribution defined on the feasible set $J_w$ and run SW-LSVI-UCB with the selected sliding window length. In the 
subsequent blocks, we update the probability of selecting the sliding window length based on the EXP3-P algorithm \citep{bubeck2012regret}. Specifically, we set the size of each block $M$ and the feasible set of length of the sliding window $J_w$ as:
\begin{align*}
	M = \lceil 5d^{1/3}T^{1/2} \rceil, \quad J_w = \{1, 2, 4,...,M\}.
\end{align*}
Denote $|J_w|=J$ and initialize the hyperparameters as:
\begin{align}\label{eq:parameter free hyperparameter}
	\gamma_1 = 0.95\sqrt{\frac{\log(J)}{J\lceil K/M \rceil}}, \quad \gamma_2 = \sqrt{\frac{\log(J)}{J\lceil K/M \rceil}}, \quad \gamma_3 = 1.05\sqrt{\frac{\log(J)}{J\lceil K/M \rceil}}, \quad q_{l,1}=0, \quad l \in [J].
\end{align}
We denote the total reward observed in the $i$-th block with sliding cycle length $w$ as $R_i(w).$ Inductively, after we observe the total reward up to the $(i-1)$-th block, we update the estimated total reward in the next block for each sliding window length $l$ as:
\begin{align}\label{eq:parameter free reward update}
	q_{l,i} = q_{l,(i-1)} + \frac{\gamma_2 +\mathbbm{1}\{l=l_i\}R_{i-1}(w_{i-1})/(MH)}{u_{l,(i-1)}}.
\end{align}
We then update the probability of selecting the sliding window length for the $i$-th block as:
\begin{align}\label{eq:parameter free probability update}
	u_{l,i} = (1 - \gamma_3)\frac{\exp(\gamma_1 q_{l,i})}{\sum_{l \in [J]}\exp(\gamma_1 q_{l,i})} + \frac{\gamma_3}{J}.
\end{align}
Intuitively, the problem of selecting the sliding window length can be viewed as a bandit problem where $J_w$ represents the feasible set of arms. We design the algorithm so that the agent can choose the sliding window length that closely approximates the optimal length without excessive exploration.  We summarize the algorithm in Algorithm~\ref{alg:B-SW-LSVI-UCB}. We present the dynamic regret bound of B-SW-LSVI-UCB in Theorem~\ref{thm:regret block sw}. 

\begin{algorithm}[t] \ 
	\caption{Block Sliding Window Least-Square Value Iteration with UCB (B-SW-LSVI-UCB)}
	\begin{algorithmic}[1]\label{alg:B-SW-LSVI-UCB}
		\REQUIRE Regularization factors $\lambda$ and $\lambda'$, and bonus multipliers $\beta$ and $\beta'$.
		\STATE Initialize $\gamma_1, \gamma_2, \gamma_3$ and  $\{q_{l,1}\}_{l \in [J]}$ as in \eqref{eq:parameter free hyperparameter}.
		\FOR{$i = 1, 2, \dots, \lceil K/M \rceil$}
		\STATE Receive the initial state $s_1^{(i-1)M+1}$.
		\STATE Update the probability of selecting the sliding window length as in \eqref{eq:parameter free probability update}.
		\STATE Sample $l_i \in [J]$ from the distribution $\{u_{l,i}\}_{l \in [J]}$. Select the sliding window length for the $i$-th block $w_i = 2^{l_i} \wedge M.$
		\STATE Run SW-LSVI-UCB with sliding window length $w_i$ in the $i$-th block.    
		\STATE Observe the total reward $R_i(w_i)$ for the $i$-th block and update the estimated total reward for each sliding window length $q_{l, i+1}$ as in \eqref{eq:parameter free reward update}.
		\ENDFOR
	\end{algorithmic}
\end{algorithm}

\begin{theorem}[Upper bound for Algorithm \ref{alg:B-SW-LSVI-UCB}] \label{thm:regret block sw}
	Suppose Assumption \ref{assumption:linear:mdp}, \ref{assumption:tv}, and \ref{assumption:orthonormal} hold. Let $\lambda = \lambda' = 1$ in \eqref{eq:least-square} and \eqref{eq:evaluation}, $\beta = \sqrt{d}$ in \eqref{eq:bonus}, and $\beta' = C'\sqrt{dH^2\cdot \log(dT/\zeta)}$ in \eqref{eq:bonus}, where $C' > 1$ is an absolute constant and $\zeta \in (0, 1]$. We have 
	\begin{align*}	
	\text{D-Regret}(T) \lesssim  d^{5/6}\Delta^{1/3}HT^{3/4} \cdot \log(dT/\zeta)
	\end{align*}
	with probability at least $1-\zeta$.
\end{theorem}
Note that there is only a gap of $T^{1/12}$ between Theorem \ref{thm:regret block sw} and Theorem \ref{thm:regret:dyn2}, meaning that we do not lose much by not knowing the variation budget in advance. We also present a parameter-free version of PROPO, called B-PROPO. The difference between B-PROPO and B-SW-LSVI-UCB lies in the feasible set of arms and the choice of algorithm in the $i$-th block. We define the feasible set as the Cartesian product of the feasible set for the sliding window length and the feasible set for the reset cycle length:
\begin{align*}
	J_{w, \tau} = \{1, 2, 4,...,M\} \times \{1, 2, 4,...,M\}.
\end{align*}
We redefine $J$ as $J = |J_{w, \tau}|.$ The total reward received in the $i$-th block is denoted as $R_i(w_i, \tau_i),$ instead of $R_i(w_i),$ to illustrate that the reward now depends on the pair of sliding window length and reset cycle length. We maintain the same parameter initialization, probability update rule, and reward update rule as in \eqref{eq:parameter free hyperparameter}, \eqref{eq:parameter free probability update}, and \eqref{eq:parameter free reward update}. We present B-PROPO in Algorithm~\ref{alg:B-PROPO}. The upper bound of the dynamic regret for B-PROPO can be derived similarly to Theorem~\ref{thm:regret block sw}. However, this bound involves terms of the total variation budget of each block, which makes it more complex to present. For the sake of brevity, we omit the result here.
   \begin{algorithm}[t] 
	\caption{Block Periodically Restarted Optimistic Policy Optimization(B-PROPO)} \label{alg:B-PROPO}
	\begin{algorithmic}[1] 
		\REQUIRE Stepsize $\alpha$, regularization factors $\lambda$ and $\lambda'$, and bonus multipliers $\beta$ and $\beta'$.
		\STATE Initialize $\gamma_1, \gamma_2, \gamma_3$ and  $\{q_{l,1}\}_{l \in [J]}$ as in \eqref{eq:parameter free hyperparameter}.
		\FOR{$i = 1, 2, \dots, \lceil K/M \rceil$}
		\STATE Receive the initial state $s_1^{(i-1)M+1}$.
		\STATE Update the probability of selecting the pair of sliding window length and reset cycle length as in \eqref{eq:parameter free probability update}.
		\STATE Sample $l_i \in [J]$ from the distribution $\{u_{l,i}\}_{l \in [J]}$. Select the pair of sliding window length and reset cycle length for the $i$-th block $(w_i, \tau_i)$ as the $l_i$-th element of $J_{w, \tau}$.
		\STATE Run PROPO with sliding window length $w_i$ and reset cycle length $\tau_i$ in the $i$-th block.    
		\STATE Observe the total reward $R_i(w_i, \tau_i)$ for the $i$-th block and update the estimated total reward for each pair of sliding window length and reset cycle length $q_{l, i+1}$ as in \eqref{eq:parameter free reward update}.
		\ENDFOR
	\end{algorithmic}
\end{algorithm}
\section{Proof Sketch of Theorem \ref{thm:regret:dyn}}\label{sec:sketch}

In this section, we sketch the proof of Theorem \ref{thm:regret:dyn}. 

To facilitate the following analysis, we define the model prediction error as
\# \label{eq:def:model:error}
l_h^k = r_h^k + \PP_h^k V_{h+1}^k - Q_h^k ,
\#
which characterizes the error using $V_h^k$ to replace $V_h^{\pi_k, k}$ in the Bellman equation \eqref{eq:bellman}.
\subsection{Proof Sketch of Theorem \ref{thm:regret:dyn}}
\begin{proof}[Proof Sketch of Theorem \ref{thm:regret:dyn}. ] \label{proof:sketch}
First, we decompose the regret of Algorithm \ref{alg:2} into two terms
\$
\text{D-Regret}(T) = {\mathcal{R}}_1 + \mathcal{R}_2,
\$
where $\mathcal{R}_1= \sum_{k=1}^{K} V_1^{\pi^*,k}(s_1^k) - V_1^{k}(s_1^k)$ and $\mathcal{R}_2 = \sum_{k=1}^{K} V_1^{k}(s_1^k) - V_1^{\pi^k,k}(s_1^k)$. Then we analyze $\mathcal{R}_1$ and $\mathcal{R}_2$ respectively. By Lemma \ref{lemma:regret:decomposition2}, we have 
\$
\mathcal{R}_1 &= \sum_{i=1}^{\rho}\sum_{k=(i - 1)\tau + 1}^{i\tau}\sum_{h=1}^H \EE_{\pi^{*, k}} \bigl[ \la Q^{k}_h(s_h,\cdot), \pi^{*, k}_h(\cdot\,|\,s_h) - \pi^k_h(\cdot\,|\,s_h) \ra  \bigr] \\
& \qquad + \sum_{i=1}^{\rho}\sum_{k=(i - 1)\tau + 1}^{i\tau}\sum_{h=1}^H \EE_{\pi^{*, k}}[l^{k}_h(s_h,a_h)] . 
\$
Applying Lemma \ref{lem:omd:term2} to the first term, we obtain
\$
&\sum_{i=1}^{\rho}\sum_{k=(i - 1)\tau + 1}^{i\tau}\sum_{h=1}^H \EE_{\pi^{*, k}} \bigl[ \la Q^{k}_h(s_h,\cdot), \pi^{*, k}_h(\cdot\,|\,s_h) - \pi^k_h(\cdot\,|\,s_h) \ra  \bigr]  \\
&\qquad \le \sqrt{2H^3T\rho \log|\cA|}  +  2\tau H^2(P_{T} + \sqrt{d}\Delta) .
\$
Meanwhile, as shown in Lemma \ref{lemma:regret:decomposition2}, we have 
\$
\mathcal{R}_2 =  \cM_{K, H, 2} - \sum_{i=1}^{\rho}\sum_{k=(i - 1)\tau + 1}^{i\tau}\sum_{h=1}^H  l^{k}_h(s^k_h,a^k_h) .
\$
Here $ \cM_{K, H, 2}$ is a martingale defined in Appendix \ref{sec:regret:decomposition}. Then by the Azuma-Hoeffding inequality, we obtain $| \cM_{K, H, 2} | \le \sqrt{16H^2T\cdot\log(4/\zeta)}$ with probability at least $1-\zeta/2$. Here $\zeta \in (0, 1]$ is a constant.

Now we only need to derive the bound of the quantity $ \sum_{i=1}^{\rho}\sum_{k=(i - 1)\tau + 1}^{i\tau}\sum_{h=1}^H ( \EE_{\pi^{*, k}}[ l^{k}_h(s_h,a_h)] -  l^{k}_h(s^k_h,a^k_h) )$.   Applying the bound of $l_h^k$ in Lemma \ref{lem:ucb} to this quantity, it holds with probability at least $1 - \zeta/2$ that
\$
& \sum_{i=1}^{\rho}\sum_{k=(i - 1)\tau + 1}^{i\tau}\sum_{h=1}^H \bigl(( \EE_{\pi^{*, k}}[ l^{k}_h(s_h,a_h)] -  l^{k}_h(s^k_h,a^k_h) \bigr) \\
& \qquad \le  2 \sum_{k=1}^K\sum_{h=1}^H \bigl( \sum_{i =1\vee (k-w)}^{k -1} \| \theta_h^i - \theta_h^{i+1} \|_2 + \sum_{i =1\vee (k-w)}^{k -1} \| \xi_h^i - \xi_h^{i+1} \|_2 + B_h^k(s, a) + \Gamma_h^k(s, a)  \bigr) .
\$
Then we apply Lemmas \ref{lemma:telescope} and \ref{lemma:telescope2} to bound this quantity by $2w\Delta H\sqrt{d} + 8dT\sqrt{\log(w)/w} + 8C' dTH\cdot  \sqrt{\log(wH^2d)/w}  \cdot \log(dT/\zeta)$, where $C'$ is a constant specified in the detailed proof.

With the help of these bounds, we derive the regret bound in Theorem \ref{thm:regret:dyn}.  
\end{proof}

\subsection{Online Mirror Descent Term}
In this subsection, we establish the upper bound of the online mirror descent term.

The following lemma characterizes the policy improvement step defined in \eqref{eq:update:pi}, where the updated policy $\pi^k$ takes the closed form in \eqref{eq:closed:form}.
\begin{lemma}[One-Step Descent] \label{lem:omd:term}
	For any distribution $\pi$ on $\cA$ and $\{\pi^k\}_{k=1}^K$ obtained in Algorithm \ref{alg:2}, it holds that 
	\$
	& \alpha \cdot \la Q_h^k ,  \pi_h(\cdot\,|\,s) - \pi^{k}(\cdot\,|\,s)\ra  \\
	& \qquad \le {\rm KL}\bigl(\pi_h(\cdot\,|\,s)\,\|\,\pi_h^k(\cdot\,|\,s)\bigr) - {\rm KL}\bigl(\pi_h(\cdot\,|\,s)\,\|\,\pi_h^{k+1}(\cdot\,|\,s)\bigr) + \alpha^2H^2/2 .
	\$
\end{lemma}

\begin{proof}
	See Appendix \ref{appendix:lem:omd:term} for a detailed proof.
\end{proof}

Based on Lemma \ref{lem:omd:term}, we establish an upper bound of the online mirror descent term in the following lemma.
\begin{lemma}[Online Mirror Descent Term] \label{lem:omd:term2}
	For the Q-functions $\{Q_h^k\}_{(k, h) \in [K] \times [H]}$ obtained in \eqref{eq:evaluation} and the policies $\{\pi_h^k\}_{(k, h) \in [K] \times [H]}$ obtained in \eqref{eq:closed:form}, we have
	\$
	&\sum_{i=1}^{\rho}\sum_{k=(i - 1)\tau + 1}^{i\tau}\sum_{h=1}^H \EE_{\pi^{*, k}} \bigl[ \la Q^{k}_h(s_h,\cdot), \pi^{*, k}_h(\cdot\,|\,s_h) - \pi^k_h(\cdot\,|\,s_h) \ra \bigr]  \\
	& \qquad \le \sqrt{2H^3T\rho \log|\cA|}  +  2\tau H^2(P_{T} + \sqrt{d}\Delta).
	\$
\end{lemma}

\begin{proof}
	See Appendix \ref{appendix:lem:omd:term2} for a detailed proof.
\end{proof}

\subsection{Model Prediction Error Term}
In this subsection, we characterize the model prediction errors arising from estimating reward functions and transition kernels. 

\begin{lemma}[Upper Confidence Bound] \label{lem:ucb}
	Under Assumptions \ref{assumption:linear:mdp} and \ref{assumption:orthonormal}, it holds with probability at least $1 - \zeta/2$ that
	\$
	 &-2B_h^k(s, a) - 2\Gamma_h^k(s, a) - \sum_{i =1\vee (k-w)}^{k -1} \| \theta_h^i - \theta_h^{i+1} \|_2 - H\sqrt{d} \cdot \sum_{i = 1\vee(k - w)}^{ k- 1} \| \xi_h^i - \xi_h^{i+1} \|_2 \\
	 & \qquad \le l_h^k(s, a) \le \sum_{i =1\vee (k-w)}^{k -1} \| \theta_h^i - \theta_h^{i+1} \|_2 + H\sqrt{d} \cdot \sum_{i = 1\vee(k - w)}^{ k- 1} \| \xi_h^i - \xi_h^{i+1} \|_2
	\$
	for any $(k, h) \in [K] \times [H]$ and $(s, a) \in \cS \times \cA$, where $w$ is the length of a sliding window defined in \eqref{eq:estimate:reward}, $B_h^k(\cdot, \cdot )$ is the bonus function of reward defined in \eqref{eq:bonus} and $\Gamma_h^k(\cdot, \cdot)$ is the bonus function of transition kernel defined in \eqref{eq:bonus}.  
	\end{lemma}

\begin{proof}
 See Appendix \ref{appendix:lem:ucb} for a detailed proof.
\end{proof}

Since our model is non-stationary, we cannot ensure that the estimated Q-function is ``optimistic in the face of uncertainty'' as $l_h^k \le 0$ like the previous work \citep{jin2019provably, cai2019provably} in the stationary case. Thanks to the sliding window method, the model prediction error here can be upper bounded by the slight changes of parameters in the sliding window. Specifically, within the sliding window, the reward functions and transition kernels can be considered unchanged, which enables us to estimate the Q-function by regression and UCB bonus, and thus achieve the optimism like the stationary case. However, reward functions and transition kernels are actually different in the sliding window, which leads to additional errors caused by parameter changes. 

 By giving the bound of the model prediction error $l_h^k$ defined in \eqref{eq:def:model:error}, Lemma \ref{lem:ucb} quantifies uncertainty and thus realizes sample efficiency. 
In detail, uncertainty is because we can only observe finite historical data and many state-action pairs $(s, a)$ being less visited or even unseen. The model prediction error of these state-action pairs may be large.
However, as is shown in Lemma \ref{lem:ucb}, the model prediction error $l_h^k$ can be bounded by the variation of sequences $\{\theta_h^i\}_{i = 1 \vee (k - w)}^{k} $ and $\{\xi_h^i\}_{i = 1 \vee (k - w)}^{k}$, together with the bonus functions $B_h^k$ and $\Gamma_h^k$ defined in \eqref{eq:bonus}, which helps us to derive the bound of the regret.  See Appendix \ref{appendix:proof:regret:dyn} for details.



\section{Experiments}\label{sec:experiment}
In this section, we perform simulation experiments to demonstrate the effectiveness of the proposed algorithms. We compare the performance of SW-LSVI-UCB and PROPO with two baseline algorithms, Random-Exploration and Epsilon-greedy \cite{watkins1989learning}. In Random-Exploration, the agent selects actions uniformly at random, while in Epsilon-greedy, the agent selects the action with the highest estimated Q function value with probability $1-\epsilon$ and selects a random action with probability $\epsilon$. We set $\epsilon = 0.05$ for Epsilon-greedy. In SW-LSVI-UCB, we set the sliding window length $w = d^{1/3}\Delta^{-2/3}T^{2/3},$ $\lambda = \lambda' = 1,$ $\beta = \sqrt{d},$ and $\beta' = \sqrt{dH^2 \cdot \log(dT/0.2)}.$ In PROPO, we set $\tau = \Pi_{[1,K]}(\lfloor (\frac{T\sqrt{\log|\cA|}}{H(P_{T} + \sqrt{d}\Delta)})^{2/3} \rfloor )$, $\alpha = 60 \sqrt{\rho \log | \cA | /(H^2K)},$ $w = d^{1/3}\Delta^{-2/3}T^{2/3},$ $\lambda = \lambda' = 1,$ $\beta = \sqrt{d},$ and $\beta' = \sqrt{dH^2\cdot \log(dT/0.2)}.$ The parameters are different from those used in the theoretical analysis at most by a constant factor. We do not tune these constants.

\paragraph{Settings of Linear Kernel MDPs} We adopt the same data generating process used in the synthetic datasets from the experiments of \cite{zhou2020nonstationary}, with the exception that we modify the transition kernel to align with our linear kernel MDP framework. In particular, we consider a non-stationary episodic MDP with $S = 15$ states, $A = 7$ actions, horizon $H = 10$, and feature dimension $d = 8$. The total number of episodes is set to $K = 1000$, resulting in $T = KH = 10000$ total number of steps. Our environment incorporates $5$ special chains, inspired by the combination lock structure. The MDP is characterized by a known feature map $\phi: \mathcal{S} \times \mathcal{A} \rightarrow \mathbb{R}^d$ and the unknown parameters $\theta_{h,k} \in \mathbb{R}^d$ and $\xi_{h,k} \in \mathbb{R}^d$. The reward function is defined as $r_h^k(s,a) = \phi(s,a)^\top \theta_{h,k}$, while the transition probability is given by $P_h^k(s'|s,a) = \psi(s,a,s')^\top \xi_{h,k}$, where $\psi: \mathcal{S} \times \mathcal{A} \times \mathcal{S} \rightarrow \mathbb{R}^d$ is randomly generated. To investigate the impact of different levels of non-stationarity, we introduce two experimental setups:
\begin{enumerate}
    \item Stochastic Reward: In this setup, the reward function changes moderately over time, with variations occurring every 100 episodes.
    \item Adversarial Reward: This setup features more frequent and substantial changes in the reward function, with variations occurring every 50 episodes. This creates a more challenging environment with larger reward fluctuations.
\end{enumerate}

In both setups, the transition function $\xi_{h,k}$ changes at a consistent rate, with variations every 100 episodes. At the boundary between each group of 100 episodes, two dimensions of $\xi_{h, k}$ are perturbed by a small, fixed amount: $+0.01$ for one dimension and $-0.01$ for the other. After perturbation,  $\xi_{h, k}$ is adjusted to ensure non-negativity of $P_h^k(s'|s,a)$ and re-normalized to maintain valid probability distributions.The reward structure maintains a combination lock property (\citet{koenig1993complexity}), with a single "good" chain providing a large reward at the end, while other actions yield small positive rewards. This design encourages deep exploration to discover the optimal policy. For details of the reward function's generating process, we refer readers to Appendix E of \cite{zhou2020nonstationary}.

\begin{figure}[htbp]
    \centering
    \begin{minipage}[b]{0.48\textwidth}
        \includegraphics[width=\linewidth]{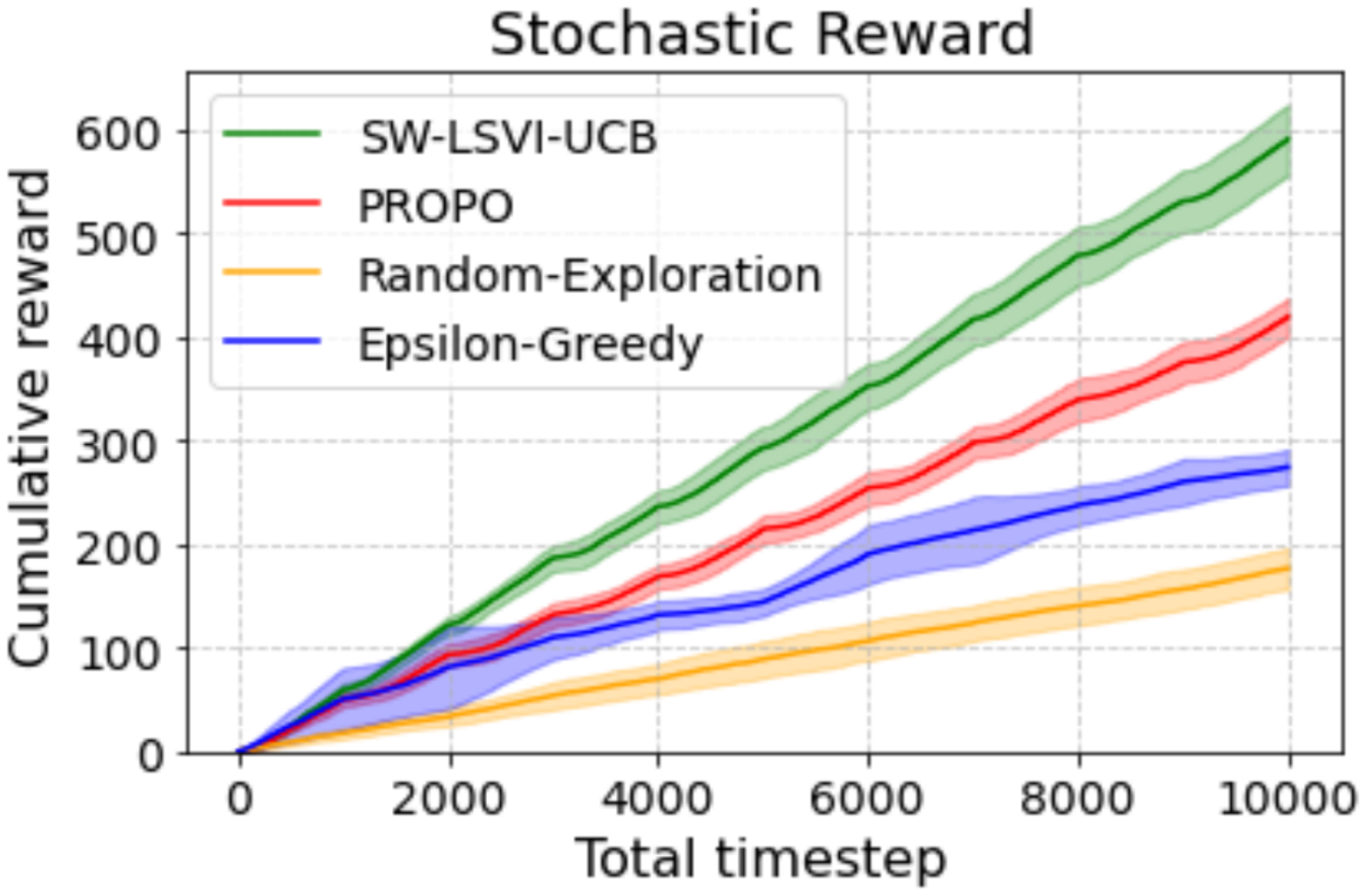}
    \end{minipage}
    \hfill
    \begin{minipage}[b]{0.48\textwidth}
        \includegraphics[width=\linewidth]{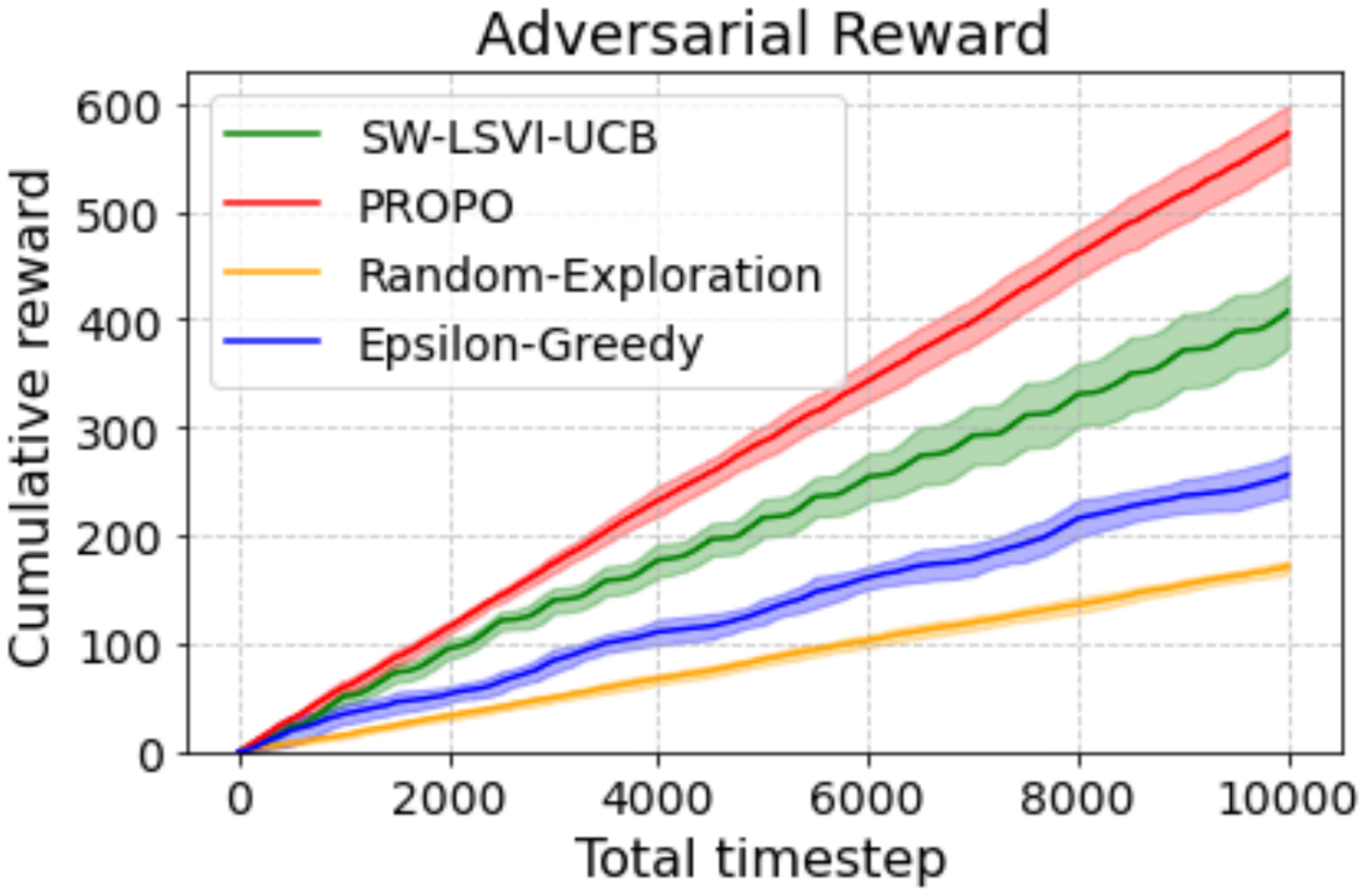}
    \end{minipage}
    \caption{Comparisons of different methods on cumulative reward under two different environments. The results are averaged over 10 trials and the error bars show the standard deviations. The reward function changes every $100$ episodes in the left subfigure, whereas the reward function changes every $50$ episodes in the right subfigure.}
    \label{fig:comparison}
\end{figure}

\paragraph{Results} Figure \ref{fig:comparison} presents the cumulative rewards of different algorithms under the two experimental setups. We observe that SW-LSVI-UCB and PROPO outperform the baselines in both environments, since they restart the estimation of the Q function to deal with the non-stationarity. In the stochastic reward setup, SW-LSVI-UCB achieves higher cumulative rewards than PROPO, which is consistent with our discussion of the theoretical result in Subsection \ref{sec:regret analysis} of regret analysis.  In the adversarial reward setup, PROPO outperforms SW-LSVI-UCB. The reason is that we implement the version of PROPO that can naturally tackle the adversarial rewards with full information.

\section{Conclusion}
In this work, we propose a provably efficient policy optimization algorithm, dubbed as PROPO,  for non-stationary linear kernel MDPs. Such an algorithm incorporates a bonus function to incentivize exploration, and more importantly, adopts sliding-window-based regression in policy evaluation and periodic restart in policy update to handle the challenge of non-stationarity.  Moreover, as a byproduct, we establish an optimistic value iteration algorithm, SW-LSVI-UCB, by combining UCB and sliding-window. We prove that PROPO and SW-LSVI-UCB both achieve sample efficiency by having sublinear dynamic regret. We also establish a dynamic regret lower bound which shows that PROPO and SW-LSVI-UCB are near-optimal. 
To the best of our knowledge, we propose the first provably efficient policy optimization method that successfully handles non-stationarity. For future work, we believe that extending our approach to develop a best of both worlds algorithm capable of handling both non-stationarity and adversarial rewards would be a valuable direction.

\section*{Acknowledgements}
We sincerely thank Jinglin Chen for sharing the experimental code from \cite{zhou2020nonstationary}.


\bibliographystyle{ims}
\bibliography{ref}

\newpage
\begin{appendix}

\section{Learning Adversarial Rewards} \label{appendix:alg:adv}

\subsection{Algorithm}
\begin{algorithm}[H] 
	\caption{PROPO (Adversarial Rewards with Full-information Feedback)}
	\begin{algorithmic}[1]
		\REQUIRE Reset cycle length $\tau$, sliding window length $w$, stepsize $\alpha$, regularization factor $\lambda'$, and bonus multiplier  $\beta'$.
		\STATE Initialize $\{\pi_h^0(\cdot \,|\, \cdot)\}_{h=1}^H$ as uniform distribution policies, $\{Q_h^0(\cdot, \cdot)\}_{h=1}^H$ as zero functions.
		
		\FOR{$k = 1, 2, \dots, K$}
		\STATE Receive the initial state $s_1^k$.
		\IF{$k \mod \tau = 1$}
		\STATE Set $\{Q_h^{k-1}\}_{h \in [H]}$ as zero functions and $\{\pi_h^{k-1}\}_{h \in [H]}$ as uniform distribution on $\cA$.
		\ENDIF
		\FOR{$h = 1, 2,  \dots ,H$}      
		\STATE $\pi^k_h(\cdot\,|\,\cdot) \propto \pi^{k-1}_h(\cdot\,|\,\cdot) \cdot \exp\{\alpha\cdot Q^{k-1}_h(\cdot,\cdot)\}$.
		\STATE Take action $a_h^k \sim \pi_h^k(\cdot \,|\, s_h^k)$.
		\STATE  Observe the reward function $r_h^k(\cdot, \cdot)$ and receive the next state $s_{h+1}^k$.
		\ENDFOR
	    \STATE Compute $Q_h^k$ by SWOPE$(k, \{\pi_h^k\}, \lambda, \lambda', \beta, \beta')$ (Algorithm \ref{alg:31}).
		\ENDFOR
	\end{algorithmic} \label{alg:21}
\end{algorithm}

\begin{algorithm}[H] 
	\caption{SWOPE (Adversarial Rewards with Full-information Feedback)}
	\begin{algorithmic}[1]
		\REQUIRE Episode index $k$, policies $\{\pi_h\}$, regularization factor $\lambda'$, and bonus multiplier $\beta'$.
		\STATE Initialize $V^k_{H+1}$ as a zero function.
		\FOR{$h = H, H-1, \dots ,0$}  
		\STATE $\eta_h^k(\cdot,\cdot) = \int_\cS \psi(\cdot, \cdot, s') \cdot V_{h+1}^{k}(s') \mathrm{d} s' $.
		\STATE $A_h^k =  \sum_{\tau = 1 \vee (k - w)}^{k-1} \eta_h^{\tau}(s_h^\tau, a_h^\tau)\eta_h^{\tau}(s_h^\tau, a_h^\tau)^\top + \lambda' I_d.$.
		\STATE $\hat{\xi}_h^k = (A_h^k)^{-1}\bigl(\sum_{\tau=1 \vee (k - w)}^{k-1} \eta_h^{\tau}(s_h^\tau, a_h^\tau) \cdot V_{h+1}^\tau(s_{h+1}^\tau) \bigr)$.
		\STATE $\Gamma_h^k(\cdot, \cdot) = \beta' \bigl( \eta_h^k(\cdot, \cdot)^\top(A_h^k)^{-1}\eta_h^k(\cdot, \cdot) \bigr) ^{1/2}$.
		\STATE $Q_h^k (\cdot , \cdot) = \min\{r_h^k(\cdot, \cdot) + \eta_h^k(\cdot, \cdot)^\top \hat{\xi}_h^k + \Gamma_h^k(\cdot, \cdot), H - h +1\}^+$. 
		\STATE$V_h^k(s) = \la Q_h^k(s,\cdot), \pi_h^k(\cdot | s) \ra _\cA$.
		\ENDFOR
	\end{algorithmic} \label{alg:31}
\end{algorithm}

\subsection{Theoretical Results}
\begin{theorem}[Formal Statement of Theorem \ref{thm:adv:informal}] \label{thm:adv}
	Suppose Assumptions \ref{assumption:linear:mdp}, \ref{assumption:tv}, and \ref{assumption:orthonormal} hold. Let $\tau = \Pi_{[1,K]}(\lfloor (\frac{T\sqrt{\log|\cA|}}{H(P_{T} + \sqrt{d}\Delta)})^{2/3} \rfloor )$, $\alpha = \sqrt{\rho \log | \cA | /(H^2K)}$ in \eqref{eq:def:l}, $w = \Theta(d^{1/3}\Delta^{-2/3}T^{2/3})$ in \eqref{eq:least-square}, $\lambda' = 1$ in \eqref{eq:evaluation}, and $\beta' = C'\sqrt{dH^2\cdot \log(dT/\zeta)}$ in \eqref{eq:bonus}, where $C' > 1$ is an absolute constant and $\zeta \in (0, 1]$. Hence, the dynamic regret of Algorithm \ref{alg:21} is bounded by
   \$
	\text{D-Regret}(T) &\lesssim  d^{5/6}\Delta_p^{1/3}HT^{2/3} \cdot \log(dT/\zeta) \\
	& \,+ 
    \left\{
    \begin{array}{rcl}
    	&\sqrt{H^3T\log|\cA|},            & {\text{if } 0 \le P_{T} + \sqrt{d}\Delta_P \le \sqrt{\frac{\log|\cA|}{K}} },\\
    	&(H^2T\sqrt{\log|\cA|})^{2/3}(P_{T} + \sqrt{d}\Delta_P)^{1/3},       & {\text{if } \sqrt{\frac{\log|\cA|}{K}} \le P_{T} + \sqrt{d}\Delta_P \lesssim K\sqrt{\log|\cA|} },\\
    	&H^2(P_{T} + \sqrt{d}\Delta_P),         & {\text{if } P_{T} + \sqrt{d}\Delta_P \gtrsim K\sqrt{\log|\cA|} },
    \end{array} \right.
	\$ 
	with probability at least $1-\zeta$.
\end{theorem}

\begin{proof}
	See Appendix \ref{appendix:proof:adv} for a detailed proof.
\end{proof}

\begin{theorem}[Formal Statement of Theorem \ref{thm:adv:informal:worse}] \label{thm:adv:worse}
	Suppose Assumptions \ref{assumption:linear:mdp} and \ref{assumption:tv} hold. Let $w = \Theta(\Delta^{-1/4}T^{1/4})$ in \eqref{eq:least-square}, $\lambda' = 1$ in \eqref{eq:evaluation}, and $\beta' = C'\sqrt{dH^2\cdot \log(dT/\zeta)}$ in \eqref{eq:bonus}, where $C' > 1$ is an absolute constant and $\zeta \in (0, 1]$. We have 
\$
		\text{D-Regret}(T) &\lesssim  d\Delta_p^{1/4}HT^{3/4} \cdot \log(dT/\zeta) \\
		& \,+ 
		\left\{
		\begin{array}{rcl}
			&\sqrt{H^3T\log|\cA|},            & {\text{if } 0 \le P_{T} + \sqrt{d}\Delta_P \le \sqrt{\frac{\log|\cA|}{K}} },\\
			&(H^2T\sqrt{\log|\cA|})^{2/3}(P_{T} + \sqrt{d}\Delta_P)^{1/3},       & {\text{if } \sqrt{\frac{\log|\cA|}{K}} \le P_{T} + \sqrt{d}\Delta_P \lesssim K\sqrt{\log|\cA|} },\\
			&H^2(P_{T} + \sqrt{d}\Delta_P),         & {\text{if } P_{T} + \sqrt{d}\Delta_P \gtrsim K\sqrt{\log|\cA|} },
		\end{array} \right.
		\$ 
		with probability at least $1-\zeta$.
\end{theorem}

\begin{proof}
	See Appendix \ref{appendix:without:assump} for a detailed proof.
\end{proof}

\section{Proof of Theorem \ref{thm:lower:bound}}\label{appendix:proof:lower:bound}

{	\centering
\begin{figure}[H] 
	\centering
  \begin{tikzpicture}
\node[state]             (s) {$x_0$};
\node[state, right=8cm of s] (r) {$x_1$};
\node[below=0.2cm of s] () {$r = 0$};
\node[below=0.2cm of r] () {$r = 1$};
\draw[every loop]
(s) edge[bend right, auto=left]  node {$\delta + \langle a, \xi\rangle$} (r)
(r) edge[bend right, auto=right] node {$\delta$} (s)
(s) edge[loop above]             node {$1 - \delta - \langle a, \xi \rangle$} (s)
(r) edge[loop above]             node {$1 - \delta$} (r);

\end{tikzpicture}
\caption{The hard-to-learn linear kernel MDP constructed in the proof of Theorem \ref{thm:lower:bound}. Note that the probability of state $x_0$ to state $x_1$ depends on the choice of action $a$.}
\label{figure:1}	
\end{figure}
}
	
\begin{proof}
	To handle the non-stationarity, we divide the total $T$ steps into $L$ segments, where each segment has $K_0 = \lfloor \frac{K}{L} \rfloor$ episodes and contains $T_0 = HK_0 = H\lfloor \frac{K}{L} \rfloor$ steps. Now we show the construction of a hard-to-learn MDP within each segment, the construction is similar to that used in previous works \citep{JMLR:v11:jaksch10a, lattimore2012pac, osband2016lower, zhou2020provably}. Consider an MDP as depicted in Figure \ref{figure:1}. The state space $\cS$ consists of two states $x_0$ and $x_1$. The action space $\cA$ consists of $2^{d-1}$ vectors $a \in \{-1, 1\}^{d - 1}$, where $d \geq 2$ is the dimension of feature map $\psi$ defined in Assumption \ref{assumption:linear:mdp}. The reward function does not depend on actions: state $x_0$ always gives reward $0$, and state $x_1$ always gives reward $1$, that is, for any $a \in \cA$,
	\$
	r(x_0, a) = 0, \qquad r(x_1, a) = 1 .
	\$
	Choosing,
	\$
	\theta = (1/d, 1/d, \cdots, 1/d)^\top \in \RR^d, \, \phi(x_0, a) = (0, 0, \cdots, 0)^\top \in \RR^d,  \, & \phi(x_1, a) = (1, 1, \cdots, 1)^\top \in \RR^d ,
	\$
	for any $a \in \cA$, it follows that $r(s, a) = \phi(s, a)^\top \theta$ for any $(s, a) \in \cS \times \cA$, and thus this reward is indeed linear.
	The probability transition $P_\xi$ is parameterized by a $(d-1)$-dimensional vector $\xi \in \Xi = \{-\epsilon/(d - 1), \epsilon/(d - 1) \}^{d - 1}$, which is defined as
	\$
	&P_{\xi}(x_0 \,|\, x_0, a) = 1 - \delta - \langle a , \xi \rangle, &\qquad  P_{\xi}(x_1 \,|\, x_0, a) =  \delta + \langle a , \xi \rangle,   \\
	&P_{\xi}(x_0 \,|\, x_1, a) = \delta ,  \qquad  &P_{\xi}(x_1 \,|\, x_1, a) = 1 - \delta,     
	\$
	where $\delta > 0$ and $\epsilon \in [0, d -1]$ are parameters which satisfy that $2\epsilon \le \delta \le 1/3$. This MDP is indeed a  linear kernel MDP with the $d$-dimensional vector $\tilde{\xi} = (\xi^\top, 1)^\top$. Specifically, we can define the feature map $\psi(s, a, s')$ as
	\$
	&\psi(x_0, a, x_0) = (-a^\top, 1 - \delta)^\top,  \qquad &\psi(x_0, a, x_1) = (a^\top, \delta)^\top,   \\
	&\psi(x_1, a, x_0) = (\mathbf{0}^\top, \delta)^\top,  \qquad &\psi(x_1, a, x_1) = (\mathbf{0}^\top, 1 - \delta)^\top ,
	\$
	and it is not difficult to verify that $P_\xi(s' \,|\, s, a) = \psi(s ,a, s')^\top \tilde{\xi}$. 
	
	Now we are ready to establish the lower bound in Theorem \ref{thm:lower:bound}. By Yao’s minimax principle \citep{yao1977probabilistic}, it is sufficient to consider deterministic policies. Hence, we assume that the policy $\pi$ obtained by the algorithm maps from a sequence of observations to an action deterministically. To facilitate the following proof, we introduce some notations. Let $N_0$, $N_1$, $N_0^a$ and $N_0^{\cA'}$ denote the total number of visits to state $x_0$, the total number of visits to $x_1$, the total number of visits to state $x_0$ followed by taking action $a$, and the total number of visits to state $x_0$ followed by taking an action in $\cA' \subseteq \cA$, respectively. Let $\cP_\xi(\cdot)$ denote the distribution over $\cS^{T_0}$, where $s_1^k = x_0$, $s_{h + 1}^k \sim P_{\xi}(\cdot \,|\, s_h^k, a_h^k)$, $a_h^k$ is decided by $\pi_h^k$. We use $\EE_{\xi}$  to denote the expectation with respect to $\cP_{\xi}$.
	
	Now we consider a segment that consists of $K_0$ episodes and each episode starts from state $x_0$. Let $s_h^k$ denote the state in the $h$-th state of the $k$-th episode. Fix $\xi \in \Xi$. We have,
	\# \label{eq:34000}
	\EE_{\xi}N_1 & = \sum_{k = 1}^{K_0} \sum_{h = 2}^H \cP_{\xi} (s_h^k = x_1) =  \sum_{k = 1}^{K_0} \sum_{h = 2}^H \cP_{\xi} (s_h^k = x_1, s_{h- 1}^k = x_1) +  \sum_{k = 1}^{K_0} \sum_{h = 2}^H \cP_{\xi} (s_h^k = x_1, s_{h-1}^k = x_0)  \notag \\
	& = \underbrace{ \sum_{k = 1}^{K_0} \sum_{h = 2}^H \cP_{\xi}(s_h^k = x_1 \,|\, s_{h-1}^k = x_1) \cP_{\xi}(s_{h-1}^k = x_1) }_{\dr(i)}+  \underbrace{\sum_{k = 1}^{K_0} \sum_{h = 2}^H \cP_{\xi}(s_h^k = x_1, s_{h-1}^k = x_0)}_{\dr(ii)} . 
	\#
	By the construction of this hard-to-learn MDP, we have $\cP_{\xi}(s_h^k = x_1 \,|\, s_{h-1}^k = x_1) = 1 - \delta$, which implies that
	\# \label{eq:34001}
	{\dr(i)} & = (1 - \delta)  \cdot \sum_{k = 1}^{K_0} \sum_{h = 2}^H \cP_{\xi}(s_{h-1}^k = x_1) \notag\\
	& = (1- \delta) \cdot \EE_{\xi}N_1 - (1 - \delta) \cdot \sum_{k = 1}^{K_0} \cP_{\xi}(s_{H}^k = x_1) .  
	\#
	Meanwhile, we have
	\$
	{\dr(ii)} = \sum_{k = 1}^{K_0} \sum_{h = 2}^H \sum_a \cP_{\xi}(s_h^k = x_1 \,|\, s_{h-1}^k = x_0, a_{h-1}^k = a) \cdot \cP_{\xi}(s_{h-1}^k =x_0, a_{h-1}^k = a) .
	\$
	By the fact that $\cP_{\xi}(s_h^k = x_1 \,|\, s_{h-1}^k = x_0, a_{h-1}^k = a) = \delta + \langle a, \xi \rangle$, we further obtain
	\# \label{eq:34002}
	{\dr(ii)} &= \sum_{k = 1}^{K_0} \sum_{h = 2}^H \sum_a (\delta + \langle a, \xi \rangle)\cdot \cP_{\xi}(s_{h-1}^k =x_0, a_{h-1}^k = a) \notag\\
	&=  \sum_a (\delta + \langle a, \xi \rangle)\cdot \bigl( \EE_{\xi}N_0^a -  \sum_{k = 1}^{K_0}\cP_{\xi}(s_{H}^k =x_0, a_{H}^k = a) \bigr) .
	\#
	Plugging \eqref{eq:34001} and \eqref{eq:34002} into \eqref{eq:34000} and rearranging gives
	\# \label{eq:34003}
	\EE_{\xi}N_1 & =  \sum_a (1 + \langle a, \xi \rangle / \delta)\cdot  \EE_{\xi}N_0^a - \underbrace{ \sum_{k =1}^{K_0} \bigl(  \frac{1- \delta}{\delta} \cdot \cP_{\xi} (s_H^k = x_1) + \sum_a (1 + \frac{ \langle a, \xi \rangle }{ \delta}) \cdot \cP_{\xi}(s_H^k = x_0, a_H^k = a) \bigr) }_{\Phi_\xi} \notag\\
	&= \EE_{\xi}N_0 + \delta^{-1} \cdot \sum_a \langle a, \xi \rangle \EE_{\xi} N_0^a - \Phi_\xi .
	\#
	By \eqref{eq:34003} and the fact that $\langle a, \xi \rangle \le \epsilon$, we further have
	\# \label{eq:34005}
	\EE_\xi N_1 &= \EE_{\xi}N_0 + \delta^{-1} \cdot \sum_a \langle a, \xi \rangle \cdot \EE_{\xi} N_0^a - \Phi_\xi  \notag\\
	&\geq \EE_{\xi}N_0 - \frac{\epsilon}{\delta} \cdot \EE_{\xi}N_0 - \sum_{k = 1}^{K_0} \bigl( \frac{1- \delta}{\delta} \cP_{\xi} (s_H^k = x_1) + (1 + \frac{\epsilon}{\delta} ) \cdot \cP_{\xi}(s_H^k = x_0)\bigr) \notag\\
	& = (1 - \frac{\epsilon}{\delta}) \cdot \EE_{\xi}N_0 - \sum_{k =1}^{K_0} \bigl( \frac{1 - \delta}{\delta} + \frac{\epsilon + 2\delta - 1}{\delta} \cdot \cP_\xi (s_H^k = x_0) \bigr) \notag \\
	& \geq  (1 - \frac{\epsilon}{\delta}) \cdot \EE_{\xi}N_0 - \frac{1 - \delta}{\delta} \cdot K_0,
	\#
	where the second equality uses the fact that $\cP_{\xi} (s_H^k = x_0) + \cP_{\xi} (s_H^k = x_1) = 1$, and the last inequality holds since $ \frac{\epsilon + 2\delta - 1}{\delta} \cdot  \cP_\xi (s_H^k = x_0)$ is negative. Together with $ N_0 + N_1 = T_0$, \eqref{eq:34005} implies that 
	\$
	\EE_\xi N_0 \le \frac{T_0 + (1 - \delta)/\delta \cdot K_0}{2 - \epsilon/ \delta} \le \frac{2T_0}{3} + \frac{2}{3\delta} K_0,
	\$
	where the last inequality follows from $2\epsilon \le \delta$ and $\delta > 0$. 
	Meanwhile, note that $\Phi_\xi$ is non-negative because $\langle a, \xi \rangle \geq - \epsilon \geq - \delta$. Combined with the fact that $N_0 + N_1 = T_0$, \eqref{eq:34003} and $\Phi_\xi \geq 0$ imply that
	\# \label{eq:34004}
	\EE_{\xi}N_1 \le T_0/2 + \delta^{-1} \cdot \sum_a \langle a, \xi \rangle  \cdot \EE_{\xi} N_0^a /2 .
	\#
	Hence, we have
	\# \label{eq:34006}
	\frac{1}{| \Xi |}\sum_\xi \EE_\xi N_1 &\le \frac{T_0}{2} + \frac{1}{2\delta | \Xi| }\sum_\xi \sum_a \langle a, \xi \rangle \EE_\xi N_0^a \notag \\
	& \le \frac{T_0}{2} + \frac{\epsilon}{2\delta(d-1)| \Xi |} \sum_{j = 1}^{d - 1}\sum_a\sum_\xi \EE_{\xi} \bigl( \mathbf{1}\{\text{sgn}(\xi_j) = \text{sgn}(a_j)\} \bigr)N_0^a ,
	\#
	where $\mathbf{1}\{\cdot\}$ is the indicator function. Here the last inequality uses the fact that $ \langle a, \xi \rangle \le \frac{\epsilon}{d - 1} \sum_{j = 1}^{d - 1}\mathbf{1}\{\text{sgn}(\xi_j) = \text{sgn}(a_j)\}$ for any $a \in \cA$ and $\xi \in \Xi$. Fix $j \in [d - 1]$. 
	We define a new vector $g(\xi)$ as
	\begin{equation}\notag
		g(\xi)_i =\left\{
		\begin{aligned}
			\xi_i,  \qquad & \text{if } i \not= j, \\
			-\xi_i,  \qquad & \text{if } i = j.
		\end{aligned}
		\right.
	\end{equation}
	Then, for any $a \in \cA$ and $\xi \in \Xi$, we have
	\# \label{eq:34007}
	&\EE_\xi \mathbf{1}\{\text{sgn}(\xi_j) = \text{sgn}(a_j)\}N_0^a +  \EE_{g(\xi)} \mathbf{1}\{\text{sgn}(g(\xi)_j) = \text{sgn}(a_j)\}N_0^a \notag\\
	& \qquad = \EE_{g(\xi)}N_0^a + \EE_\xi \mathbf{1}\{\text{sgn}(\xi_j) = \text{sgn}(a_j)\}N_0^a -  \EE_{g(\xi)} \mathbf{1}\{\text{sgn}(\xi_j) = \text{sgn}(a_j)\}N_0^a.
	\#
	Taking summation of \eqref{eq:34007} over $a$ and $\xi$, and because $g(\xi)$ is uniformly distributed over $\Xi$ when $\xi$ is uniformly distributed over $\Xi$, we have
	\# \label{eq:34008}
	& 2\sum_a\sum_\xi \EE_{\xi} \bigl( \mathbf{1}\{\text{sgn}(\xi_j) = \text{sgn}(a_j)\} \bigr)N_0^a  \notag \\
	& \qquad = \sum_\xi \sum_a \bigl( \EE_{g(\xi)}N_0^a + \EE_\xi \mathbf{1}\{\text{sgn}(\xi_j) = \text{sgn}(a_j)\}N_0^a -  \EE_{g(\xi)} \mathbf{1}\{\text{sgn}(\xi_j) = \text{sgn}(a_j)\}N_0^a \bigr)  \notag \\
	&  \qquad =  \sum_\xi \bigl(\EE_{g(\xi)}N_0 + \EE_{\xi} N_0^{\cA_j^\xi} - \EE_{g(\xi)} N_0^{\cA_j^\xi} \bigr) ,
	\#
	where $\cA_j^\xi = \{a : \text{sgn}(\xi_j) = \text{sgn}(a_j)\}$. By Lemma \ref{lemma:pinsker} and the fact that $N_0^{\cA_j^\xi} \le T_0$, we have
	\# \label{eq:34009}
	\EE_{\xi} N_0^{\cA_j^\xi} - \EE_{g(\xi)} N_0^{\cA_j^\xi} \le \frac{cT_0}{16} \sqrt{{\rm{KL}} (\cP_{g(\xi)} \| \cP_{\xi}) } ,
	\#
	where $c = 8\sqrt{\log2} $. Moreover, by Lemma \ref{lemma:kl:bound}, we have
	\# \label{eq:34010}
	{\rm{KL}} (\cP_{\xi'} \| \cP_{\xi}) \le \frac{16\epsilon^2}{(d-1)^2 \delta} \EE_{\xi}N_0 .
	\#
	Plugging  \eqref{eq:34007}, \eqref{eq:34008}, \eqref{eq:34009} and \eqref{eq:34010} into \eqref{eq:34006}, we obtain
	\# \label{eq:34011}
	\frac{1}{| \Xi |}\sum_\xi \EE_\xi N_1 &\le \frac{T_0}{2} + \frac{\epsilon}{4\delta(d - 1)| \Xi| } \sum_{j = 1}^{d - 1}\sum_{\xi} \bigl( \EE_{\xi'} N_0 + \frac{cT_0\epsilon}{2d\sqrt{\delta}} \sqrt{\EE_{\xi}N_0 }  \bigr)  \notag \\
	& \le \frac{T_0}{2} + \frac{\epsilon}{4\delta(d - 1)| \Xi| } \sum_{j = 1}^{d - 1}\sum_{\xi} \biggl(  \frac{2T_0}{3} + \frac{2}{3\delta} K_0 + \frac{cT_0\epsilon}{2d\sqrt{\delta}} \sqrt{\frac{2T_0}{3} + \frac{2}{3\delta} K_0 }  \biggr)  \notag \\
	& = \frac{T_0}{2} +   \frac{\epsilon T_0}{6\delta} + \frac{\epsilon K_0}{6\delta^2} + \frac{cT_0\epsilon^2}{8d\delta \sqrt{\delta}} \sqrt{\frac{2T_0}{3} + \frac{2}{3\delta} K_0 }  .
	\#
	Note that for a given $\xi$, whether in state $x_0$ or $x_1$, the optimal policy is to choose $a_\xi = [\text{sgn}(\xi_i)]_{i = 1}^{d - 1}$. Hence, we can calculate the stationary distribution and find that the optimal average reward is $\frac{\delta + \epsilon}{2\delta + \epsilon}$. Recall the definition of dynamic regret in \eqref{eq:def:regret:dyn}, we have
	\#
	\frac{1}{| \Xi| } \sum_{\xi} \EE_{\xi} {\text{D-Regret}(T_0)} & \geq  \frac{\delta + \epsilon}{2\delta + \epsilon} \cdot T_0 -   \frac{1}{| \Xi |}\sum_\xi \EE_\xi N_1 \notag \\
	& \geq \frac{\delta + \epsilon}{2\delta + \epsilon} \cdot T_0 -  \frac{T_0}{2} -   \frac{\epsilon T_0}{6\delta} - \frac{\epsilon K_0}{6\delta^2} - \frac{cT_0\epsilon^2}{8d\delta \sqrt{\delta}} \sqrt{\frac{2T_0}{3} + \frac{2}{3\delta} K_0 } .
	\#
	Setting $\delta = \Theta(\frac{1}{H})$ and $\epsilon = \Theta(\frac{d}{\sqrt{HT_0}})$, we have
	\$
	\frac{1}{| \Xi| } \sum_{\xi} \EE_{\xi} {\text{D-Regret}(T_0)} \geq \Omega(d\sqrt{HT_0}).
	\$
	Recall that in our episodic setting, the transition kernels $\PP_1, \PP_2, \cdots, \PP_H$ may be different. By the same argument in \cite{jin2018q} (consider $H$ distinct hard-to-learn MDPs and set $\delta = \Theta(\frac{1}{H})$ and $\epsilon = \Theta(\frac{d}{\sqrt{T_0}})$), we obtain a dynamic regret lower bound of $\Omega(dH\sqrt{T_0})$ in the stationary linear kernel MDPs. For non-stationary linear kernel MDPs, the number of segments $L$ is under budget constraint $2\epsilon HL/\sqrt{d} \le \Delta$. By choosing $L = \Theta(d^{-1/3}\Delta^{2/3}H^{-2/3}T^{1/3})$, we have
	\$
	\frac{1}{| \Xi| } \sum_{\xi} \EE_{\xi} {\text{D-Regret}(T)}  \geq \Omega(L\cdot dH\sqrt{T/L} ) = \Omega(d^{5/6} \Delta^{1/3}H^{2/3}T^{2/3}) ,
	\$
	which concludes the proof of Theorem \ref{thm:lower:bound}.

\end{proof}

\section{Regret Decomposition} \label{sec:regret:decomposition}

Recall the definition of model prediction error in \eqref{eq:def:model:error}
\$
l_h^k = r_h^k + \PP_h^k V_{h+1}^k - Q_h^k .
\$
Meanwhile, for any $(k,h)\in[K]\times[H]$, we define $\cF_{k,h,1}$ as the $\sigma$-algebra generated by the following state-action sequence and reward functions,
\$
\{(s^\tau_i, a^\tau_i)\}_{(\tau, i)\in [k-1] \times [H]} \cup \{r^\tau\}_{\tau\in [k]} \cup \{(s^k_i, a^k_i)\}_{i\in [h]} .
\$
Similarly, we define $\cF_{k,h,2}$ as the $\sigma$-algebra generated by 
\$
\{(s^\tau_i, a^\tau_i)\}_{(\tau, i)\in [k-1] \times [H]} \cup \{r^\tau\}_{\tau\in [k]}  \cup \{(s^k_i, a^k_i)\}_{i\in [h]} \cup \{s^k_{h+1}\},
\$
where $s_{H+1}^k$ is a null state for any $k \in [K]$. The $\sigma$-algebra sequence $\{\cF_{k,h,m}\}_{(k,h,m)\in[K]\times[H]\times[2]}$ is a filtration with respect to the timestep index $t(k,h,m)=(k-1)\cdot 2H+(h-1)\cdot2+m$. It holds that $\cF_{k,h,m}\subseteq \cF_{k',h',m'}$ for any $t(k,h,m) \le t(k',h',m')$.

\begin{lemma}[Dynamic Regret Decomposition] \label{lemma:regret:decomposition2}
		For the policies $\{\pi^k\}_{k=1}^K$ obtained in Algorithm \ref{alg:2} and the optimal policies $\pi^{*, k}$ in $k$-th episode, we have the following decomposition
	\# \label{eq:regret:decomposition2}
	\text{D-Regret}(T) &= \sum_{k=1}^K \bigl(V^{\pi^{*, k},k}_1(s_1^k) - V^{\pi^k,k}_1(s_1^k)\bigr) \\
	&=  \underbrace{\sum_{i=1}^{\rho}\sum_{k=(i - 1)\tau + 1}^{i\tau}\sum_{h=1}^H \EE_{\pi^{*, k}} \bigl[ \la Q^{k}_h(s_h,\cdot), \pi^{*, k}_h(\cdot\,|\,s_h) - \pi^k_h(\cdot\,|\,s_h) \ra \bigr]}_{\dr (i)} + \underbrace{ \cM_{K, H, 2}}_{\dr (ii)}   \notag\\
	& \qquad +\underbrace{\sum_{i=1}^{\rho}\sum_{k=(i - 1)\tau + 1}^{i\tau}\sum_{h=1}^H \EE_{\pi^{*, k}}[l^{k}_h(s_h,a_h)]}_{\dr (iii)}    + \underbrace{ \sum_{i=1}^{\rho}\sum_{k=(i - 1)\tau + 1}^{i\tau}\sum_{h=1}^H  -l^{k}_h(s^k_h,a^k_h)}_{\dr (iv)}, \notag
	\#
\end{lemma}

\begin{proof}
	Recall the definition of dynamic regret in \eqref{eq:def:regret:dyn}, we have
	\# \label{eq:140}
	\text{D-Regret}(T) &= \sum_{k=1}^K \bigl( V^{\pi^{*, k},k}_1(s^k_1) - V^{\pi^k,k}_1(s^k_1)  \bigr)  \notag\\
	&= \sum_{i = 1}^{\rho}\sum_{k=(i - 1)\tau + 1}^{i\tau} \bigl( V^{\pi^{*, k},k}_1(s^k_1) - V^{\pi^k,k}_1(s^k_1)  \bigr) .
	\#
	Note that 
	\# \label{eq:141}
	V^{\pi^{*, k},k}_1(s^k_1) - V^{\pi^k,k}_1(s^k_1) = \underbrace{V^{\pi^{*, k},k}_1(s^k_1) -V_1^k(s_1^k)}_{\dr (i)} + \underbrace{V_1^k(s_1^k) - V^{\pi^k,k}_1(s^k_1)}_{\dr (ii)} .
	\#
	\vskip4pt
	\noindent{\bf Term (i):}
	By Bellman equation we have
	\# \label{eq:511}
	V_h^{\pi^{*, k},k}(s)-V_h^k(s)  & = \la Q_h^{\pi^k,k}(s,\cdot), \pi^{*, k}_h(\cdot \,|\, s) \ra_\cA - \la Q_h^k(s,\cdot), \pi^k_h(\cdot \,|\, s) \ra_\cA    \\
	&  = \la Q_h^{\pi^k,k}(s,\cdot) - Q_h^k(s,\cdot), \pi^{*, k}_h(\cdot \,|\, s) \ra_\cA + \la Q_h^k(s,\cdot), \pi^{*, k}_h(\cdot \,|\, s) -  \pi^k_h(\cdot \,|\, s) \ra_\cA  \notag
	\#
	for any $(s,h,k) \in \cS \times [H] \times [K]$. Meanwhile, by the definition of the model prediction error in \eqref{eq:def:model:error}, we have
	\$
	Q_h^k = r_h^k + \PP_h^k V_{h+1}^k -l_h^k.
	\$
	Combining with Bellman equation in \eqref{eq:bellman}, we further obtain
	\# \label{eq:512}
	Q_h^{\pi^k,k} - Q_h^k = \PP_h^k (V_{h+1}^{\pi^k,k} - V_{h+1}^k) + l_h^k.
	\#
	Plugging \eqref{eq:511} into \eqref{eq:512}, we obtain
	\# \label{eq:513}
	V_h^{\pi^{*, k}, k}(s)-V_h^k(s) &= \la \PP_h^k(V_h^{\pi^k,k} - V_h^k)(s) ,  \pi^{*, k}_h(\cdot \,|\, s) \ra_\cA + \la l_h^k(s,\cdot), \pi^{*, k}_h(\cdot \,|\, s) \ra_\cA   \notag\\ 
	& \qquad +\la Q_h^k(s,\cdot), \pi^{*, k}_h(\cdot \,|\, s) -  \pi^k_h(\cdot \,|\, s) \ra_\cA  .
	\#
	For notational simplicity, for any $(k,h)\in[K]\times[H]$ and function $f: \cS\times\cA\rightarrow \RR$. we define the operators $\mathbb{I}_h^k$ and $\mathbb{I}_{k,h}$ respectively by
	\# \label{eq:def:i}
	(\mathbb{I}_h^k f)(s) = \la f(x,\cdot), \pi^{*, k}_h(\cdot\,|\,s) \ra, \quad
	(\mathbb{I}_{k, h} f)(s) = \la f(x,\cdot), \pi^k_h(\cdot\,|\,s) \ra .
	\#
	Also, we define 
	\# \label{eq:def:mu} 
	\mu_h^k (s) = (\mathbb{I}_h^k Q_h^k)(s) - (\mathbb{I}_{k, h} Q_h^k)(s) = \langle Q_h^k(s, \cdot), \pi_h^{*, k}(\cdot \,|\, s) - \pi_h^k(\cdot \,|\, s) \rangle
	\#
	With this notation, recursively expanding \eqref{eq:513} over $h \in [H]$, we have
	\$
	V^{\pi^*,k}_1 - V^{k}_1 & =
	\Bigl(\prod_{h=1}^H \mathbb{I}_h^k\PP_h^k \Bigr) (V^{\pi^*,k}_{H+1} - V^{k}_{H+1})
	+ \sum_{h=1}^H \Bigl(\prod_{i=1}^{h-1} \mathbb{I}_i^k\PP_i^k \Bigr)
	\mathbb{I}_h^k l^{k}_h +
	\sum_{h=1}^H \Bigl(\prod_{i=1}^{h-1} \mathbb{I}_i^k\PP_i^k \Bigr) \mu^k_h \\
	& =  \sum_{h=1}^H \Bigl(\prod_{i=1}^{h-1} \mathbb{I}_i^k\PP_i^k \Bigr)
	\mathbb{I}_h^k l^{k}_h +
	\sum_{h=1}^H \Bigl(\prod_{i=1}^{h-1} \mathbb{I}_i^k \PP_i^k \Bigr) \mu^k_h ,
	\$
	where the last inequality follows from $V^{\pi^*,k}_{H+1} = V^{k}_{H+1} = 0$. By the definitions of $\PP_h^k$ in \eqref{eq:def:p}, $\mathbb{I}_h^k$ in \eqref{eq:def:i}, and $\mu_h^k$ in \eqref{eq:def:mu}, we further obtain
	\# \label{eq:142}
	\text{Term} {\dr (i)} = \sum_{h=1}^H \EE_{\pi^{*, k}} \bigl[ \la Q^{k}_h(s_h,\cdot), \pi^{*, k}_h(\cdot\,|\,s_h) - \pi^k_h(\cdot\,|\,s_h) \ra  \bigr] + \sum_{h=1}^H \EE_{\pi^{*, k}}[l^{k}_h(s_h,a_h)].
	\#
	\vskip4pt
	\noindent{\bf Term (ii):}
	Recall the definition of value function $V_h^{\pi^k,k}$ in \eqref{eq:bellman}, the estimated function $V_h^k$ in \eqref{eq:evaluation} and the operator $\mathbb{I}_h^k$ in \eqref{eq:def:i}, we expand the model prediction error $l_h^k$ into
	\$
	l_h^k(s_h^k,a_h^k) &= r_h^k(s_h^k,a_h^k) + (\PP_h^k V_{h+1}^k)(s_h^k,a_h^k) - Q_h^k(s_h^k,a_h^k)  \\
	& = \bigl( r_h^k(s_h^k,a_h^k) + (\PP_h^k V_{h+1}^k)(s_h^k,a_h^k) - Q_h^{\pi^k,k}(s_h^k,a_h^k) \bigr) + Q_h^{\pi^k,k}(s_h^k,a_h^k) - Q_h^k(s_h^k,a_h^k) \\
	& = \bigl( \PP_h^k(V_{h+1}^k - V_{h+1}^{\pi^k,k})  \bigr)(s_h^k,a_h^k)   + ( Q_h^{\pi^k,k} - Q_h^k) (s_h^k,a_h^k)  ,
	\$
	where the last equality follows from the Bellman equation in \eqref{eq:bellman}. Then we can expand $V_h^k(s_h^k) - V_h^{\pi^k,k}(s_h^k)$ into
	\$
	V_h^k(s_h^k) - V_h^{\pi^k,k}(s_h^k) &= \bigl( \mathbb{I}_{k, h}(Q_h^k - Q_h^{\pi^k,k})\bigr)(s_h^k) + l_h^k(s_h^k,a_h^k) -  l_h^k(s_h^k,a_h^k) \\
	& =  \bigl( \mathbb{I}_{k, h}(Q_h^k - Q_h^{\pi^k,k})\bigr)(s_h^k)  + ( Q_h^{\pi^k,k} - Q_h^k) (s_h^k,a_h^k)  \\
	& \qquad+  \bigl( \PP_h^k(V_{h+1}^k - V_{h+1}^{\pi^k,k})  \bigr)(s_h^k,a_h^k)  - l_h^k(s_h^k,a_h^k) .
	\$
	To facilitate our analysis, we define 
	\# \label{eq:120}
	& D_{k,h,1} = \bigl(\mathbb{I}_{k, h}(Q_h^k - Q_h^{\pi^k,k})\bigr)(s_h^k) - Q_h^k - Q_h^{\pi^k,k} , \\
	& D_{k,h,2} =  \bigl( \PP_h^k(V_{h+1}^k - V_{h+1}^{\pi^k,k})  \bigr)(s_h^k,a_h^k) - (V_{h+1}^k - V_{h+1}^{\pi^k,k})(s_{h+1}^k)  . \notag
	\#
	Hence, we have
	\# \label{eq:121}
	V_h^k(s_h^k) - V_h^{\pi^k,k}(s_h^k) = D_{k,h,1} + D_{k,h,2} +  (V_{h+1}^k - V_{h+1}^{\pi^k,k})(s_{h+1}^k) -  l_h^k(s_h^k,a_h^k) 
	\#
	for any $(k,h) \in [K] \times [H]$. For any $k \in [K]$, recursively expanding \eqref{eq:121} across $h \in [H]$ yields
	\# \label{eq:122}
	\text{Term} {\dr (ii)} &= \sum_{h=1}^H (D_{k,h,1} + D_{k,h,2}) - \sum_{h=1}^H l_h^k(s_h^k,a_h^k) +  \bigl(V_{H+1}^k(s_{H+1}^k) - V_{H+1}^{\pi^k,k}(s_{H+1}^k) \bigr) \notag\\
	& = \sum_{h=1}^H (D_{k,h,1} + D_{k,h,2}) - \sum_{h=1}^H l_h^k(s_h^k,a_h^k),
	\#
	where the last equality uses the fact that $V_{H+1}^k(s_{H+1}^k) = V_{H+1}^{\pi^k,k}(s_{H+1}^k) = 0$. By the definitions of $\cF_{k,h,1}$ and $\cF_{k,h,2}$, we have the $D_{k,h,1} \in \cF_{k,h,1}$ and $D_{k,h,2} \in \cF_{k,h,2}$. Hence, for any $(k,h) \in [K] \times [H]$,
	\$
	\EE[D_{k,h,1} | \cF_{k,h-1,2}] =0, \quad \EE[D_{k,h,2} | \cF_{k,h,1}] =0 .
	\$
	Notice that $\cF_{k,0,2} = \cF_{k-1,H,2}$ for any $k \geq 2$, which implies the corresponding timestep index $t(k,0,2) = t(k-1,H,2) = 2H(k-1)$. Meanwhile, we define $\cF_{1,0,2}$ to be empty. Thus we can define the following martingale
	\#\label{eq:123}
	\cM_{k,h,m} &= \sum_{\tau=1}^{k-1}\sum_{i=1}^H (D_{\tau,i,1}+D_{\tau,i,2}) + \sum_{i=1}^{h-1} (D_{k,i,1}+D_{k,i,2}) + \sum_{\ell=1}^{m} D_{k,h,\ell} \notag\\
	&= \sum_{\substack{(\tau,i,\ell)\in[K]\times[H]\times[2],\\ t(\tau,i,\ell)\le t(k,h,m)}} D_{\tau,i,\ell},
	\# 
	where $t(k,h,m) = 2(k-1)H + 2(h-1) + m$ is the timestep index. This martingale is obviously adapted to the filtration $\{\cF_{k,h,m}\}_{(k,h,m)\in[K]\times[H]\times[2]}$, and particularly we have
	\#\label{eq:124}
	\cM_{K,H, 2} = \sum_{k=1}^K\sum_{h=1}^{H} (D_{k,h,1}+D_{k,h,2}) .
	\#
    Plugging \eqref{eq:142} and \eqref{eq:122} into \eqref{eq:140}, we conclude the proof of Lemma \ref{lemma:regret:decomposition2}.
\end{proof}

\section{Proofs of Lemmas in Section \ref{sec:sketch}}

\subsection{Proof of Lemma \ref{lem:omd:term}} \label{appendix:lem:omd:term}

\begin{proof}
	For any $(k, h) \in [K] \times [H]$, let $z_k(s) = \sum_{a' \in \cA} p(a')\cdot \exp(\alpha \cdot \pi_h^{k}(a'\,|\,s))$. Since $z_k(s)$ is a constant function, it holds that, for any $s \in \cS$, 
	\$
	\la \log z_k(s) ,   \pi_h(\cdot\,|\,s) - \pi_h^{k+1}(\cdot\,|\,s) \ra = 0
	\$
	Hence, for any $s \in \cS $, it holds that 
 \$
 & {\rm KL}\bigl(\pi_h(\cdot\,|\,s)\,\|\,\pi_h^k(\cdot\,|\,s)\bigr) - {\rm KL}\bigl(\pi_h(\cdot\,|\,s)\,\|\,\pi_h^{k+1}(\cdot\,|\,s)\bigr)\notag\\
 & \qquad= \bigl\la \log\bigl(\pi_h^{k+1}(\cdot\,|\,s)/\pi_h^k(\cdot\,|\,s)\bigr), \pi_h(\cdot\,|\,s)\bigr\ra\notag\\
 & \qquad= \bigl\la \log\bigl(\pi_h^{k+1}(\cdot\,|\,s)/\pi_h^k(\cdot\,|\,s)\bigr), \pi_h(\cdot\,|\,s) - \pi_h^{k+1}(\cdot\,|\,s)\bigr\ra + {\rm KL}\bigl(\pi_h^{k+1}(\cdot\,|\,s)\,\|\,\pi_h^k(\cdot\,|\,s)\bigr)\notag\\
 & \qquad= \bigl\la \log z_k(s) + \log\bigl(\pi_h^{k+1}(\cdot\,|\,s)/\pi_h^k(\cdot\,|\,s)\bigr), \pi_h(\cdot\,|\,s) - \pi_h^{k+1}(\cdot\,|\,s)\bigr\ra + {\rm KL}\bigl(\pi_h^{k+1}(\cdot\,|\,s)\,\|\,\pi_h^k(\cdot\,|\,s)\bigr)\notag .
 \$
 Recall that $\pi_h^{k+1}(\cdot \,|\, \cdot) \propto \pi_h^k(\cdot \,|\, \cdot ) \cdot \exp\{\alpha \cdot Q_h^{k}(\cdot \,|\, \cdot ) \}$, we have
 \$
 & {\rm KL}\bigl(\pi_h(\cdot\,|\,s)\,\|\,\pi_h^k(\cdot\,|\,s)\bigr) - {\rm KL}\bigl(\pi_h(\cdot\,|\,s)\,\|\,\pi_h^{k+1}(\cdot\,|\,s)\bigr) \\
 & \qquad = \alpha \cdot \la Q_h^k ,  \pi_h(\cdot\,|\,s) - \pi_h^{k+1}(\cdot\,|\,s)\ra + {\rm KL}\bigl(\pi_h^{k+1}(\cdot\,|\,s)\,\|\,\pi_h^k(\cdot\,|\,s)\bigr)\notag .
 \$ 
 Thus,
 \# \label{eq:521}
 & \alpha \cdot \la Q_h^k ,  \pi_h(\cdot\,|\,s) - \pi^{k}(\cdot\,|\,s)\ra  \\
 & \qquad =   \alpha \cdot \la Q_h^k ,  \pi_h(\cdot\,|\,s) - \pi_h^{k+1}(\cdot\,|\,s)\ra +  \alpha \cdot \la Q_h^k ,  \pi_h^{k+1}(\cdot\,|\,s) - \pi^{k}(\cdot\,|\,s)\ra \notag \\
 & \qquad \le {\rm KL}\bigl(\pi_h(\cdot\,|\,s)\,\|\,\pi_h^k(\cdot\,|\,s)\bigr) - {\rm KL}\bigl(\pi_h(\cdot\,|\,s)\,\|\,\pi_h^{k+1}(\cdot\,|\,s)\bigr) - {\rm KL}\bigl(\pi_h^{k+1}(\cdot\,|\,s)\,\|\,\pi_h^k(\cdot\,|\,s)\bigr)  \notag \\
 & \qquad \qquad +\alpha \cdot  \|Q_h^k(s,\cdot)\|_{\infty}\cdot \|\pi_h^k(\cdot\,|\,s) - \pi_h^{k+1}(\cdot\,|\,s)\|_1, \notag
 \#
 where the last inequality uses Cauchy-Schwartz inequality. Meanwhile, by Pinsker's inequality, it holds that
 \# \label{eq:522}
 {\rm KL}\bigl(\pi_h^{k+1}(\cdot\,|\,s) \| \pi_h^{k}(\cdot\,|\,s)\bigr) \geq \|\pi_h^{k}(\cdot\,|\,s) - \pi_h^{k+1}(\cdot\,|\,s) \|_1^2/2.
 \#
 Plugging \eqref{eq:522} into \eqref{eq:521}, combined with the fact that $\|Q_h^k(s,\cdot)\|_{\infty} \le H$ for any $s \in \cS$, we have
 \$
 & \alpha \cdot \la Q_h^k ,  \pi_h(\cdot\,|\,s) - \pi^{k}(\cdot\,|\,s)\ra  \\
 & \qquad \le {\rm KL}\bigl(\pi_h(\cdot\,|\,s)\,\|\,\pi_h^k(\cdot\,|\,s)\bigr) - {\rm KL}\bigl(\pi_h(\cdot\,|\,s)\,\|\,\pi_h^{k+1}(\cdot\,|\,s)\bigr)  \\
 & \qquad \qquad  -  \|\pi_h^k(\cdot\,|\,s) - \pi_h^{k+1}(\cdot\,|\,s) \|_1^2/2 + \alpha H \|\pi_h^k(\cdot\,|\,s) - \pi_h^{k+1}(\cdot\,|\,s) \|_1 \\
 & \qquad \le {\rm KL}\bigl(\pi_h(\cdot\,|\,s)\,\|\,\pi_h^k(\cdot\,|\,s)\bigr) - {\rm KL}\bigl(\pi_h(\cdot\,|\,s)\,\|\,\pi_h^{k+1}(\cdot\,|\,s)\bigr) + \alpha^2H^2/2 ,
 \$
which completes the proof of Lemma \ref{lem:omd:term}.
\end{proof}

\subsection{Proof of Lemma \ref{lem:omd:term2}} \label{appendix:lem:omd:term2}

\begin{proof}
	  Recall that $\rho = \lceil K/\tau \rceil$. First, we have the decomposition
	\#
	&\sum_{i=1}^{\rho}\sum_{k=(i - 1)\tau + 1}^{i\tau}\sum_{h=1}^H \EE_{\pi^{*, k}} \bigl[ \la Q^{k}_h(s_h,\cdot), \pi^{*, k}_h(\cdot\,|\,s_h) - \pi^k_h(\cdot\,|\,s_h) \ra  \bigr]   \\
	& \qquad + \underbrace{\sum_{i=1}^{\rho}\sum_{k=(i - 1)\tau + 1}^{i\tau}\sum_{h=1}^H \EE_{\pi^{*, (i - 1)\tau + 1 }} \bigl[ \la Q^{k}_h(s_h,\cdot), \pi^{*, k}_h(\cdot\,|\,s_h) - \pi^k_h(\cdot\,|\,s_h) \ra  \bigr]}_{\dr (A)} \notag \\
	& \qquad +  \underbrace{\sum_{i=1}^{\rho}\sum_{k=(i - 1)\tau + 1}^{i\tau}\sum_{h=1}^H (\EE_{\pi^{*, k}} - \EE_{\pi^{*, (i - 1)\tau + 1 }} ) \bigl[ \la Q^{k}_h(s_h,\cdot), \pi^{*, k}_h(\cdot\,|\,s_h) - \pi^k_h(\cdot\,|\,s_h) \ra  \bigr]}_{\dr (B)} \notag .
	\#
	We can further decompose Term A as
	\$
	{\dr{Term (A)}} &=  \underbrace{\sum_{i=1}^{\rho}\sum_{k=(i - 1)\tau + 1}^{i\tau}\sum_{h=1}^H \EE_{\pi^{*, (i - 1)\tau + 1 }} \bigl[ \la Q^{k}_h(s_h,\cdot), \pi^{*, (i - 1)\tau + 1}_h(\cdot\,|\,s_h) - \pi^k_h(\cdot\,|\,s_h) \ra  \bigr] }_{A_1}  \notag\\
	& \qquad +  \underbrace{\sum_{i=1}^{\rho}\sum_{k=(i - 1)\tau + 1}^{i\tau}\sum_{h=1}^H \EE_{\pi^{*, (i - 1)\tau + 1 }} \bigl[ \la Q^{k}_h(s_h,\cdot), \pi^{*, k}_h(\cdot\,|\,s_h) - \pi^{*, (i - 1)\tau +1}_h(\cdot\,|\,s_h) \ra \bigr]}_{A_2} .
	\$
	By Lemma \ref{lem:omd:term}, we have 
	{\small
	\# \label {eq:43431}
	A_1 &\le \alpha KH^3/2 +  \sum_{h=1}^H \sum_{i=1}^{\rho} \frac{1}{\alpha}  \notag\\ 
	& \quad \times \sum_{k=(i - 1)\tau + 1}^{i\tau}\Bigl(   \EE_{\pi_h^{*, (i - 1)\tau + 1}} \bigl[ {\rm KL}(\pi_h^{*, (i - 1)\tau + 1}(\cdot\,|\,s_h)\,\|\,\pi_h^k(\cdot\,|\,s_h))  
	- {\rm KL}\bigl(\pi_h^{*, (i - 1)\tau + 1}(\cdot\,|\,s_h)\,\|\,\pi_h^{k+1}(\cdot\,|\,s_h)\bigr) \bigr]   \Bigr)  \notag \\
	&= \frac{1}{2}\alpha KH^3 + \sum_{h=1}^H \sum_{i=1}^{\rho} \frac{1}{\alpha} \cdot  \notag \\
	& \quad \times \Bigl( \EE_{\pi_h^{*, (i - 1)\tau + 1}} \bigl[ {\rm KL}\bigl(\pi_h^{*, (i - 1)\tau + 1}(\cdot\,|\,s_h)\,\|\,\pi_h^{(i - 1)\tau+1}(\cdot\,|\,s_h)\bigr) - {\rm KL}\bigl(\pi_h^{*, (i - 1)\tau + 1}(\cdot\,|\,s_h)\,\|\,\pi_h^{i\tau+1}(\cdot\,|\,s_h)\bigr) \bigr] \Bigr)  \notag\\
	& \le  \frac{1}{2}\alpha KH^3 + \frac{1}{\alpha} \cdot \sum_{h=1}^H \sum_{i=1}^{\rho}\Bigl( \EE_{\pi_h^{*, (i - 1)\tau + 1}} \big[ {\rm KL}\bigl(\pi_h^{*, (i - 1)\tau + 1}(\cdot\,|\,s_h)\,\|\,\pi_h^{(i - 1)\tau+1}(\cdot\,|\,s_h)\bigr) \bigr]  \Bigr) .
	\# }
	Here the second inequality is obtained by the fact that the ${\rm KL}$-divergence is non-negative. Note that $\pi_h^{(i - 1)\tau + 1}$ is the uniform policy, that is, $\pi_h^{(i - 1)\tau + 1}(a \,|\, s_h) = \frac{1}{|\cA|}$ for any $a \in \cA$. Hence, for any policy $\pi$ and $i \in [\rho]$, we have
	\# \label{eq:43432}
	{\rm KL}\bigl(\pi_h(\cdot\,|\,s_h)\,\|\,\pi_h^{(i - 1)\tau + 1}(\cdot\,|\,s_h)\bigr)  \notag 
	&= \sum_{a \in \cA} \pi_h(a\,|\,s_h) \cdot \log \bigl(|\cA| \cdot \pi_h(a\,|\,s_h)\bigr)  \notag \\
	& =  \log |\cA| +  \sum_{a \in \cA} \pi_h(a\,|\,s_h) \cdot \log \bigl(\pi_h(a\,|\,s_h)\bigr) \le \log |\cA| ,
	\#
	where the last inequality follows from the fact that the entropy of $\pi_h(\cdot \,|\,s_h)$ is non-negative. Plugging \eqref{eq:43432} into \eqref{eq:43431}, we have
	\# \label{eq:omd:bound}
	A_1 \le  \alpha H^3K/2 + \rho H  \log | \cA | /\alpha = \sqrt{2H^3T\rho \log|\cA|} ,
	\#
	where the last inequality holds since we set $\alpha = \sqrt{2\rho\log|\cA|/(H^2K)}$ in \eqref{eq:def:l}. Meanwhile, 
	\# \label{eq:434321}
	A_2 &\le \sum_{i=1}^{\rho}\sum_{k=(i - 1)\tau + 1}^{i\tau}\sum_{h=1}^H \EE_{\pi^{*, (i - 1)\tau + 1 }} \bigl[  H \cdot  \| \pi^{*, k}_h(\cdot\,|\,s_h) - \pi^{*, (i - 1)\tau +1}_h(\cdot\,|\,s_h) \|_1   \bigr]  \notag\\
	& \le H \cdot \sum_{i=1}^{\rho}\sum_{k=(i - 1)\tau + 1}^{i\tau}\sum_{h=1}^H\sum_{t = (i - 1)\tau + 2}^{k} \EE_{\pi^{*, (i - 1)\tau + 1 }} \bigl[  \| \pi^{*, k}_h(\cdot\,|\,s_h) - \pi^{*, t - 1}_h(\cdot\,|\,s_h) \|_1   \bigr] \notag \\
	& \le H \cdot \sum_{i=1}^{\rho}\sum_{k=(i - 1)\tau + 1}^{i\tau} \sum_{t = (i - 1)\tau + 2}^{i \tau} \sum_{h=1}^H \max_{s \in \cS} \| \pi^{*, t}_h(\cdot\,|\,s) - \pi^{*, t - 1}_h(\cdot\,|\,s) \|_1  \notag\\
	& = H\tau \cdot \sum_{t = 1}^K \sum_{h = 1}^H \max_{s \in \cS} \| \pi^{*, t}_h(\cdot\,|\,s) - \pi^{*, t - 1}_h(\cdot\,|\,s) \|_1 = H\tau P_T,
	\#
	where the first inequality follows by Holder's inequality and the fact that $\|Q_h^k( s , \cdot )\|_\infty \le H$, the second inequality follows from triangle inequality, and the last inequality is obtained by the definition of $P_T$ in \eqref{eq:def:pt}. Combing \eqref{eq:omd:bound} and \eqref{eq:434321}, we have
	\# \label{eq:434322}
	{\dr{Term (A)}} \le \sqrt{2H^3T\rho \log|\cA|} + H\tau P_T.
	\#
	To derive an upper bound of Term (B), we need the following lemma.
	\begin{lemma} \label{lem:visitation2}
		It holds that 
		\$
		 \sum_{i=1}^{\rho}\sum_{k=(i - 1)\tau + 1}^{i\tau}\sum_{h=1}^H (\EE_{\pi^{*, k}} - \EE_{\pi^{*, (i - 1)\tau + 1 }} ) \bigl[ \la Q^{k}_h(s_h,\cdot), \pi^{*, k}_h(\cdot\,|\,s_h) - \pi^k_h(\cdot\,|\,s_h) \ra  \bigr] \le  \tau H^2(P_T + \Delta_P),
		 \$
		 where $\Delta_P = \sum_{k = 1}^K \sum_{h = 1}^H \max_{(s, a) \in \cS \times \cA} \|  P_h^{k}( \cdot \,|\, s, a) - P_h^{k + 1}(\cdot \,|\, s, a) \|_1$.
	\end{lemma}

	\begin{proof}
		See Appendix \ref{appendix:pf:visitation2} for a detailed proof.
	\end{proof}

	By Lemma \ref{lem:visitation2}, we have
	\# \label{eq:434323}
	{\dr{Term (B)}} \le \tau H^2 (P_T + \Delta_P).
	\#
	Furthermore, by Assumption \ref{assumption:linear:mdp}, we further obtain
	\# \label{eq:434324}
	\Delta_P &= \sum_{k = 1}^K \sum_{h = 1}^H \max_{(s, a) \in \cS \times \cA} \|  P_h^{k}( \cdot \,|\, s, a) - P_h^{k + 1}(\cdot \,|\, s, a) \|_1  \notag \\
	& =  \sum_{k = 1}^K \sum_{h = 1}^H \max_{(s, a) \in \cS \times \cA} \sum_{s' \in \cS} |  P_h^{k}( s' \,|\, s, a) - P_h^{k + 1}(s' \,|\, s, a) |  \notag \\
	& =  \sum_{k = 1}^K \sum_{h = 1}^H \max_{(s, a) \in \cS \times \cA} \sum_{s' \in \cS} |  \psi(s, a, s')^\top (\xi_h^k - \xi_h^{k + 1}) | \notag \\
	& \le   \sum_{k = 1}^K \sum_{h = 1}^H \max_{(s, a) \in \cS \times \cA} \sum_{s' \in \cS} \|  \psi(s, a, s') \|_2 \cdot  \| \xi_h^k - \xi_h^{k + 1}\|_2  ,
	\# 
	where the last inequality follows from Cauchy-Schwarz inequality . Recall the assumption that $\sum_{s' \in \cS} \|  \psi(s, a, s') \|_2 \le \sqrt{d}$ for any $(s, a) \in \cS \times \cA$, we have
	\# \label{eq:434325}
	&\sum_{k = 1}^K \sum_{h = 1}^H \max_{(s, a) \in \cS \times \cA} \sum_{s' \in \cS} \|  \psi(s, a, s') \|_2 \cdot  \| \xi_h^k - \xi_h^{k + 1}\|_2  \notag \\
	& \qquad \le \sqrt{d} \sum_{k = 1}^K \sum_{h = 1}^H  \| \xi_h^k - \xi_h^{k + 1}\|_2  = \sqrt{d} B_{P} \le \sqrt{d} \Delta.
	\#
	Combining \eqref{eq:434322}, \eqref{eq:434323}, \eqref{eq:434324} and \eqref{eq:434325}, we have
	\$ 
	\sum_{i=1}^{\rho}\sum_{k=(i - 1)\tau + 1}^{i\tau}\sum_{h=1}^H \EE_{\pi^{*, k}} \bigl[ \la Q^{k}_h(s_h,\cdot), \pi^{*, k}_h(\cdot\,|\,s_h) - \pi^k_h(\cdot\,|\,s_h) \ra  \bigr]  &\le \sqrt{2H^3T\rho \log|\cA|}  +  \tau H^2(P_{T} + \sqrt{d}\Delta) + \tau H P_{T} \\
	 &\le \sqrt{2H^3T\rho \log|\cA|}  +  2\tau H^2(P_{T} + \sqrt{d}\Delta),
	\$ 
	which concludes the proof.
	
\end{proof}

\subsection{Proof of Lemma \ref{lem:ucb} } \label{appendix:lem:ucb}
\begin{proof}
	We first derive the upper bound of $-l_h^k(\cdot, \cdot)$. As defined in \eqref{eq:def:model:error}, for any $(k, h) \in [K] \times [H]$ and $(s, a) \in \cS \times \cA$,
	 \$
	-l_h^k(s, a) &= Q_h^k(s, a) - (r_h^k + \PP_h^kV_{h+1}^k)(s,a)  .
	\$
	Meanwhile, by the definition of $Q_h^k$  in \eqref{eq:evaluation}, we have
	\$
	Q_h^k (\cdot , \cdot) &= \min\{\phi(\cdot, \cdot)^\top\hat{\theta}_h^k + \eta_h^k(\cdot, \cdot)^\top \hat{\xi}_h^k + B_h^k(\cdot, \cdot) + \Gamma_h^k(\cdot, \cdot), H - h +1\}^+ \\
	& \le  \phi(s, a)^\top \hat{\theta}_h^k + \eta_h^k(s, a)^\top \hat{\xi}_h^k + B_h^k(s, a) + \Gamma_h^k(s, a)
	\$
	for any $(k, h) \in [K] \times [H]$ and $(s, a) \in \cS \times \cA$. Hence, we obtain
	\$
	-l_h^k(s, a) &= Q_h^k(s, a) - (r_h^k + \PP_h^k V_{h+1}^k)(s,a)  \\
	&\le \phi(s, a)^\top\hat{\theta}_h^k + \eta_h^k(s, a)^\top \hat{\xi}_h^k + B_h^k(s, a) + \Gamma_h^k(s, a) - (r_h^k + \PP_h^k V_{h+1}^k)(s,a)  \\
	&= \underbrace{\phi(s, a)^\top\hat{\theta}_h^k + B_h^k(s, a) - r_h^k(s, a)}_{\displaystyle{\rm (i)}}  + \underbrace{\eta_h^k(s, a)^\top \hat{\xi}_h^k + \Gamma_h^k(s, a) - \PP_h^k V_{h+1}^k(s, a)}_{\displaystyle{\rm (ii)}} .
	\$
	\vskip4pt
	\noindent{\bf Term (i):} 
    By the definition of $\hat{\theta}_h^k$ in \eqref{eq:estimate:reward}, we have 
	\$
	\hat{\theta}_h^k - \theta_h^k &= (\Lambda_h^{k})^{-1} \biggl(\sum_{\tau = 1\vee (k-w)}^{k-1}  \phi(s_h^\tau,a_h^\tau) r_h^\tau(s_h^\tau,a_h^\tau) \biggr) -   \theta_h^k\\
	& = (\Lambda_h^{k})^{-1} \biggl(\sum_{\tau = 1\vee (k-w)}^{k-1}  \phi(s_h^\tau,a_h^\tau) r_h^\tau(s_h^\tau,a_h^\tau) - \Lambda_h^k \theta_h^k \biggr)  \\
	&= (\Lambda_h^k)^{-1} \biggl( \sum_{\tau = 1\vee (k-w)}^{k-1} \phi(s_h^\tau, a_h^\tau)\phi(s_h^\tau, a_h^\tau)^\top(\theta_h^\tau - \theta_h^k) - \lambda \cdot \theta_h^k \biggr)  ,
	\$
	where the last equality is obtained by the definition of $\Lambda_h^k$ in \eqref{eq:estimate:reward} and the assumption that $r_h^\tau(s_h^\tau, a_h^\tau) = \phi(s_h^\tau, a_h^\tau)^\top \theta_h^\tau$ for any $(\tau, h) \in [K] \times [H]$. Hence, for any $(k, h) \in [K] \times [H]$ and $(s, a) \in \cS \times \cA$, we have
	\# \label{eq:con210}
	& | \phi(s, a)^\top (\hat{\theta}_h^k - \theta_h^k) | \\
	& \qquad \le  \underbrace{\biggl| \phi(s, a)^\top (\Lambda_h^k)^{-1} \biggl( \sum_{\tau = 1\vee (k-w)}^{k-1} \phi(s_h^\tau, a_h^\tau)\phi(s_h^\tau, a_h^\tau)^\top(\theta_h^\tau - \theta_h^k) \biggr) \biggr|}_{\rm (i.1)}  + \underbrace{| \phi(s, a)^\top (\Lambda_h^k)^{-1} ( \lambda \cdot \theta_h^k )  |}_{\rm (i.2)}  .  \notag
	\#
	Then we derive the upper bound of term $\rm (i.1)$ and term $\rm (i.2)$, respectively.
	\vskip4pt
	\noindent{\bf Term (i.1):} 
	By Cauchy-Schwarz inequality, we have
	\# \label{eq:con211}
	& \biggl| \phi(s, a)^\top (\Lambda_h^k)^{-1} \biggl( \sum_{\tau = 1\vee (k-w)}^{k-1} \phi(s_h^\tau, a_h^\tau)\phi(s_h^\tau, a_h^\tau)^\top(\theta_h^\tau - \theta_h^k) \biggr) \biggr| \notag\\
	& \qquad \le  \| \phi(s, a) \|_2 \cdot \biggl\| (\Lambda_h^k)^{-1} \biggl( \sum_{\tau = 1\vee (k-w)}^{k-1} \phi(s_h^\tau, a_h^\tau)\phi(s_h^\tau, a_h^\tau)^\top(\theta_h^\tau - \theta_h^k) \biggr) \biggr\|_2 \notag\\
	& \qquad \le \biggl\| (\Lambda_h^k)^{-1} \biggl( \sum_{\tau = 1\vee (k-w)}^{k-1} \phi(s_h^\tau, a_h^\tau)\phi(s_h^\tau, a_h^\tau)^\top(\theta_h^\tau - \theta_h^k) \biggr) \biggr\|_2 ,
	\#
	where the last inequality follows from the fact that $\| \phi(s, a) \|_2 \le 1$ for any $(s, a) \in \cS \times \cA$. Moreover, we have
	\# \label{eq:con212}
	& \biggl\| (\Lambda_h^k)^{-1} \biggl( \sum_{\tau = 1\vee (k-w)}^{k-1} \phi(s_h^\tau, a_h^\tau)\phi(s_h^\tau, a_h^\tau)^\top(\theta_h^\tau - \theta_h^k) \biggr) \biggr\|_2 \notag \\
	& \qquad = \biggl\| (\Lambda_h^k)^{-1} \biggl( \sum_{\tau = 1\vee (k-w)}^{k-1} \phi(s_h^\tau, a_h^\tau)\phi(s_h^\tau, a_h^\tau)^\top\Bigl(\sum_{i=\tau}^{k-1}(\theta_h^i - \theta_h^{i+1})\Bigr)  \biggr) \biggr\|_2  ,
	\#
	where the last equality follows from the fact that $\theta_h^\tau - \theta_h^k = \sum_{i = \tau}^{k -1}(\theta_h^i - \theta_h^{i+1}) $. By exchanging the order of summation, we further have
	\# \label{eq:con213}
	&  \biggl\| (\Lambda_h^k)^{-1} \biggl( \sum_{\tau = 1\vee (k-w)}^{k-1} \phi(s_h^\tau, a_h^\tau)\phi(s_h^\tau, a_h^\tau)^\top\Bigl(\sum_{i=\tau}^{k-1}(\theta_h^i - \theta_h^{i+1})\Bigr)  \biggr) \biggr\|_2  \notag \\
	& \qquad =  \biggl\| (\Lambda_h^k)^{-1} \biggl( \sum_{i = 1\vee (k-w)}^{k-1} \Bigl( \sum_{\tau = 1\vee (k-w)}^{i} \phi(s_h^\tau, a_h^\tau)\phi(s_h^\tau, a_h^\tau)^\top (\theta_h^i - \theta_h^{i+1})\Bigr)  \biggr) \biggr\|_2 \notag \\
	& \qquad \le \sum_{i = 1\vee (k-w)}^{k-1} \biggl\| (\Lambda_h^k)^{-1} \Bigl(  \sum_{\tau = 1\vee (k-w)}^{i} \phi(s_h^\tau, a_h^\tau)\phi(s_h^\tau, a_h^\tau)^\top  \Bigr) (\theta_h^i - \theta_h^{i+1})  \biggr\|_2 .
	\#
	By the fact that for any matrix $A \in \RR^{d \times d}$ and a vector $x \in \RR^d$, $\| Ax \|_2 \le \lambda_{\max}(A^\top A) \|x\|_2$, we have
	\# \label{eq:con214}
	& \sum_{i = 1\vee (k-w)}^{k-1} \biggl\| (\Lambda_h^k)^{-1} \Bigl(  \sum_{\tau = 1\vee (k-w)}^{i} \phi(s_h^\tau, a_h^\tau)\phi(s_h^\tau, a_h^\tau)^\top  \Bigr) (\theta_h^i - \theta_h^{i+1})  \biggr\|_2 \notag\\ 
	& \qquad \le \sum_{i = 1\vee (k-w)}^{k-1} \lambda_{\max} \biggl( \Bigl(  \sum_{\tau = 1\vee (k-w)}^{i} \phi(s_h^\tau, a_h^\tau)\phi(s_h^\tau, a_h^\tau)^\top  \Bigr)(\Lambda_h^k)^{-2} \notag\\
	& \qquad\qquad\qquad\qquad \Bigl(  \sum_{\tau = 1\vee (k-w)}^{i} \phi(s_h^\tau, a_h^\tau)\phi(s_h^\tau, a_h^\tau)^\top  \Bigr) \biggr) \| (\theta_h^i - \theta_h^{i+1})  \|_2 .
	\#
	Meanwhile, by Assumption \ref{assumption:orthonormal}, we assume $\phi(s_h^\tau, a_h^\tau) = z_h^\tau \varPsi_{i(s_h^\tau, a_h^\tau)} = z_h^\tau \Psi e_{i(s_h^\tau, a_h^\tau)},$ where $i(s_h^\tau, a_h^\tau) \in [d]$ and $e_i$ is the $i$-th standard orthonormal basis. For simplicity, we define $M_1 = \sum_{\tau = 1 \vee (k - w)}^{k - 1} e_{i(s_h^\tau, a_h^\tau)} e_{i(s_h^\tau, a_h^\tau)}^\top + \lambda I_d$ and $M_2 = \sum_{\tau = 1 \vee (k - w)}^{i} e_{i(s_h^\tau, a_h^\tau)} e_{i(s_h^\tau, a_h^\tau)}^\top$. Then, we have
	\# \label{eq:con2141}
	&\lambda_{\max} \biggl( \Bigl(  \sum_{\tau = 1\vee (k-w)}^{i} \phi(s_h^\tau, a_h^\tau)\phi(s_h^\tau, a_h^\tau)^\top  \Bigr)(\Lambda_h^k)^{-2} \Bigl(  \sum_{\tau = 1\vee (k-w)}^{i} \phi(s_h^\tau, a_h^\tau)\phi(s_h^\tau, a_h^\tau)^\top  \Bigr) \biggr) \notag\\
	& \qquad = \lambda_{\max} \bigl(\Psi M_2 \Psi^\top (\Psi M_1 \Psi^\top)^{-2} \Psi M_2 \Psi^\top\bigr) = \lambda_{\max} (M_2 M_1^{-2} M_2) \leq 1,
	\#
	where we have used the fact that both $M_1$ and $M_2$ are diagonal matrices in the last step.

	Combined with \eqref{eq:con211}, \eqref{eq:con212}, \eqref{eq:con213}, \eqref{eq:con214}, and \eqref{eq:con2141}, we obtain
	\# \label{eq:con215}
	|\mathrm{Term (i.1)}| &\le  \biggl| \phi(s, a)^\top (\Lambda_h^k)^{-1} \biggl( \sum_{\tau = 1\vee (k-w)}^{k-1} \phi(s_h^\tau, a_h^\tau)\phi(s_h^\tau, a_h^\tau)^\top(\theta_h^\tau - \theta_h^k) \biggr) \biggr|  \notag\\
	&\le \sum_{i =1\vee (k-w)}^{k -1} \| \theta_h^i - \theta_h^{i+1} \|_2 .
	\#
	
	\vskip4pt
	\noindent{\bf Term (i.2):}
	By Cauchy-Schwarz inequality, we obtain
	\$
	| \phi(s, a)^\top (\Lambda_h^k)^{-1} ( \lambda \cdot \theta_h^k )  | \le \|  \phi(s, a) \|_{(\Lambda_h^k)^{-1}} \cdot \| \lambda \cdot \theta_h^k \|_{(\Lambda_h^k)^{-1}} .
	\$
	Note the fact that $\Lambda_h^k \succeq \lambda I_d$, which implies $\lambda_{\min}((\Lambda_h^k)^{-1}) \ge \lambda$. We further obtain
	\$
	\|\lambda \cdot \theta_h^k \|_{(\Lambda_h^k)^{-1}}^2 \le \frac{1}{\lambda_{\min}\bigl((\Lambda_h^k)^{-1}\bigr)} \cdot \||\lambda \cdot \theta_h^k \|_2^2   \le \frac{1}{\lambda}\cdot \lambda^2 d = \lambda d .
	\$
	Hence, we have
	\# \label{eq:con216}
	| \phi(s, a)^\top (\Lambda_h^k)^{-1} ( \lambda \cdot \theta_h^k )  | \le \sqrt{\lambda d} \cdot  \|  \phi(s, a) \|_{(\Lambda_h^k)^{-1}} .
	\#
	Setting $\beta_k = \sqrt{\lambda d}$  for any $k \in [K]$ in the bonus function $B_h^k$ defined in \eqref{eq:bonus}. Plugging \eqref{eq:con215} and \eqref{eq:con216} into \eqref{eq:con210}, we obtain 
	\# \label{eq:con2161}
	| \phi(s, a)^\top (\hat{\theta}_h^k - \theta_h^k) | \le  B_h^k(s, a) + \sum_{i =1\vee (k-w)}^{k -1} \| \theta_h^i - \theta_h^{i+1} \|_2
	\#
	for any $(k, h) \in [K] \times [H]$. Hence, for any $(s, a) \in \cS \times \cA$, we have
	\# \label{eq:con217}
	\phi(s, a)\hat{\theta}_h^k + B_h^k(s, a) - r_h^k(s,a) \le 2B_h^k(s, a) + \sum_{i =1\vee (k-w)}^{k -1} \| \theta_h^i - \theta_h^{i+1} \|_2.
	\#

	\vskip4pt
	\noindent{\bf Term (ii):} 
	Recall that $\eta_h^k$ defined in \eqref{eq:def:eta} takes the form 
	\$
	\eta_h^k(\cdot,\cdot) = \int_\cS \psi(\cdot, \cdot, s') \cdot V_{h+1}^{k}(s') \mathrm{d} s' 
	\$
	for any $(k, h) \in [K] \times [H]$ and $(s, a) \in \cS \times \cA$.
	 Meanwhile, by Assumption \ref{assumption:linear:mdp}, we obtain
	\# \label{eq:ucb111}
	(\PP_h^k V_{h+1}^k)(s, a) &= \int_\cS \psi (s, a, s')^\top \xi_h^k \cdot V_{h+1}^k(s')\mathrm{d} s' \\
	& = \eta_h^k(s,a)^\top \xi_h^k = \eta_h^k(s,a)^\top(A_h^k)^{-1}A_h^k \xi_h^k \notag
    \#
    for any $(k, h) \in [K] \times [H]$ and $(s, a) \in \cS \times \cA$. Recall the definition of $A_h^k$ in \eqref{eq:def:w}, we have
    \$
    (\PP_h^k V_{h+1}^k)(s,a) &= \eta_h^k(s, a)^\top(A_h^k)^{-1}\biggl( \sum_{\tau = 1\vee(k - w)}^{k-1} \eta_h^{\tau}(s_h^\tau, a_h^\tau)\eta_h^{\tau}(s_h^\tau, a_h^\tau)^\top\xi_h^k + \lambda' \cdot \xi_h^k \biggr)  \\
    & = \eta_h^k(s, a)^\top(A_h^k)^{-1}\biggl( \sum_{\tau = 1\vee(k - w)}^{k-1} \eta_h^{\tau}(s_h^\tau, a_h^\tau) \cdot (\PP_h^kV_{h+1}^\tau)(s_h^\tau, a_h^\tau) + \lambda' \cdot \xi_h^k \biggr)
    \$
    for any $(k, h) \in [K] \times [H]$ and $(s, a) \in \cS \times \cA$. Here the second equality is obtained by \eqref{eq:ucb111}. Recall the definition of $\hat{\xi}_h^k$ in \eqref{eq:def:w}, we have
    \# \label{eq:ucb112}
    & \eta_h^k(\cdot, \cdot)^\top \hat{\xi}_h^k - (\PP_h^k V_{h+1}^k)(s,a) \\
    & \qquad = \underbrace{\eta_h^k(s, a)^\top(A_h^k)^{-1}\biggl( \sum_{\tau = 1\vee(k - w)}^{k-1} \eta_h^{\tau}(s_h^\tau, a_h^\tau) \cdot \bigl( V_{h+1}^\tau(s_{h+1}^\tau) - (\PP_h^k V_{h+1}^\tau)(s_h^\tau, a_h^\tau) \bigr) \biggr)}_{\displaystyle{\rm (ii.1)}} \notag\\
	&\qquad - \underbrace{\lambda' \cdot \eta_h^k(s, a)^\top(A_h^k)^{-1}\xi_h^k }_{\displaystyle{\rm (ii.2)}} \notag
    \#
     for any $(k, h) \in [K] \times [H]$ and $(s, a) \in \cS \times \cA$. 
    \vskip4pt
    \noindent{\bf Term (ii.1):} 
    We can decompose $\rm Term (ii.1)$ as 
    \# \label{eq:ucb1140}
    \text{Term (ii.1)} = & \underbrace{ \eta_h^k(s, a)^\top(A_h^k)^{-1}\biggl( \sum_{\tau = 1\vee(k - w)}^{k-1} \eta_h^{\tau}(s_h^\tau, a_h^\tau) \cdot \bigl( V_{h+1}^\tau(s_{h+1}^\tau) - (\PP_h^\tau V_{h+1}^\tau)(s_h^\tau, a_h^\tau) \bigr) \biggr)}_{\rm (ii.1.1)} \\
    & \quad + \underbrace{ \eta_h^k(s, a)^\top(A_h^k)^{-1}\biggl( \sum_{\tau = 1\vee(k - w)}^{k-1} \eta_h^{\tau}(s_h^\tau, a_h^\tau) \cdot \bigl( (\PP_h^\tau V_{h+1}^\tau)(s_h^\tau, a_h^\tau)  - (\PP_h^k V_{h+1}^\tau)(s_h^\tau, a_h^\tau) \bigr) \biggr) }_{\rm (ii.1.2)} \notag
    \#
    By the definition of $A_h^k$ in \eqref{eq:def:w}, $(A_h^k)^{-1}$ is a positive definite matrix. Hence, by Cauchy-Schwarz inequality, 
    \# \label{eq:ucb113}
    &| \text{Term (ii.1.1)} | \\
	& \quad \le \sqrt{\eta_h^k(s, a)^\top(A_h^k)^{-1}\eta_h^k(s, a) } \cdot \biggl\| \sum_{\tau = 1\vee(k - w)}^{k-1} \eta_h^{\tau}(s_h^\tau, a_h^\tau) \cdot \bigl( V_{h+1}^\tau(s_{h+1}^\tau) - (\PP_h^\tau V_{h+1}^\tau)(s_h^\tau, a_h^\tau) \bigr) \biggr\|_{(A_h^k)^{-1}} \notag
    \#
     for any $(k, h) \in [K] \times [H]$ and $(s, a) \in \cS \times \cA$. Under the event $\mathcal{E}$ defined in \eqref{eq:event} of Lemma \ref{lem:eventlm},  which happens with probability at least $1-\zeta/2$, it holds that
     \# \label{eq:ucb114}
     | \text{Term (ii.1.1)} | \le C'' \sqrt{dH^2\cdot \log(dT/\zeta)} \cdot \sqrt{\eta_h^k(s, a)^\top(A_h^k)^{-1}\eta_h^k(s, a) } 
     \#
      for any $(k, h) \in [K] \times [H]$ and $(s, a) \in \cS \times \cA$. Here $C'' > 0$ is an absolute constant defined in Lemma \ref{lem:eventlm}. Meanwhile, by \eqref{eq:ucb111}, we have $(\PP_h^k V_{h+1}^\tau)(s, a) = \eta_h^\tau(s, a)^\top \xi_h^k$ and $\PP_h^\tau V_{h+1}^\tau (s, a)= \eta_h^\tau(s, a)^\top  \xi_h^\tau$ for any $(s, a) \in \cS \times \cA$, which implies
      \$
       | \text{Term (ii.1.2)} | &= \biggl| \eta_h^k(s, a)^\top(A_h^k)^{-1}\biggl( \sum_{\tau = 1\vee(k - w)}^{k-1} \eta_h^{\tau}(s_h^\tau, a_h^\tau)\eta_h^{\tau}(s_h^\tau, a_h^\tau)^\top (\xi_h^\tau - \xi_h^k)  \biggr) \biggr|   \\
       & \le \|\eta_h^k(s,a)\|_2 \cdot \biggl\| (A_h^k)^{-1}\biggl( \sum_{\tau = 1\vee(k - w)}^{k-1} \eta_h^{\tau}(s_h^\tau, a_h^\tau)\eta_h^{\tau}(s_h^\tau, a_h^\tau)^\top (\xi_h^\tau - \xi_h^k)  \biggr) \biggr\|_2 \\
       & \le H\sqrt{d} \cdot \biggl\| (A_h^k)^{-1}\biggl( \sum_{\tau = 1\vee(k - w)}^{k-1} \eta_h^{\tau}(s_h^\tau, a_h^\tau)\eta_h^{\tau}(s_h^\tau, a_h^\tau)^\top (\xi_h^\tau - \xi_h^k)  \biggr) \biggr\|_2 ,
      \$
      where the last inequality is obtained by Assumption \ref{assumption:linear:mdp}. Then, by the same derivation of \eqref{eq:con215}, we have
      \# \label{eq:ucb1141}
       | \text{Term (ii.1.2)} |  \le H\sqrt{d} \cdot \sum_{i = 1}^{ k- 1} \| \xi_h^i - \xi_h^{i+1} \|_2.
      \#
      Plugging \eqref{eq:ucb114} and \eqref{eq:ucb1141} into \eqref{eq:ucb1140}, we obtain
      \# \label{eq:ucb1142}
       | \text{Term (ii.1)} | \le C'' \sqrt{dH^2\cdot \log(dT/\zeta)} \cdot \sqrt{\eta_h^k(s, a)^\top(A_h^k)^{-1}\eta_h^k(s, a) } + H\sqrt{d} \cdot \sum_{i = 1\vee(k - w)}^{ k- 1} \| \xi_h^i - \xi_h^{i+1} \|_2.
      \#
      \vskip4pt
      \noindent{\bf Term (ii.2):} 
       For any $(k, h) \in [K] \times [H]$ and $(s, a) \in \cS \times \cA$, we have 
       \# \label{eq:ucb115}
       | \text{Term (ii.2)} | & \le  \lambda' \cdot \sqrt{\eta(s, a)^\top(A^{k}_h)^{-1}\eta(s,a)} \cdot \| \xi_h^k \|_{(A^{k}_h)^{-1}} \\
       &\le \sqrt{\lambda'}\cdot \sqrt{\eta(s, a)^\top(A^{k}_h)^{-1}\eta(s,a)} \cdot  \| \xi_h^k \|_2 \notag \\
       &\le  \sqrt{\lambda'd}\cdot \sqrt{\eta(s, a)^\top(A^{k}_h)^{-1}\eta(s,a)}, \notag 
       \#
       where the first inequality follows from Cauchy-Schwarz inequality, the second inequality follows from the fact that $A^{k}_h\succeq \lambda' \cdot{I_d}$ and the last inequality is obtained by Assumption \ref{assumption:linear:mdp}. Plugging \eqref{eq:ucb1142} and \eqref{eq:ucb115} into \eqref{eq:ucb112}, we have
       \# \label{eq:ucb116}
       & | \eta_h^k(\cdot, \cdot)^\top \hat{\xi}_h^k - (\PP_h^k V_{h+1}^k)(s,a) | \notag \\
       & \qquad \le C' \sqrt{dH^2\cdot \log(dT/\zeta)} \cdot \sqrt{\eta_h^k(s, a)^\top(A_h^k)^{-1}\eta_h^k(s, a) } + H\sqrt{d} \cdot \sum_{i = 1\vee(k - w)}^{ k- 1} \| \xi_h^i - \xi_h^{i+1} \|_2
       \#
       for any $(k, h) \in [K] \times [H]$ and $(s, a) \in \cS \times \cA$. Here $C' > 1$ is another absolute constant. Setting 
       \$
       \beta' = C' \sqrt{dH^2\cdot \log(dT/\zeta)}
       \$
       in the bonus function $\Gamma_h^k$ defined in \eqref{eq:bonus}. Hence, by \eqref{eq:ucb116}, we have
       \# \label{eq:ucb1161}
       | \eta_h^k(s, a)^\top \hat{\xi}_h^k - (\PP_h^k V_{h+1}^k)(s,a) | \le \Gamma_h^k(s, a) + H\sqrt{d} \cdot \sum_{i = 1\vee(k - w)}^{ k- 1} \| \xi_h^i - \xi_h^{i+1} \|_2
       \#
       for any $(k, h) \in [K] \times [H]$ and $(s, a) \in \cS \times \cA$ under event $\mathcal{E}$. Hence, 
       \# \label{eq:ucb117}
       \eta_h^k(s, a)^\top \hat{\xi}_h^k + \Gamma_h^k(s, a) - \PP_h^k V_{h+1}^k(s,a) \le 2\Gamma_h^k(s, a) + H\sqrt{d} \cdot \sum_{i = 1\vee(k - w)}^{ k- 1} \| \xi_h^i - \xi_h^{i+1} \|_2
       \#
       for any $(k, h) \in [K] \times [H]$ and $(s, a) \in \cS \times \cA$ under event $\mathcal{E}$.
       Combining \eqref{eq:con217} and \eqref{eq:ucb117}, we have 
       \# \label{eq:ucb130}
       - l_h^k(s,a) &= Q_h^k(s, a) - (r_h^k + \PP_h^k V_{h+1}^k)(s, a) \notag \\
       & \le  2B_h^k(s, a) + 2\Gamma_h^k(s, a) + \sum_{i =1\vee (k-w)}^{k -1} \| \theta_h^i - \theta_h^{i+1} \|_2 + H\sqrt{d} \cdot \sum_{i = 1\vee(k - w)}^{ k- 1} \| \xi_h^i - \xi_h^{i+1} \|_2 .
       \#
       Then, we show that $l_h^k(s, a) \le \sum_{i =1\vee (k-w)}^{k -1} \| \theta_h^i - \theta_h^{i+1} \|_2 + H\sqrt{d} \cdot \sum_{i = 1\vee(k - w)}^{ k- 1} \| \xi_h^i - \xi_h^{i+1} \|_2$ for any $(k, h) \in [K] \times [H]$ and $(s, a) \in \cS \times \cA$ under event $\mathcal{E}$.
       \# \label{eq:ucb131}
       l_h^k(s,a) &=  (r_h^k + \PP_h^k V_{h+1}^k)(s,a) - Q_h^k \notag \\
      & = (r_h^k + \PP_h^k V_{h+1}^k)(s,a) - \min\{\phi(s, a)\hat{\theta}_h^k + \eta_h^k(s, a)^\top \hat{\xi}_h^k + B_h^k(s, a) + \Gamma_h^k(s,a), H - h +1\}  \notag \\
      & = \max\{r_h^k(s,a) - \phi(s, a)\hat{\theta}_h^k - B_h^k(s, a) + (\PP_h^k V_{h+1}^k)(s,a) -  \eta_h^k(\cdot, \cdot)^\top \hat{\xi}_h^k - \Gamma_h^k(s, a), \notag \\ & \qquad\qquad (r_h^k + \PP_h^k V_{h+1}^k)(s,a) - (H-h+1) \} .
       \# 
       By \eqref{eq:con2161} and \eqref{eq:ucb1161}, we have 
       \# \label{eq:ucb132}
       &r_h^k(s,a) - \phi(s, a)\hat{\theta}_h^k - B_h^k(s, a) + (\PP_h^k V_{h+1}^k)(s,a) -  \eta_h^k(\cdot, \cdot)^\top \hat{\xi}_h^k - \Gamma_h^k(s, a)  \notag \\
       & \qquad \le \sum_{i =1\vee (k-w)}^{k -1} \| \theta_h^i - \theta_h^{i+1} \|_2 + H\sqrt{d} \cdot \sum_{i = 1\vee(k - w)}^{ k- 1} \| \xi_h^i - \xi_h^{i+1} \|_2.
       \#
       Also, we note the fact that $V_{h+1}^k \le H - h$, it is not difficult to show that
       \# \label{eq:ucb133}
       (r_h^k + \PP_h^k V_{h+1}^k)(s,a) - (H-h+1) \le 0 .
       \#
       Plugging \eqref{eq:ucb132} and \eqref{eq:ucb133} into \eqref{eq:ucb131}, we obtain
       \# \label{eq:ucb134}
       l_h^k (s,a) \le \sum_{i =1\vee (k-w)}^{k -1} \| \theta_h^i - \theta_h^{i+1} \|_2 + H\sqrt{d} \cdot \sum_{i = 1\vee(k - w)}^{ k- 1} \| \xi_h^i - \xi_h^{i+1} \|_2
       \#
       for any $(k, h) \in [K] \times [H]$ and $(s, a) \in \cS \times \cA$ under event $\mathcal{E}$.
       Combining \eqref{eq:ucb130} and \eqref{eq:ucb134}, we finish the proof of Lemma \ref{lem:ucb}.
\end{proof}

\subsection{Proof of Lemma \ref{lem:visitation2}} \label{appendix:pf:visitation2}
\begin{proof}
	Our proof has some similarities with the proof of Lemma 4 in \citet{fei2020dynamic}. However, they only consider the stationary transition kernels, which makes the analysis easier. Our proof relies on the following lemma.
\begin{lemma}\label{lem:visitation}
	For any $(h, k') \in [H] \times [K]$, $\{k_j\}_{j = 1}^{h - 1} \in [K]$,  $j \in [h - 1]$, $(s_1, s_h) \in \cS \times \cS$, and policies $\{\pi^{i}\}_{i \in [H]} \cup \{\pi'\}$, we have
	\$
	& |P_1^{k_1, \pi(1)}\cdots P_j^{k_j, \pi(j)} \cdots P_{h - 1}^{k_{h-1}, \pi(h -1)}(s_h\,|\,s_1) - P_1^{k_1, \pi(1)}\cdots P_j^{k', \pi'} \cdots P_{h - 1}^{k_{h - 1}, \pi(h - 1)}(s_h\,|\,s_1) |\\
	&\qquad \le \| \pi_j^{(j)} - \pi'_j \|_{\infty, 1} + \max_{(s, a) \in \cS \times \cA} \|  P_j^{k_j}( \cdot \,|\, s_j, a) - P_j^{k'}(\cdot \,|\, s, a) \|_1 .
	\$
\end{lemma}

\begin{proof}
	First, we have 
	\# \label{eq:vis0}
	& |P_1^{k_1, \pi(1)}\cdots P_j^{k_j, \pi(j)} \cdots P_{h - 1}^{k_{h-1}, \pi(h -1)}(s_h\,|\,s_1) - P_1^{k_1, \pi(1)}\cdots P_j^{k', \pi'} \cdots P_{h - 1}^{k_{h - 1}, \pi(h - 1)}(s_h\,|\,s_1) | \notag\\
	& \quad  \le |P_1^{k_1, \pi(1)}\cdots P_j^{k_j, \pi(j)} \cdots P_{h - 1}^{k_{h-1}, \pi(h -1)}(s_h\,|\,s_1) - P_1^{k_1, \pi(1)}\cdots P_j^{k', \pi(j)} \cdots P_{h - 1}^{k_{h-1}, \pi(h -1)}(s_h\,|\,s_1) | \notag \\
	&\quad \quad  +  |P_1^{k_1, \pi(1)}\cdots P_j^{k', \pi(j)} \cdots P_{h - 1}^{k_{h-1}, \pi(h -1)}(s_h\,|\,s_1) - P_1^{k_1, \pi(1)}\cdots P_j^{k', \pi'} \cdots P_{h - 1}^{k_{h-1}, \pi(h -1)}(s_h\,|\,s_1) | . 
	 \#
	 By the definition of Markov kernel, we have
	 {\small
	 \# \label{eq:vis1}
	 & |P_1^{k_1, \pi(1)}\cdots P_j^{k_j, \pi(j)} \cdots P_{h - 1}^{k_{h-1}, \pi(h -1)}(s_h\,|\,s_1) - P_1^{k_1, \pi(1)}\cdots P_j^{k', \pi(j)} \cdots P_{h - 1}^{k_{h-1}, \pi(h -1)}(s_h\,|\,s_1) |  \\
	 & \quad \le \sum_{s_2, s_3, \cdots, s_{h-1} }|  P_j^{k_j, \pi(j)}(s_{j + 1} \,|\, s_j) - P_j^{k', \pi(j)}(s_{j + 1} \,|\, s_j) | \cdot \prod_{i \in [h - 1] \backslash j} P_i^{k_i, \pi(i)}(s_{i+1} \,|\, s_i)  \notag \\
	 & \quad \le \sum_{s_2,  \cdots, s_j, s_{j+2}, \cdots, s_{h-1} }  \sum_{s_{j+1}}  |P_j^{k_j, \pi(j)}(s_{j + 1} \,|\, s_j) - P_j^{k', \pi(j)}(s_{j + 1} \,|\, s_j) | \cdot  \max_{s_{j+1} \in \cS} \prod_{i \in [h - 1] \backslash j} P_i^{k_i, \pi(i)}(s_{i+1} \,|\, s_i) \notag \\
	 & \quad \le \sum_{s_2,  \cdots, s_{j - 1}, s_{j+2}, \cdots, s_{h-1} } \max_{s_j \in \cS} \sum_{s_{j+1}}  |P_j^{k_j, \pi(j)}(s_{j + 1} \,|\, s_j) - P_j^{k', \pi(j)}(s_{j + 1} \,|\, s_j) | \cdot  \sum_{s_j} \max_{s_{j+1} \in \cS} \prod_{i \in [h - 1] \backslash j} P_i^{k_i, \pi(i)}(s_{i+1} \,|\, s_i) , \notag
	 \#}
	 where the last two inequalities is obtained by H\"{o}lder's inequality. By the definition of Markov kernel, we further have
	 \# \label{eq:vis2}
	 |P_j^{k_j, \pi(j)}(s_{j + 1} \,|\, s_j) - P_j^{k', \pi(j)}(s_{j + 1} \,|\, s_j) | &= \big| \sum_a \pi^{(j)}(a \,|\, s_j) \bigl( P_j^{k_j}(s_{j+1} \,|\, s_j, a) - P_j^{k_j}(s_{j+1} \,|\, s_j, a) \bigr)        \big|  \notag \\
	 & \le \max_a |  P_j^{k_j}(s_{j+1} \,|\, s_j, a) - P_j^{k_j}(s_{j+1} \,|\, s_j, a) | .
	 \#
	 Hence, we obtain 
	 \# \label{eq:vis3}
	 &|P_1^{k_1, \pi(1)}\cdots P_j^{k_j, \pi(j)} \cdots P_{h - 1}^{k_{h-1}, \pi(h -1)}(s_h\,|\,s_1) - P_1^{k_1, \pi(1)}\cdots P_j^{k', \pi(j)} \cdots P_{h - 1}^{k_{h-1}, \pi(h -1)}(s_h\,|\,s_1) | \notag\\
	 & \quad \le \max_{(s_j, a) \in \cS \times \cA} \|  P_j^{k_j}( \cdot \,|\, s_j, a) - P_j^{k'}(\cdot \,|\, s_j, a) \|_1 \cdot \sum_{s_2,  \cdots, s_j, s_{j+2}, \cdots, s_{h-1} } \max_{s_{j+1} \in \cS} \prod_{i \in [h - 1] \backslash j} P_i^{k_i, \pi(i)}(s_{i+1} \,|\, s_i) \notag \\
	 & \quad = \max_{(s_j, a) \in \cS \times \cA} \|  P_j^{k_j}( \cdot \,|\, s_j, a) - P_j^{k'}(\cdot \,|\, s_j, a) \|_1 \notag \\
	 & \quad \qquad \times \sum_{ s_{j+2}, \cdots, s_{h-1} } \max_{s_{j+1} \in \cS} \prod_{i = j + 1}^{h - 1}  P_i^{k_i, \pi(i)}(s_{i+1} \,|\, s_i) \cdot \sum_{s_2, \cdots, s_j} \prod_{i =1}^{j - 1}P_i^{k_i, \pi(i)}(s_{i+1} \,|\, s_i) \notag \\
	 & \quad \le  \|  P_j^{k_j} - P_j^{k'} \|_{\infty},
	 \#
	 where the last inequality follows from the definition of $\|  P_j^{k_j} - P_j^{k'} \|_{\infty},$ the facts that 
\begin{align*}
    \sum_{s_{j+2}, \cdots, s_{h-1}} 
    \max_{s_{j+1} \in \mathcal{S}} 
    \prod_{i = j + 1}^{h - 1} 
    P_i^{k_i, \pi(i)}(s_{i+1} \mid s_i) &\leq 1, \quad 
    \sum_{s_2, \cdots, s_j} 
    \prod_{i = 1}^{j - 1} 
    P_i^{k_i, \pi(i)}(s_{i+1} \mid s_i) &\leq 1.
\end{align*}

	 Moreover, by Lemma 5 in \cite{fei2020dynamic}, we have 
	 \# \label{eq:vis4}
	&|P_1^{k_1, \pi(1)}\cdots P_j^{k', \pi(j)} \cdots P_{h - 1}^{k_{h-1}, \pi(h -1)}(s_h\,|\,s_1) - P_1^{k_1, \pi(1)}\cdots P_j^{k', \pi'} \cdots P_{h - 1}^{k_{h-1}, \pi(h -1)}(s_h\,|\,s_1) | \notag\\
	& \qquad \le \| \pi_j^{(j)} - \pi'_j \|_{\infty, 1} .
	 \#
	 Plugging \eqref{eq:vis3} and \eqref{eq:vis4} into \eqref{eq:vis0}, we conclude the proof of Lemma \ref{lem:visitation}.
	\end{proof}
	 Back to our proof, for any $(k, t, h) \in [K] \times [K] \times [H]$, we have
	 \# \label{eq:vis5}
	 &|(\EE_{\pi^{*,t}} - \EE_{\pi^{*, t-1}})[\mathbf{1}(s_h) ] \notag \\
	 &\qquad \le \|P_1^{t, \pi^{*,t}} P_2^{t, \pi^{*,t}} \cdots P_{h-1}^{t, \pi^{*,t}}(\cdot \,|\, s_1) - P_1^{t-1, \pi^{*,t-1}} P_2^{t-1, \pi^{*,t-1}} \cdots P_{h-1}^{t-1, \pi^{*,t-1}}(\cdot \,|\, s_1)\|_{\infty} \notag \\
	 & \qquad \le \sum_{j = 1}^{h-1} \|P_1^{t-1, \pi^{*,t-1}} \cdots P_{j-1}^{t-1, \pi^{*,t-1}} P_j^{t, \pi^{*,t}} \cdots P_{h-1}^{t, \pi^{*,t}}(\cdot \,|\, s_1) \notag \\
	 &\qquad \qquad \qquad \qquad - P_1^{t-1, \pi^{*,t-1}} \cdots P_{j-1}^{t-1, \pi^{*,t-1}} P_j^{t-1, \pi^{*,t-1}} \cdots P_{h-1}^{t, \pi^{*,t}}(\cdot \,|\, s_1)\|_{\infty} \notag \\
	 & \qquad\le \sum_{j = 1}^{h-1} (\|P_j^{t} - P_j^{t-1}\|_{\infty} + \|\pi_{j}^{*, t} - \pi_j^{*, t-1}\|_{\infty}),
	 \#
	 where the second inequality uses the triangle inequality and the last inequality follows from Lemma \ref{lem:visitation}. Then we know 
	 \#
	 &\sum_{i=1}^{\rho}\sum_{k=(i - 1)\tau + 1}^{i\tau}\sum_{h=1}^H (\EE_{\pi^{*, k}} - \EE_{\pi^{*, (i - 1)\tau + 1 }} ) \bigl[ \la Q^{k}_h(s_h,\cdot), \pi^{*, k}_h(\cdot\,|\,s_h) - \pi^k_h(\cdot\,|\,s_h) \ra  \bigr] \notag\\
	 & \qquad \le \sum_{i=1}^{\rho}  \sum_{k=(i - 1)\tau + 1}^{i\tau}\sum_{h=1}^H \sum_{t = (i-1)\tau +2}^k H \cdot  \sum_{j = 1}^{h-1} (\|P_j^{t} - P_j^{t-1}\|_{\infty} + \|\pi_{j}^{*, t} - \pi_j^{*, t-1}\|_{\infty}) \notag \\
	 &\qquad \le H \sum_{h=1}^H \Bigl(\sum_{i=1}^{\rho}  \sum_{k=(i - 1)\tau + 1}^{i\tau}\sum_{t = (i -1)\tau + 1}^{i\tau}\sum_{j = 1}^H (\|P_j^{t} - P_j^{t-1}\|_{\infty} + \|\pi_{j}^{*, t} - \pi_j^{*, t-1}\|_{\infty}) \Bigr), \notag \\
	 &\qquad\le H \sum_{h=1}^H \Bigl(\tau  \sum_{k=1}^{K}\sum_{j = 1}^H (\|P_j^{t} - P_j^{t-1}\|_{\infty} + \|\pi_{j}^{*, t} - \pi_j^{*, t-1}\|_{\infty}) \Bigr)
	 \#
	 where the first inequality follows from \eqref{eq:vis5}. Together with the definitions of $\Delta_P$ and $P_T$, we have
	 \$
	 \sum_{i=1}^{\rho}\sum_{k=(i - 1)\tau + 1}^{i\tau}\sum_{h=1}^H (\EE_{\pi^{*, k}} - \EE_{\pi^{*, (i - 1)\tau + 1 }} ) \bigl[ \la Q^{k}_h(s_h,\cdot), \pi^{*, k}_h(\cdot\,|\,s_h) - \pi^k_h(\cdot\,|\,s_h) \ra  \bigr] \le  \tau H^2(\Delta_P + P_T),
	 \$
	 which concludes the proof of Lemma \ref{lem:visitation2}.
\end{proof}

\section{Proof of Main Results}

\subsection{Proof of Theorem \ref{thm:regret:dyn}} \label{appendix:proof:regret:dyn}
\begin{proof}
	By Lemma \ref{lemma:regret:decomposition2}, we decompose dynamic regret of Algorithm \ref{alg:2} into four parts:
	\# \label{eq:3360}
   	\text{D-Regret}(T) &= \sum_{k=1}^K \bigl(V^{\pi^{*, k},k}_1(s_1^k) - V^{\pi^k,k}_1(s_1^k)\bigr) \\
    &=  \underbrace{\sum_{i=1}^{\rho}\sum_{k=(i - 1)\tau + 1}^{i\tau}\sum_{h=1}^H \EE_{\pi^{*, k}} \bigl[ \la Q^{k}_h(s_h,\cdot), \pi^{*, k}_h(\cdot\,|\,s_h) - \pi^k_h(\cdot\,|\,s_h) \ra \bigr]}_{\dr (i)} + \underbrace{ \cM_{K, H, 2}}_{\dr (ii)}   \notag\\
    & \qquad +\underbrace{\sum_{i=1}^{\rho}\sum_{k=(i - 1)\tau + 1}^{i\tau}\sum_{h=1}^H \EE_{\pi^{*, k}}[l^{k}_h(s_h,a_h)]}_{\dr (iii)}    + \underbrace{ \sum_{i=1}^{\rho}\sum_{k=(i - 1)\tau + 1}^{i\tau}\sum_{h=1}^H  -l^{k}_h(s^k_h,a^k_h)}_{\dr (iv)}, \notag
    \#
    Now we establish the upper bound of these four parts, respectively. 
    \vskip4pt
    \noindent{\bf Upper Bounding (i):} 
    By Lemma \ref{lem:omd:term2}, we have
   \# \label{eq:3361}
   \text{Term} {\rm (i)} \le \sqrt{2H^3T\rho \log|\cA|}  +  2\tau H^2(P_{T} + \sqrt{d}\Delta) .
   \# 
   	
	Then we discuss several cases.
	\begin{itemize}
		\item 
		If $0 \le P_{T} + \sqrt{d}\Delta \le \sqrt{\frac{\log|\cA|}{K}}$, then $\tau = \Pi_{[1,K]}(\lfloor (\frac{T\sqrt{\log|\cA|}}{H(P_{T} + \sqrt{d}\Delta)})^{2/3} \rfloor ) = K$, which implies that $\rho = 1$. Then \eqref{eq:3361} yields
		\# \label{eq:33611}
		\text{Term} {\rm (i)}  &\le 2H^2\sqrt{K\log|\cA|} +  \cdot H^2\sqrt{K\log|\cA|} = 3\sqrt{H^3T\log|\cA|} .
		\#
		\item 
		If $\sqrt{\frac{\log|\cA|}{K}} \le P_{T} + \sqrt{d}\Delta \le 2^{-3/2} \cdot K\sqrt{\log|\cA|}$, we have $\tau \in [2, K]$ and \eqref{eq:3361} yields
		\# \label{eq:33612}
		\text{Term} {\rm (i)}  &\le 2\cdot \frac{1}{\sqrt{\tau}} H^2K\sqrt{\log|\cA|} +  \cdot \tau H^2\sqrt{K\log|\cA|}  \notag \\ 
		&\le 5 (H^2T\sqrt{\log|\cA|})^{2/3}(P_{T} + \sqrt{d}\Delta)^{1/3}  .
		\#
		\item 
		If $P_{T} > 2^{-3/2} \cdot K\sqrt{\log|\cA|}$, we have $\tau = 1$ and therefore $\rho = K$. Then \eqref{eq:3361} implies
		\# \label{eq:33613}
		\text{Term} {\rm (i)}  &\le 2H^2K \sqrt{\log|\cA|} +  \cdot H^2P_{T} \le 9 H^2(P_{T} + \sqrt{d}\Delta) .
		\#
	\end{itemize}
    Combining \eqref{eq:33611}, \eqref{eq:33612} and \eqref{eq:33613}, we have
    \# \label{eq:33614}
    \text{Term} {\rm (i)}  \le \left\{
    \begin{array}{rcl}
    	&\sqrt{H^3T\log|\cA|},            & {\text{if } 0 \le P_{T} + \sqrt{d}\Delta \le \sqrt{\frac{\log|\cA|}{K}} },\\
    	&(H^2T\sqrt{\log|\cA|})^{2/3}(P_{T} + \sqrt{d}\Delta)^{1/3},       & {\text{if } \sqrt{\frac{\log|\cA|}{K}} \le P_{T} + \sqrt{d}\Delta \lesssim K\sqrt{\log|\cA|} },\\
    	&H^2(P_{T} + \sqrt{d}\Delta),         & {\text{if } P_{T} + \sqrt{d}\Delta \gtrsim K\sqrt{\log|\cA|} },
    \end{array} \right.
    \#

    \vskip4pt
    \noindent{\bf Upper Bounding (ii):}
    Recall that 
    \$
    \cM_{K,H, 2} = \sum_{k=1}^K\sum_{h=1}^{H} (D_{k,h,1}+D_{k,h,2}) .
    \$
    Here the $D_{k,h,1}$ and $D_{k,h,2}$ defined in \eqref{eq:120} take the following forms,
    \$
    & D_{k,h,1} = \bigl(\mathbb{I}_h^k(Q_h^k - Q_h^{\pi^k,k})\bigr)(s_h^k) - Q_h^k - Q_h^{\pi^k,k} , \\
    & D_{k,h,2} =  \bigl( \PP_h^k(V_{h+1}^k - V_{h+1}^{\pi^k,k})  \bigr)(s_h^k,a_h^k) - (V_{h+1}^k - V_{h+1}^{\pi^k,k})(s_{h+1}^k)  . \notag
    \$
    By the truncation of $\phi(\cdot, \cdot)\hat{\theta}_h^k + \eta_h^k(\cdot, \cdot)^\top \hat{\xi}_h^k + B_h^k(\cdot, \cdot) + \Gamma_h^k(\cdot, \cdot)$ into range $[0,H-h+1]$ in \eqref{eq:evaluation}, we know that $Q^{k}_h,Q^{\pi^k,k}_h,V^{k}_{h+1},V^{\pi^k,k}_{h+1} \in[0,H]$, which implies that $|D_{k,h,1}| \le 2H$ and $|D_{k,h,2}| \le 2H$ for any $(k,h) \in [H] \times [K]$. Applying the Azuma-Hoeffding inequality to the martingale $\cM_{K,H, 2} $, we obtain
    \$
    P(|\cM_{K,H, 2} | > \varepsilon) \le 2 \exp \biggl(  \frac{-\varepsilon^2}{16H^3K} \biggr).
    \$
    For any $\zeta \in (0,1)$, if we set $\varepsilon = \sqrt{16H^3K\cdot\log(4/\zeta)}$, we have
    \# \label{eq:martingale:bound}
    |\cM_{K,H,2}|\le \sqrt{16H^2T\cdot\log(4/\zeta)}
    \#
    with probability at least $1- \zeta/2$.
    
    \vskip4pt
    \noindent{\bf Upper Bounding (iii):} 
    By Lemma \ref{lem:ucb}, it holds with probability at least $1 - \zeta/2$ that
    \$
    l_h^k (s, a) \le \sqrt{dw} \cdot \sum_{i =1\vee (k-w)}^{k -1} \| \theta_h^i - \theta_h^{i+1} \|_2 + Hd\sqrt{w} \cdot \sum_{i = 1\vee(k - w)}^{ k- 1} \| \xi_h^i - \xi_h^{i+1} \|_2
    \$
    for any $(k, h) \in [K] \times [H]$ and $(s, a) \in \cS \times \cA$, which implies that  
    \#\label{eq:3333}
    &\sum_{k=1}^K\sum_{h=1}^H \EE_{\pi^*}[l^{k}_h(s_h,a_h)\,|\, s_1 = s_1^k] \notag\\
	&\qquad\le \sum_{k=1}^K\sum_{h=1}^H \sum_{i =1\vee (k-w)}^{k -1} \| \theta_h^i - \theta_h^{i+1} \|_2  + Hd \cdot \sum_{k=1}^K\sum_{h=1}^H \sum_{i =1\vee (k-w)}^{k -1} \| \xi_h^i - \xi_h^{i+1} \|_2 \notag\\
    &\qquad = \cdot \sum_{h=1}^H\sum_{k=1}^K \sum_{i =1\vee (k-w)}^{k -1} \| \theta_h^i - \theta_h^{i+1} \|_2 + Hd\cdot  \sum_{h=1}^H\sum_{k=1}^K \sum_{i =1\vee (k-w)}^{k -1} \| \xi_h^i - \xi_h^{i+1} \|_2 \notag \\
    &\qquad \le \sum_{h=1}^H\sum_{k=1}^K w \cdot  \| \theta_h^k - \theta_h^{k+1}\|_2 + H\sqrt{d} \cdot \sum_{h=1}^H\sum_{k=1}^K w \cdot  \| \xi_h^k - \xi_h^{k+1}\|_2 \notag \\
    &\qquad \le wB_{T} + wH\sqrt{d}B_{P} \le w\Delta H\sqrt{d}.
    \#
    Here the last inequality follows from the definition of total variation budget in \eqref{eq:def:tv}.
    \vskip4pt
    \noindent{\bf Upper Bounding (iv):} 
    As is shown in Lemma \ref{lem:ucb}, it holds with probability at least $1 - \zeta/2$ that
    \$
    - l_h^k (s, a) \le 2B_h^k(s, a) + 2\Gamma_h^k(s, a) + \sum_{i =1\vee (k-w)}^{k -1} \| \theta_h^i - \theta_h^{i+1} \|_2 + H\sqrt{d} \cdot \sum_{i = 1\vee(k - w)}^{ k- 1} \| \xi_h^i - \xi_h^{i+1} \|_2
    \$
    for any $(k, h) \in [K] \times [H]$ and $(s, a) \in \cS \times \cA$. 
    Meanwhile,  by the definitions of $Q_h^k$ and $l_h^k$ in \eqref{eq:evaluation} and \eqref{eq:def:model:error}, we have that $| l_h^k(s, a) | \le 2H$. Hence, 
    \$
    - l_h^k (s, a) \le 2B_h^k(s, a) + 2H\wedge2\Gamma_h^k(s, a) + \sum_{i =1\vee (k-w)}^{k -1} \| \theta_h^i - \theta_h^{i+1} \|_2 + H\sqrt{d} \cdot \sum_{i = 1\vee(k - w)}^{ k- 1} \| \xi_h^i - \xi_h^{i+1} \|_2,
    \$
	which further implies
    \# \label{eq:3337}
    \sum_{k=1}^K\sum_{h=1}^H - l_h^k (s_h^k, a_h^k) &\le 2\sum_{k=1}^K\sum_{h=1}^H B_h^k(s, a) + 2\sum_{k=1}^K\sum_{h=1}^H H\wedge\Gamma_h^k(s, a)  \\
    &\qquad + \sum_{k=1}^K\sum_{h=1}^H\sum_{i =1\vee (k-w)}^{k -1} \| \theta_h^i - \theta_h^{i+1} \|_2 + H\sqrt{d} \cdot \sum_{k=1}^K\sum_{h=1}^H\sum_{i = 1\vee(k - w)}^{ k- 1} \| \xi_h^i - \xi_h^{i+1} \|_2. \notag
    \#
    By Cauchy-Schwarz inequality, we have
    \# \label{eq:3338}
    \sum_{h = 1}^H \sum_{k = 1}^K B_h^k(s_h^k, a_h^k) & \le  \beta \cdot \sum_{h=1}^H\sum_{k=1}^K \sqrt{\phi(s_h^k, a_h^k)^\top (\Lambda_h^k)^{-1}\phi_h^k(s_h^k, a_h^k)}  \notag\\
    & \le   \beta \cdot \sum_{h=1}^H \biggl( K \cdot \sum_{k=1}^K \phi(s_h^k, a_h^k)^\top (\Lambda_h^k)^{-1}\phi_h^k(s_h^k, a_h^k) \biggr)^{1/2} \notag\\
    & = \beta \sqrt{K}\cdot  \sum_{h=1}^H \sqrt{ \sum_{k=1}^K \| \phi(s_h^k, a_h^k)\|_{(\Lambda_h^k)^{-1}} } .
    \#
    As we set $\lambda = 1$, we have that $\Lambda_h^k \succeq I_d$, which implies 
    \$
    \|\phi(s_h^k, a_h^k)\|_{(\Lambda_h^k)^{-1}} \le \|\phi(s_h^k, a_h^k)\|_2  \le 1 
    \$
    for any $(k, h) \in [K] \times [H]$ and $(s, a) \in \cS \times \cA$. By Lemma \ref{lemma:telescope2}, we further have 
    \# \label{eq:3339}
    \sum_{k=1}^K \| \phi(s_h^k, a_h^k)\|_{(\Lambda_h^k)^{-1}} \le 2d \lceil K/w \rceil \log\bigl( (w + \lambda) / \lambda \bigr) \le  4dK \log(w)/w .
    \#
    Combining \eqref{eq:3338} and \eqref{eq:3339}, we further obtain 
    \# \label{eq:3340}
    \sum_{h = 1}^H \sum_{k = 1}^K B_h^k(s_h^k, a_h^k) \le 4dT\sqrt{\log(w)/w} .
    \#
    Meanwhile, by the definition of $\Gamma_h^k$ in \eqref{eq:bonus}, we have
    \$
    \sum_{h = 1}^H \sum_{k = 1}^K H\wedge\Gamma_h^k(s_h^k, a_h^k) & =  \beta' \cdot \sum_{h=1}^H\sum_{k=1}^K H/\beta' \wedge \sqrt{\eta_h^k(s_h^k, a_h^k)^\top (A_h^k)^{-1}\eta_h^k(s_h^k, a_h^k)} .
    \$
    Recall that 
    \$
    \beta' = C' \sqrt{dH^2\cdot \log(dT/\zeta)},
    \$
    which implies that $\beta' > H$. Thus, we have
    \# \label{eq:3341}
    \sum_{h = 1}^H \sum_{k = 1}^K H\wedge\Gamma_h^k(s_h^k, a_h^k) & \le  \beta' \cdot \sum_{h=1}^H\sum_{k=1}^K 1 \wedge \sqrt{\eta_h^k(s_h^k, a_h^k)^\top (A_h^k)^{-1}\eta_h^k(s_h^k, a_h^k)}  \notag\\
    & \le   \beta' \cdot \sum_{h=1}^H \biggl( K \cdot \sum_{k=1}^K 1\wedge \|\eta_h^k(s_h^k, a_h^k)\|_{(A_h^k)^{-1}} \biggr)^{1/2} , 
    \#
    where the second inequality follows from Cauchy-Schwarz inequality. Note the facts that $A_h^1 = \lambda' I_d$ and $\|\eta_h^k(s, a)\|_2 \le \sqrt{d}H$ for any $(k, h) \in [K] \times [H]$ and $(s, a) \in \cS \times \cA$. By the same proof of Lemma \ref{lemma:telescope2}, we have 
    \# \label{eq:3342}
    \sum_{k=1}^K 1\wedge \|\eta_h^k(s_h^k, a_h^k)\|_{(A_h^k)^{-1}} \le 2d \lceil K/w \rceil \log\bigl((wH^2d + \lambda')/ \lambda' \bigr) \le 4dK\log(wH^2d)/w .
    \#
    Combining \eqref{eq:3341} and \eqref{eq:3342}, we have
    \# \label{eq:3344}
    \sum_{h = 1}^H \sum_{k = 1}^K \Gamma_h^k(s_h^k, a_h^k) &\le 2 \beta' \sqrt{dT^2\cdot \log(wH^2d)/w} \notag \\
    & = 4C' dTH\cdot  \sqrt{\log(wH^2d)/w}  \cdot \log(dT/\zeta).
    \#
    where $C'>1$ is an absolute constant and $T = HK$. 
	By the same proof in \eqref{eq:3333}, we have 
    \# \label{eq:3345}
    \sum_{k=1}^K\sum_{h=1}^H \sum_{i =1\vee (k-w)}^{k -1} \| \theta_h^i - \theta_h^{i+1} \|_2  + H\sqrt{d} \cdot \sum_{i = 1\vee(k - w)}^{ k- 1} \| \xi_h^i - \xi_h^{i+1} \|_2 \le w\Delta H\sqrt{d}.
    \#
    Plugging \eqref{eq:3340}, \eqref{eq:3344} and \eqref{eq:3345} into \eqref{eq:3337}, we have
    \# \label{eq:3346}
    \sum_{h = 1}^H \sum_{k = 1}^K -l_h^k(s_h^k, a_h^k) \le w\Delta H\sqrt{d} + 4dT\sqrt{\log(w)/w} + 4C' dTH\cdot  \sqrt{\log(wH^2d)/w}  \cdot \log(dT/\zeta) .
    \#
	Meanwhile, by \eqref{eq:martingale:bound}, \eqref{eq:3333}  and \eqref{eq:3346}, it holds with probability at least $1 - \zeta$ that
	\# \label{eq:3362}
	\text{Term} {\rm (ii)} + \text{Term} {\rm (iii)} + \text{Term} {\rm (iv)} &\le  \sqrt{16H^2T\cdot\log(4/\zeta)} + 2w\Delta H\sqrt{d} \\
	&\qquad + 8dT\sqrt{\log(w)/w} + 8C' dTH\cdot  \sqrt{\log(wH^2d)/w}  \cdot \log(dT/\zeta) \notag\\
	& \lesssim  d^{5/6}\Delta^{1/3}HT^{2/3} \cdot \log(dT/\zeta) . \notag
	\#
	Here we uses the facts that $w = \Theta(d^{1/3}\Delta^{-2/3}T^{2/3})$ . Plugging \eqref{eq:33614} and \eqref{eq:3362} into \eqref{eq:3360}, we finish the proof of Theorem \ref{thm:regret:dyn}.
\end{proof}

\subsection{Proof of Theorem \ref{thm:adv}} \label{appendix:proof:adv}
\begin{proof}
	Following the proof of Lemma \ref{lem:ucb}, we can obtain
	\$
	 - 2\Gamma_h^k(s, a)  - H\sqrt{d} \cdot \sum_{i = 1\vee(k - w)}^{ k- 1} \| \xi_h^i - \xi_h^{i+1} \|_2  \le l_h^k(s, a) \le   H\sqrt{d} \cdot \sum_{i = 1\vee(k - w)}^{ k- 1} \| \xi_h^i - \xi_h^{i+1} \|_2.
	\$
	Notably, here the estimation error $l_h^k$ only comes from the transition estimation. The remaining proof is the same as the Theorem \ref{thm:regret:dyn} except that $\Delta$ is replaced by $\Delta_P$.	
\end{proof}

\subsection{Proof of Theorem \ref{thm:regret:dyn2}} \label{appendix:proof:regret:dyn2}
\begin{proof}
	Let $\tau = K$ in Lemma \ref{lemma:regret:decomposition2}, we have
	\# \label{eq:33600}
	\text{D-Regret}(T) &= \sum_{k=1}^K \bigl(V^{\pi^{*, k},k}_1(s_1^k) - V^{\pi^k,k}_1(s_1^k)\bigr) \\
	&=  \underbrace{\sum_{k=1}^{K}\sum_{h=1}^H \EE_{\pi^{*, k}} \bigl[ \la Q^{k}_h(s_h,\cdot), \pi^{*, k}_h(\cdot\,|\,s_h) - \pi^k_h(\cdot\,|\,s_h) \ra \bigr]}_{\dr (i)} + \underbrace{ \cM_{K, H, 2}}_{\dr (ii)}   \notag\\
	& \qquad +\underbrace{\sum_{k=1}^{K}\sum_{h=1}^H  \EE_{\pi^{*, k}}[l^{k}_h(s_h,a_h)]}_{\dr (iii)}    + \underbrace{ \sum_{k=1}^{K}\sum_{h=1}^H   -l^{k}_h(s^k_h,a^k_h)}_{\dr (iv)}, \notag
	\#
	Since policies $\pi_h^k$ are greedy with respect to $Q_h^k$ for any $(k, h) \in [K] \times [H]$, we have 
	\# \label{eq:336140}
	\sum_{k=1}^{K}\sum_{h=1}^H \EE_{\pi^{*, k}} \bigl[ \la Q^{k}_h(s_h,\cdot), \pi^{*, k}_h(\cdot\,|\,s_h) - \pi^k_h(\cdot\,|\,s_h) \ra\bigr] \le 0.
	\#
	By the same derivation of \eqref{eq:3362} in the proof of Theorem \ref{thm:regret:dyn}, we have
	\# \label{eq:33620}
	\text{Term} {\rm (ii)} + \text{Term} {\rm (iii)} + \text{Term} {\rm (iv)} &\le  \sqrt{16H^2T\cdot\log(4/\zeta)} + 2w\Delta H\sqrt{d} \notag \\
	&\qquad + 8dT\sqrt{\log(w)/w} + 8C' dTH\cdot  \sqrt{\log(wH^2d)/w}  \cdot \log(dT/\zeta) \notag\\
	& \lesssim  d^{5/6}\Delta^{1/3}HT^{2/3} \cdot \log(dT/\zeta) . 
	\#
	Here we uses the facts that $w = \Theta(d^{1/3}\Delta^{-2/3}T^{2/3})$ . Plugging \eqref{eq:336140} and \eqref{eq:33620} into \eqref{eq:33600}, we finish the proof of Theorem \ref{thm:regret:dyn2}.
\end{proof}

\subsection{Proof of Theorem \ref{thm:regret block sw}} \label{appendix:proof:regret:block sw} 
\begin{proof}
	We denote $w^{*} = \lceil d^{1/3}(1 + \Delta)^{- 2/3}T^{1/2} \rceil.$ Since $M = \lceil 5d^{1/3}T^{1/2} \rceil,$ we have $1 \leq w^{*} \leq M.$ Therefore, there exists a $w^{+} \in J_w$ such that $w^{+} \leq w^{*} \leq 2w^{+}.$ We first decompose the dynamic regret as follows:

	\# 
	\text{D-Regret}(T) &= \sum_{k=1}^K \bigl(V^{\pi^{*, k},k}_1(s_1^k) - V^{\pi^k,k}_1(s_1^k)\bigr) \\
	&=  \underbrace{\sum_{k=1}^{K}V^{\pi^{*, k},k}_1(s_1^k) - \sum_{i=1}^{\lceil K/M \rceil}R_i(w^{+})}_{\dr (i)} + \underbrace{\sum_{i=1}^{\lceil K/M \rceil}R_i(w^{+}) - R_i(w_i)}_{\dr (ii)}.  
	\# 
Denote the total variation budget in the $i$-th block as $\Delta_i.$ For the first term, it holds with probability at least $1 - \zeta$ that 

\# \label{eq:regret block sw first term} 
    \text{Term(i)} &\le \sum_{i=1}^{\lceil K/M \rceil} \sqrt{16H^2 MH \cdot\log(4/\zeta)} + 2w^{+}\Delta_i H\sqrt{d} \notag\\
	&\qquad + 8dMH\sqrt{\log(w^{+})/w^{+}} + 8C' dMH^2\cdot  \sqrt{\log(w^{+}H^2d)/w^{+}}  \cdot \log(dMH/\zeta) \notag\\
    & \lesssim TH^{1/2}M^{-1/2}\log(1/\zeta) + w^{+}\Delta Hd^{1/2} + dT{w^{+}}^{-1/2} + dTH{w^{+}}^{-1/2}  \cdot \log(dT/\zeta) \notag\\
	&\lesssim TH^{1/2}M^{-1/2}\log(1/\zeta) + w^{*}\Delta Hd^{1/2} + dT{w^{*}}^{-1/2} + dTH{w^{*}}^{-1/2}  \cdot \log(dT/\zeta),
    \# 
where the first inequality follows from from the same derivation of \eqref{eq:3362} in the proof of Theorem \ref{thm:regret:dyn}. For the second term, we apply the regret bound of EXP3-P algorithm. In our case, we have $J$ arms, $\lceil K/M \rceil$ time steps and the loss is bounded within $[0, MH].$ Therefore, we can upper bound the second term as follows:
\# \label{eq:regret block second term}
\text{Term(ii)} \lesssim MH \sqrt{\frac{JT}{MH}}. 
\#
We complete the proof by substituting the values of $w^{*} = \lceil d^{1/3}(1 + \Delta)^{- 2/3}T^{1/2} \rceil,$ $M = \lceil 5d^{1/3}T^{1/2} \rceil$ and $J = \log_{2}M +1$ into Term(i) + Term(ii).

\end{proof}

\section{Results without Assumption \ref{assumption:orthonormal}} 
 
\subsection{Missing Proofs} \label{appendix:without:assump}
Without Assumption \ref{assumption:orthonormal}, we can obtain the following lemma, whose proof is deferred to Appendix \ref{appendix:pf:ucb2}.
\begin{lemma}[Upper Confidence Bound] \label{lem:ucb2}
	Under Assumption \ref{assumption:linear:mdp}, it holds with probability at least $1 - \zeta/2$ that
	\$
	 &-2B_h^k(s, a) - 2\Gamma_h^k(s, a) - \sqrt{dw} \cdot \sum_{i =1\vee (k-w)}^{k -1} \| \theta_h^i - \theta_h^{i+1} \|_2 - Hd\sqrt{w} \cdot \sum_{i = 1\vee(k - w)}^{ k- 1} \| \xi_h^i - \xi_h^{i+1} \|_2 \\
	 & \qquad \le l_h^k(s, a) \le \sqrt{dw} \cdot \sum_{i =1\vee (k-w)}^{k -1} \| \theta_h^i - \theta_h^{i+1} \|_2 + Hd\sqrt{w} \cdot \sum_{i = 1\vee(k - w)}^{ k- 1} \| \xi_h^i - \xi_h^{i+1} \|_2
	\$
	for any $(k, h) \in [K] \times [H]$ and $(s, a) \in \cS \times \cA$, where $w$ is the length of a sliding window defined in \eqref{eq:estimate:reward}, $B_h^k(\cdot, \cdot )$ is the bonus function of reward defined in \eqref{eq:bonus} and $\Gamma_h^k(\cdot, \cdot)$ is the bonus function of transition kernel defined in \eqref{eq:bonus}.  
\end{lemma}

Compared with Lemma \ref{lem:ucb}, Lemma \ref{lem:ucb2} shows that we can achieve optimism with slightly large bonus functions. Equipped with this lemma, we can derive slightly worse regret bounds for our algorithms.

\begin{proof}[Proof of Theorem \ref{thm:dyn3}]
	The proof is the same as the proof of Theorem \ref{thm:regret:dyn} except that Lemma~\ref{lem:ucb} is replaced by Lemma \ref{lem:ucb2}. 
\end{proof}

\begin{proof}[Proof of Theorem \ref{thm:adv:worse}]
 The proof is the same as the proof of Theorem \ref{thm:adv} except that Lemma \ref{lem:ucb} is replaced by Lemma \ref{lem:ucb2}. 
\end{proof}

\begin{proof}[Proof of Theorem \ref{thm:vi:worse}]
 The proof is the same as the proof of Theorem \ref{thm:regret:dyn2} except that Lemma~\ref{lem:ucb} is replaced by Lemma \ref{lem:ucb2}. 
\end{proof}

\subsection{Proof of Lemma \ref{lem:ucb2}} \label{appendix:pf:ucb2}
\begin{proof}
	Following the notations in the proof of Lemma \ref{lem:ucb}, we establish tighter bounds for {{Term (i.1)}} in \eqref{eq:con215} and {{Term (ii.1.2)}} in \eqref{eq:ucb1141} (see the proof of Lemma \ref{lem:ucb} in \S \ref{appendix:lem:ucb}). 
	
	For any $(k, h) \in [K] \times [H]$ and $(s, a) \in \cS \times \cA$, we have
	\# \label{eq:11100}
	|{\mathrm{Term (i.1)}}| &= \biggl| \phi(s, a)^\top (\Lambda_h^k)^{-1} \biggl( \sum_{\tau = 1\vee (k-w)}^{k-1} \phi(s_h^\tau, a_h^\tau)\phi(s_h^\tau, a_h^\tau)^\top(\theta_h^\tau - \theta_h^k) \biggr) \biggr| \notag\\
	&  \le \sum_{\tau = 1\vee (k-w)}^{k-1} \bigl| \phi(s, a)^\top (\Lambda_h^k)^{-1}  \phi(s_h^\tau, a_h^\tau) \bigr| \cdot \bigl|\phi(s_h^\tau, a_h^\tau)^\top(\theta_h^\tau - \theta_h^k)  \bigr| \notag\\
	&  \le \sum_{\tau = 1\vee (k-w)}^{k-1} \bigl| \phi(s, a)^\top (\Lambda_h^k)^{-1}  \phi(s_h^\tau, a_h^\tau) \bigr| \cdot \bigl\|\phi(s_h^\tau, a_h^\tau)\bigr\|_2 \cdot \bigl\|\theta_h^\tau - \theta_h^k\bigr\|_2 \notag\\
	& \le \sum_{\tau = 1\vee (k-w)}^{k-1} \bigl| \phi(s, a)^\top (\Lambda_h^k)^{-1}  \phi(s_h^\tau, a_h^\tau) \bigr| \cdot \sum_{i = \tau}^{k - 1}\bigl\|\theta_h^i - \theta_h^{i + 1} \bigr\|_2,
	\#
	where the second inequality is obtained by Cauchy-Schwarz inequality and the last last inequality follows from the facts that $\|\phi(\cdot, \cdot)\|_2 \le 1$ and  $\bigl\|\theta_h^\tau - \theta_h^k\bigr\|_2 \le \bigl\|\sum_{i = \tau}^{k - 1}(\theta_h^i - \theta_h^{i + 1})  \bigr\|_2 \le \sum_{i = \tau}^{k - 1}\bigl\|\theta_h^i - \theta_h^{i + 1} \bigr\|_2 $. Note that $\sum_{\tau = 1\vee (k-w)}^{k-1}\sum_{i = \tau}^{k - 1} = \sum_{i = 1\vee (k-w)}^{k-1}\sum_{\tau = 1 \vee (k - w)}^{i}$, we further obtain that 
	\# \label{eq:11101}
	&\sum_{\tau = 1\vee (k-w)}^{k-1} \bigl| \phi(s, a)^\top (\Lambda_h^k)^{-1}  \phi(s_h^\tau, a_h^\tau) \bigr| \cdot \sum_{i = \tau}^{k - 1}\bigl\|\theta_h^i - \theta_h^{i + 1} \bigr\|_2  \notag\\
	& \qquad = \sum_{i = 1\vee (k-w)}^{k-1}\sum_{\tau = 1 \vee (k - w)}^{i} \bigl| \phi(s, a)^\top (\Lambda_h^k)^{-1}  \phi(s_h^\tau, a_h^\tau) \bigr| \cdot \bigl\|\theta_h^i - \theta_h^{i + 1} \bigr\|_2 \notag\\
	& \qquad \le \sum_{i = 1\vee (k-w)}^{k-1} \sqrt{\sum_{\tau = 1 \vee (k - w)}^{i} \|\phi(s, a)\|_{(\Lambda_h^k)^{-1}}^2   \cdot \sum_{\tau = 1 \vee (k - w)}^{i} \|\phi(s_h^\tau, a_h^\tau)\|_{(\Lambda_h^k)^{-1}}^2 } \cdot \bigl\|\theta_h^i - \theta_h^{i + 1} \bigr\|_2 .
	\#
	Note that $\Lambda_h^k \succeq I_d$, which further implies 
	\# \label{eq:11102}
	\sum_{\tau = 1 \vee (k - w)}^{i} \|\phi(s, a)\|_{(\Lambda_h^k)^{-1}}^2 \le \sum_{\tau = 1 \vee (k - w)}^{i} \|\phi(s, a)\|_{2}^2 \le \sum_{\tau = 1 \vee (k - w)}^{i} 1 \le w.
	\# 
	Meanwhile, we have 
	\# \label{eq:11103}
	\sum_{\tau = 1 \vee (k - w)}^{i} \|\phi(s_h^\tau, a_h^\tau)\|_{(\Lambda_h^k)^{-1}}^2 &= \sum_{\tau = 1 \vee (k - w)}^{i} \tr\bigl(\phi(s_h^\tau, a_h^\tau)^\top {(\Lambda_h^k)^{-1}}\phi(s_h^\tau, a_h^\tau)\bigr) \notag\\
	&  =  \tr\biggl( {(\Lambda_h^k)^{-1}} \sum_{\tau = 1 \vee (k - w)}^{i}\phi(s_h^\tau, a_h^\tau)\phi(s_h^\tau, a_h^\tau)^\top \biggr).
	\# 
    Similar to the derivation of \eqref{eq:con2141}, we have 
	\# \label{eq:111031}
	\tr\biggl( {(\Lambda_h^k)^{-1}} \sum_{\tau = 1 \vee (k - w)}^{i}\phi(s_h^\tau, a_h^\tau)\phi(s_h^\tau, a_h^\tau)^\top \biggr) = \sum_{i = 1}^d \frac{\lambda_i}{\lambda_i + \lambda} \le d,
	\#
	where $\lambda_i$ is the $i$-th eigenvalue of $\sum_{\tau = 1 \vee (k - w)}^{i}\phi(s_h^\tau, a_h^\tau)\phi(s_h^\tau, a_h^\tau)^\top$.
	Plugging \eqref{eq:11101}, \eqref{eq:11102}, \eqref{eq:11103} and \eqref{eq:111031} into \eqref{eq:11100}, we have 
	\# \label{eq:11104}
	|{\mathrm{Term (i.1)}}| & = \biggl| \phi(s, a)^\top (\Lambda_h^k)^{-1} \biggl( \sum_{\tau = 1\vee (k-w)}^{k-1} \phi(s_h^\tau, a_h^\tau)\phi(s_h^\tau, a_h^\tau)^\top(\theta_h^\tau - \theta_h^k) \biggr) \biggr| \notag\\
	& \le \sqrt{dw} \cdot \sum_{i = 1\vee (k-w)}^{k-1} \bigl\|\theta_h^i - \theta_h^{i + 1} \bigr\|_2.
	\#
	Similarly, for any $(k, h) \in [K] \times [H]$ and $(s, a) \in \cS \times \cA$, we have 
	\# \label{eq:11105}
	|{\mathrm{Term (ii.1.2)}}| & = \biggl| \eta_h^k(s, a)^\top(A_h^k)^{-1}\biggl( \sum_{\tau = 1\vee(k - w)}^{k-1} \eta_h^{\tau}(s_h^\tau, a_h^\tau)\eta_h^{\tau}(s_h^\tau, a_h^\tau)^\top (\xi_h^\tau - \xi_h^k)  \biggr) \biggr| \notag \\
	& \le Hd\sqrt{w} \cdot \sum_{i = 1 \vee (k - w)}^{ k- 1} \| \xi_h^i - \xi_h^{i+1} \|_2.
	\#
	Plugging these two new bounds in the original proof of Theorem \ref{thm:regret:dyn} (cf. \S \ref{appendix:proof:regret:dyn}), we can obtain the desired results.
\end{proof}

\section{Useful Lemmas}

\begin{lemma}\label{lemma:telescope}
	Let $\{ \phi_t \}_{t=1}^\infty$ be an $\RR^d$-valued sequence with $\|\phi_t\|_2 \le 1$. Also, let $\Lambda_0\in\RR^{d\times d}$ be a positive-definite matrix with $\lambda_{\min}(\Lambda_0)\ge1$ and $\Lambda_t=\Lambda_0 + \sum_{j=1}^{t-1} \phi_j\phi_j^\top$. For any $t\in \ZZ_+$, it holds that
	\$
	\log\biggl( \frac{\det(\Lambda_{t+1})}{\det(\Lambda_1)} \biggr) \le
	\sum_{j=1}^t \phi^\top_j \Lambda^{-1}_{j}\phi_j \le
	2\log\biggl( \frac{\det(\Lambda_{t+1})}{\det(\Lambda_1)} \biggr).
	\$
\end{lemma}
 
 \begin{proof}
  See \cite{dani2008stochastic, rusmevichientong2010linearly, jin2019provably, cai2019provably} for a detailed proof.
 \end{proof}

\begin{lemma} \label{lemma:telescope2}
	For the $\Lambda_h^k$ defined in \eqref{eq:estimate:reward}, we have 
	\$
	\sum_{k=1}^K 1\wedge \|\phi(s_h^k, a_h^k)\|_{(\Lambda_h^k)^{-1}} \le 2d \lceil K/w \rceil \log\bigl( (w + \lambda) / \lambda \bigr) 
	\$
	for any $h \in [H]$.
\end{lemma}

\begin{proof}
	First, we rewrite the sums as follows.
	\# \label{eq:tele20}
	\sum_{k=1}^K 1 \wedge \|\phi(s_h^k, a_h^k)\|_{(\Lambda_h^k)^{-1}} = \sum_{t = 0}^{\lceil K/w \rceil - 1}\sum_{k = tw + 1}^{(t+1)w} 1 \wedge \|\phi(s_h^k, a_h^k)\|_{(\Lambda_h^k)^{-1}} .
	\#
	For the $t$-th block of length $w$ we define the matrix 
	\$
	W_h^{k, t} = \sum_{\tau = tw +1}^{k - 1}\phi(s_h^\tau, a_h^\tau)\phi(s_h^\tau, a_h^\tau)^\top + \lambda I_d .
	\$ 
	Recall the $\Lambda_h^k$ in \eqref{eq:estimate:reward}
	\$
	\Lambda_h^{k} = \sum_{\tau = 1\vee (k-w)}^{k-1}  \phi(s_h^\tau,a_h^\tau) \phi(s_h^\tau,a_h^\tau)^\top + \lambda I_d .
	\$
	Note that $\Lambda_h^k$ contains extra terms which are positive definite matrices for any $(k, h) \in  [tw, (t+1)w] \times [H]$, we have $\Lambda_h^k \succeq W_h^{k,t}$ for any $(k, h) \in  [tw, (t+1)w] \times [H]$. Hence, 
	\$
	(\Lambda_h^k)^{-1} \preceq (W_h^{k,t})^{-1}
	\$ 
	for any $(k, h) \in  [tw, (t+1)w] \times [H]$, which implies that 
	\# \label{eq:tele21}
	\sum_{t = 0}^{\lceil K/w \rceil - 1}\sum_{k = tw + 1}^{(t+1)w} 1 \wedge \|\phi(s_h^k, a_h^k)\|_{(\Lambda_h^k)^{-1}} &\le \sum_{t = 0}^{\lceil K/w \rceil - 1}\sum_{k = tw + 1}^{(t+1)w} 1 \wedge \|\phi(s_h^k, a_h^k)\|_{(W_h^{k, t})^{-1}}  \notag\\
	& \le \sum_{t = 0}^{\lceil K/w \rceil - 1} 2 \log\biggl( \frac{\det(W_{h}^{(t+1)w+1, t})}{\det(W_h^{tw, t})} \biggr) ,
	\#
	where the last inequality follows from Lemma \ref{lemma:telescope}. Moreover, we have $\|\phi(s, a)\|_2 \le 1$ for any $(s, a) \in \cS \times \cA$, which implies 
	\$
	W_{h}^{(t+1)w+1, t} = \sum_{\tau = tw +1}^{(t+1)w}\phi(s_h^\tau, a_h^\tau)\phi(s_h^\tau, a_h^\tau)^\top + \lambda I_d \preceq (w + \lambda)\cdot I_d
	\$ 
	for any $h \in [H]$. It holds for any $h \in [H]$ that 
	\# \label{eq:tele22}
	2 \log\biggl( \frac{\det(W_{h}^{(t+1)w+1, t})}{\det(W_h^{tw, t})} \biggr) \le 2d \log\bigl( (w + \lambda) / \lambda \bigr) .
	\#
	Plugging \eqref{eq:tele22} and \eqref{eq:tele21} into \eqref{eq:tele20}, we conclude the proof of Lemma \ref{lemma:telescope2}. 
\end{proof}

\begin{lemma}\label{lem:eventlm}
	Let $\lambda'=1$ in \eqref{eq:evaluation}.  For any $\zeta\in (0,1]$, the event $\mathcal{E}$ that, for any $(k,h)\in [K]\times[H]$,
	\#\label{eq:event}
	\Bigl\| \sum_{\tau=1}^{k-1} \eta_h^\tau(s_h^\tau,a_h^\tau) \cdot \bigl( V^{k}_{h+1}(s^\tau_{h+1}) - (\mathbb{P}_h^\tau V^{k}_{h+1})(s^\tau_{h}, a^\tau_h) \bigr) \Bigr\|_{(A^{k}_h)^{-1}}
	\le C''\sqrt{dH^2\cdot \log(dT/\zeta)}
	\#
	happens with probability at least $1-\zeta/2$, where $C''>0$ is an absolute constant that is independent of $C$.
\end{lemma}
\begin{proof}
	See Lemma B.3 of \cite{jin2019provably} or Lemma D.1 of \cite{cai2019provably} for a detailed proof.
\end{proof}

\begin{comment}

\begin{lemma}[Self-Normalized Bound for Vector-Valued Martingales] \label{lem:self}
	Let $\{\bar{\cF}_t\}_{t=0}^\infty$ be a filtration. Let $\{\eta_t\}_{t=0}^\infty$ be a real-valued stochastic process such that $\eta_t$ is $\bar{\cF}_t$-measurable and conditionally $R$-sub--Gaussian for some $R > 0$, that is,
	\# \label{eq:self1}
	\EE[e^{\lambda \eta_t} | \bar{\cF}_t ] \le e^{\lambda^2\sigma^2/2}
	\#
	for any $\lambda \in \RR$. Let $\{X_t\}_{t=1}^\infty$ be an $\RR^d$-valued stochastic process such that $X_t$ is $\cF_{t-1}$-measurable. Assume that $V$ is a $d \times d$ positive definite matrix. For ant $t \ge 0$, we define
	\$
	\bar{V}_t = V + \sum_{\tau = 1}^tX_\tau X_\tau^\top, \qquad S_t = \sum_{\tau = 1}^t \eta_tX_\tau .
	\$
	Then for any $\delta > 0$, with probability at least $1 - \delta$ that 
	\$
	\|S_t \|_{\bar{V}_t^{-1}} \le 2R^2\log\biggl(\frac{\det(\bar{V}_t^{-1})^{1/2}\det(V)^{1/2}}{\delta} \biggr) .
	\$
\end{lemma}
\begin{proof}
	See Theorem 1 of \cite{abbasi2011improved} for a detailed proof.
\end{proof}

	内容...
\end{comment}

\begin{lemma}[Pinsker's inequality] \label{lemma:pinsker}
	Denote $s \in \{s_1, s_2, \cdots, s_T\} \in \cS$ be the observed states from step $1$ to $T$. For any two distributions $\cP_1$ and $\cP_2$ over $\cS$ and any bounded function $f : \cS^\top \rightarrow [0, B]$, we have
	\$
	\EE_1 f(s) - \EE_2 f(s) \le \frac{\sqrt{\log2}B}{2} \cdot \sqrt{ \rm{KL}(\cP_2 \| \cP_1) }, 
	\$
	where $\EE_1$ and $\EE_2$ denote expectations with respect to $\cP_1$ and $\cP_2$.
\end{lemma}
\begin{proof}
	See Lemma 13 in \cite{JMLR:v11:jaksch10a} or Lemma B.4 in \cite{zhou2020provably} for a detailed proof.
\end{proof}

\begin{lemma} \label{lemma:kl:bound}
	Suppose $\xi$ and $\xi'$ have the same entries except for $j$-th coordinate. We also assume that $2\epsilon \le \delta \le 1/3$, then we have
	\$
	 {\rm{KL}}(\cP_{\xi'} \| \cP_{\xi}) \le \frac{16\epsilon^2}{(d-1)^2 \delta} \EE_{\xi}N_0 . 
	\$
\end{lemma}
\begin{proof}
	See Lemma 6.8 in \cite{zhou2020provably} for a detailed proof.
\end{proof}

\end{appendix}

\end{document}